%% file: 0main.tex
\documentclass{article} 
\usepackage{iclr2027_conference,times}

\input{math_commands.tex}

\usepackage{hyperref}
\usepackage{url}

\iclrfinalcopy

\usepackage{microtype}

\usepackage{bm}
\usepackage{subcaption}
\usepackage{mismath}
\usepackage{mathtools}
\usepackage{amsmath}
\usepackage{amsthm}
\usepackage{amssymb}
\usepackage{algorithm}
\usepackage{algorithmic}
\usepackage{graphicx}
\usepackage{setspace}
\usepackage{color}
\usepackage{wrapfig}
\usepackage{lipsum}
\usepackage{xcolor}
\usepackage{pgffor}
\usepackage{makecell}
\usepackage{multicol}
\usepackage{multirow}
\usepackage{siunitx}
\usepackage{todonotes}
\usepackage{enumitem}

\usepackage{thmtools}
\usepackage{thm-restate}
\usepackage{cleveref}

\newtheorem{theorem}{Theorem}
\newtheorem{lemma}{Lemma}

\newtheorem{remark}{Remark}
\newtheorem{assumption}{Assumption}

\usepackage{booktabs}
\usepackage{dsfont}

\setlist[itemize]{leftmargin=2em}

\title{Observation-Aligned Mask Priors for Learning Physical Fields from Authentic Occlusions}

\author{
    Chiyuan Ma$^{1}$, 
    Zihan Zhou$^{1,2}$, 
    Tianshu Yu$^{\dagger, 1,2}$ \\
    \textnormal{$^1$School of Data Science, The Chinese University of Hong Kong, Shenzhen} \\
    \textnormal{$^2$Shanghai Artificial Intelligence Laboratory} \\
    \texttt{\{chiyuanma, zihanzhou1\}@link.cuhk.edu.cn} \\ \texttt{yutianshu@cuhk.edu.cn} \\
}

\begin{document}

\maketitle

\begin{abstract}
Learning physical fields directly from incomplete observations is challenging because authentic occlusions are structured, sample-dependent, and often missing not at random, whereas existing methods typically rely on heuristic masking rules or predefined mask distributions. We propose Observation-Aligned Mask Priors, a framework that learns the distribution of authentic observation masks and uses it to construct context-query partitions for training from incomplete data. Specifically, we pretrain a Bayesian Flow Network (BFN) on binary observation masks to capture real occlusion topologies, then guide BFN sampling with a globally normalized cross-entropy objective to generate sample-specific masks aligned with each sparse observation. The intersection between the guided mask and the observed mask defines the context, and the remaining observed entries become query targets for a diffusion-based reconstruction model. We show that this intersection-based partitioning gives every valid observed dimension a strictly positive probability of being queried, preventing zero-query dead zones and local generative collapse. Experiments on three real-world oceanographic datasets with authentic satellite occlusions, across resolutions up to 256$\times$256, show consistent improvements over strong diffusion baselines in MSE and PSNR. These results demonstrate that learning mask priors from authentic occlusions is an effective alternative to heuristic masking for learning from incomplete physical observations without access to fully observed fields.
\end{abstract}

\input{1intro}

\input{2preliminary}

\input{3method}

\input{4experiment}

\input{5conclusion}

\section*{AI use statement}
In this work, we used generative AI tools for polishing the writing of the paper.
We have not used generative AI tools for coming up idea, and auto research are not applicable to this work. Additionally, we used generative AI tools for checking and correcting errors. We have reviewed all AI-assisted work. LLM-generated sentence was verified by authors. We take responsibility for the final content of this work, including text, claims or artifacts produced with the aid of generative AI.

\section*{Reproducibility statement}
The sources of the datasets have been provided in Appendix~\ref{app: dataset settings}, and detailed parameters have been provided in Appendix~\ref{app: implement_details} to make the results reproducible.

\bibliography{ref}
\bibliographystyle{iclr2027_conference}

\appendix
\input{7appendix}

\end{document}

%% file: math_commands.tex
\usepackage{amsmath,amsfonts,bm}

\def\eqref#1{equation~\ref{#1}}

\def\1{\bm{1}}

\DeclareMathAlphabet{\mathsfit}{\encodingdefault}{\sfdefault}{m}{sl}
\SetMathAlphabet{\mathsfit}{bold}{\encodingdefault}{\sfdefault}{bx}{n}

\newcommand{\softmax}{\mathrm{softmax}}



%% file: 1intro.tex
\section{Introduction}
Learning physical fields from observational data is a central goal of scientific machine learning, with broad impact on ocean monitoring~\citep{barth2024ensemble, martin2025generative}, weather analysis~\citep{bi2023accurate,lam2023learning}, and environmental forecasting~\citep{hsu2025forecasting, ashfaq2025predicting}. In these applications, however, the underlying physical field is almost never observed on a complete spatial grid~\citep{chen2025transfer}. Measurements arrive through authentic occlusions caused by cloud cover, sensor swath gaps, orbital trajectories, and geographic boundaries ~\citep{letraon:hal-03405376}, yielding observation masks that are structured, sample-dependent, and often missing not at random (MNAR). These masks are not generic corruption that can be approximated by random dropout; they are part of the sensing process itself. A central challenge is therefore to learn physical fields directly from incomplete observations, without requiring access to fully observed fields, while respecting the topology of real-world occlusions.

\paragraph{Existing Bottleneck.} Recently, a promising direction is to train directly on incomplete data through \emph{context-query partitioning}~\citep{zhou2025incomplete}. Instead of supervising a model with complete ground truth, one splits the observed region into a visible context and a withheld query region, and trains the model to reconstruct the query from the context. This removes the unrealistic requirement of fully observed supervision, but it introduces a new bottleneck: the success of learning now depends critically on how the context-query partition is constructed. In fact, the seminal study~\citep{zhou2025incomplete} shows that a spatial dimension can only be learned if it has a strictly positive probability of being assigned to the query region; otherwise its expected gradient vanishes and the model may output arbitrary values there. Existing approaches typically instantiate this partition using heuristic masking rules~\citep{zhou2025incomplete, daras2023ambientdiffusionlearningclean, majid2026ambientphysicstrainingneural}, such as independent pixel dropout, regular block masking, or hand-designed saliency heuristics. While convenient, these strategies implicitly assume a predefined mask distribution and are poorly matched to authentic physical occlusions, whose geometry is shaped by clouds, coastlines, and sensing constraints. When missingness is structured and potentially MNAR, the mismatch between heuristic masking and real observational topology becomes a fundamental obstacle.


\paragraph{Our Solution.}
The key idea is to learn, rather than assume, the mask distribution underlying authentic observations. We propose \textbf{Observation-Aligned Mask Priors}, which learns a generative prior over real observation masks and uses it to construct sample specific context--query partitions (see Fig.~\ref{fig: main_figure}). Specifically, we pretrain a Bayesian Flow Network (BFN)~\citep{graves2025bayesianflownetworks} on binary observation masks to capture authentic occlusion topology, and guide its sampling toward the current sparse observation using a globally normalized cross-entropy objective. The intersection between the guided mask and the observation mask defines the context, while the remaining observed entries form the query region.
This yields data-driven partitions that preserve realistic occlusion structure while adapting to each individual sample.

\paragraph{Theoretical \& Empirical Value.} Theoretically, we show that intersection-based partitioning gives every valid observed dimension a strictly positive probability of being queried. We first establish this property under independent sampling from the learned prior and then show that it remains valid under observation-aligned guided partitioning, thereby avoiding zero-query dead zones that can lead to local generative collapse and structural information deficiency. Empirically, across three real-world oceanographic datasets~\citep{letraon:hal-03405376} with genuine satellite occlusions---Black Sea CHL, Baltic Sea NANO, and Global Ocean SSS---and across resolutions up to 256$\times$256, our method consistently outperforms strong diffusion-based baselines in MSE and PSNR.

More related works are discussed in Appendix~\ref{app: related work}. We present our contributions as follows:
\begin{itemize}[leftmargin=*,nosep]
    \item \textbf{From heuristic masks to learned occlusion priors.} We introduce \emph{Observation-Aligned Mask Priors}, a new framework that replaces handcrafted context-query partitioning rules with a learned generative prior over authentic observation masks, and aligns this prior to each input sample through guided mask generation.

    \item \textbf{Strict positivity, not just empirical coverage.} We provide theoretical guarantees showing that our intersection-based partitioning strategy assigns every valid observed dimension a strictly positive query probability, both under mask-prior sampling and under observation-aligned guidance, thereby preventing zero-query dead zones and local collapse.

    \item \textbf{Real gains under authentic occlusions.} We validate the framework on \emph{three challenging real-world oceanographic datasets} with genuine structured missingness and show consistent improvements over diffusion baselines across multiple spatial resolutions, without ever requiring fully observed training fields.
\end{itemize}

%% file: 2preliminary.tex
\section{Preliminaries} \label{sec:preliminary}
In realistic physical science applications, training datasets typically only consist of partial observations $\bm{u}_{\text{obs}} = \bm{M} \odot \bm{u}_0$, where $\bm{M}$ is a binary observation mask. To learn complete physical fields directly from incomplete data, previous work~\citep{zhou2025incomplete} introduced a context-query partitioning framework.

\paragraph{Learning Fields via Context-Query Partitioning.} During training, each incomplete sample is strategically partitioned into an observable context mask $\bm{M}_{\text{ctx}} \subset \bm{M}$ and a non-request mask $\bm{M}_{\text{qry}} \subseteq \bm{M}$. The imputation model $\bm{u}_{\bm{\phi}}$ is trained to reconstruct the query region conditioned solely on the context region:
\begin{equation} \label{eq: training objective}
    \mathcal{L}(\bm{u}_{\text{obs}}, \bm{M}_{\text{ctx}}, \bm{M}_{\text{qry}}) = ||\bm{M}_{\text{qry}} \odot (\bm{u}_{\bm{\phi}}(\bm{M}_{\text{ctx}} \odot \bm{u}_{\text{obs}}, \bm{M}_{\text{ctx}}) - \bm{u}_{\text{obs}})||^2.
\end{equation}
The context and query regions are constructed to be disjoint so that learning is restricted to values withheld from the model input, making it easier for the model to learn. Since context values are already observed and can be directly copied during reconstruction, the objective focuses supervision on the query region, encouraging the model to infer missing structure rather than  rely on local pixel memorization.

\paragraph{Theoretical Foundation of Global Prediction.} The success of this framework fundamentally depends on the context-query splitting strategy. Theoretical analysis establishes that minimizing Eq.(\ref{eq: training objective}) leads to an optimal prediction $(\bm{u}_\phi)_i$ for a specific feature dimension $i$ given by:
\begin{equation}  \label{eq: strict positive condition}
    \scalebox{0.94}{$\left( \bm{u_\bm{\phi}} \left(\bm{M}_{\text{ctx}} \odot \bm{u}_{\text{obs}}, \bm{M}_{\text{ctx}} \right) \right)_i = \begin{cases}
        \mathbb{E} \left[ \left( \bm{u}_0 \right)_i \mid \bm{M}_{\text{ctx}} \odot \bm{u}_{\text{obs}}, \bm{M}_{\text{ctx}} \right], &P( \left(\bm{M}_{\text{qry}}\right)_i = 1 \mid \bm{M}_{\text{ctx}}) > 0 \\
        \text{an arbitrary value}, &P( \left(\bm{M}_{\text{qry}}\right)_i = 1 \mid \bm{M}_{\text{ctx}}) = 0
    \end{cases}$}
\end{equation}
This reveals that the model can only learn a meaningful conditional expectation for a specific spatial dimension if its probability of being queried is strictly positive. This condition ensures non-vanishing expected gradients, as detailed in Appendix~\ref{app:learning dynamics basis and analysis}.

\paragraph{Bottleneck of Heuristic Design.} To ensure that the model learns a complete physical field, every originally unobserved dimension must have a strictly positive probability of being queried during training. Although previous work~\citep{zhou2025incomplete, majid2026ambientphysicstrainingneural} identified this necessity, it relied on manually designing $\bm{M}_{\text{ctx}}$ selection strategies tailored to each individual observation pattern, such as independent pixel dropout and block-wise occlusion. Such customized methods inevitably introduce distribution-specific hyperparameters and fail to integrate the design of various mask topologies into a unified framework. More importantly, relying on simplistic handcrafted rules makes it mathematically arduous to universally guarantee this strict positivity across complex real-world spatial dependencies. Therefore, overcoming the bottleneck requires moving beyond heuristic splitting rules to establish a unified mechanism capable of providing robust theoretical guarantee for any valid spatial topology.


%% file: 3method.tex
\section{Method}  \label{sec: section method}

\begin{figure}[t]
    \centering
    \includegraphics[width=\textwidth]{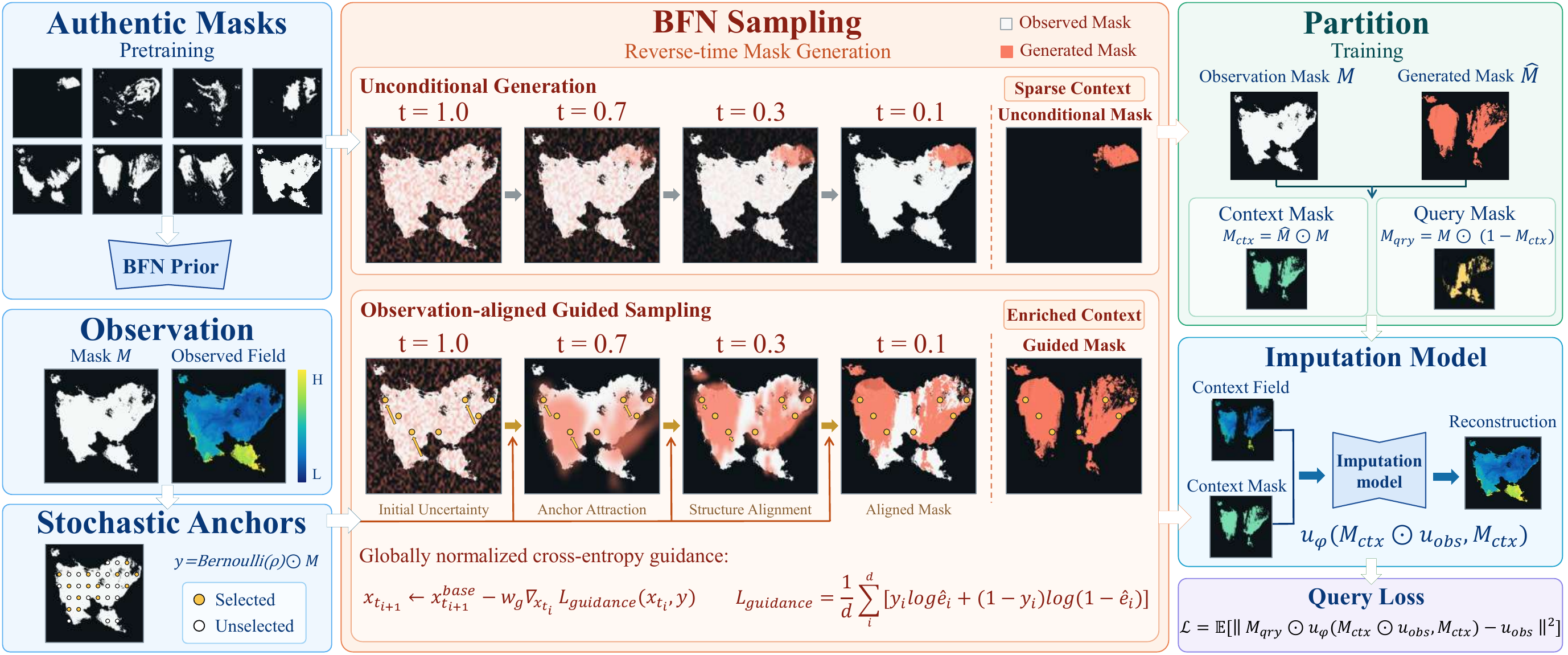}
    \caption{\textbf{Overview of our framework.} \emph{Left}:  A Bayesian Flow Network (BFN) is first pretrained on authentic observation masks to learn a prior over real occlusion topologies. \emph{Middle}:  Given an incomplete sample, stochastic anchors are drawn from the observed mask and used to guide BFN sampling through a globally normalized cross-entropy objective, yielding an observation-aligned generated mask. \emph{Right}:  The intersection between the generated mask and the observed mask defines the context region, and the remaining observed entries form the query region.The imputation model is then trained to reconstruct the query from the context.}
    \label{fig: main_figure}
\end{figure}

\subsection{Overview: The Observation-Aligned Imputation Framework}
In this work, we propose a novel observation-aligned imputation framework to reconstruct complete physical fields from incomplete observations.
Our method builds upon the context-query learning paradigm, where a prediction model learns to infer missing query regions $\bm{M}_{\text{qry}}$ conditioned on observed context regions $\bm{M}_{\text{ctx}}$, but fundamentally transforms how this partitioning is constructed. Rather than relying on simplistic, random splitting strategies that fail to capture the structured nature of real-world missingness, our framework intelligently drives the data partitioning process using a generative prior.

Specifically, our framework pipeline consists of three interconnected modules, which will be detailed in the following subsections:
\begin{itemize}
    \item \textbf{Mask Distribution Modeling (Sec.~\ref{sec: Modeling Mask Distributions via Bayesian Flow Networks}):} We first introduce a Bayesian Flow Network (BFN)~\citep{graves2025bayesianflownetworks} to explicitly model the authentic spatial distribution of the discrete binary observation masks.
    \item \textbf{Observation-Aligned Mask Generation (Sec.~\ref{sec: Observation-Aligned Context Mask Generation}):} To align the generated masks with the actual sparse observations of a given sample, we propose a classifier guidance mechanism during the BFN sampling phase. This ensures the generated context masks respect both the learned physical occlusion priors and the available data points.
    \item \textbf{Guided Context-Query Training (Sec.~\ref{sec: Imputation Training with Guided Partitioning}):} Finally, these intelligently generated masks are used to construct $\bm{M}_{\text{ctx}}$ and $\bm{M}_{\text{qry}}$, dynamically dictating the training of the prediction model. 
\end{itemize}

By replacing heuristic random masking with this generative, observation-aligned selection mechanism, our framework actively concentrates the prediction model's learning capacity on the most critical missing regions, significantly bridging the gap between unstructured data matching and structured physical realities.

\subsection{Context-Query Partitioning via Generative Priors} \label{sec: Context-Query Partitioning via Generative Priors}
As established in Sec.~\ref{sec:preliminary}, relying on heuristic masking strategies introduces severe limitations, as handcrafted rules cannot encapsulate the complex, structured nature of authentic occlusions. To establish a universally robust framework, we shift paradigm from manually partitioning observations to explicitly modeling the realistic underlying mask distribution $P(\bm{M})$ using a data-driven generative prior.
By driving the context-query selection through this learned prior, we establish a unified mechanism that guarantees non-zero query probabilities for any valid spatial topology. Specifically, we theoretically establish that intersecting two independent samples from $P(\bm{M})$ naturally satisfies this strict positivity requirement, thus eliminating the need for heuristic designs.
\begin{restatable}
[Strict Positivity of Query Probabilities via Mask Intersection]{theorem}{uniformqueryexposure} 
\label{thm:uniform_query_exposure}
Let $\bm{M}_1, \bm{M}_2 \in \{0, 1\}^d$ be independently sampled binary masks from the same underlying distribution $p(\bm{M})$. Define the context mask as $\bm{M}_{\operatorname{ctx}} = \bm{M}_1 \odot \bm{M}_2$ and the query mask as $\bm{M}_{\operatorname{qry}} = \bm{M}_1 \odot (\bm{1} - \bm{M}_{\operatorname{ctx}})$. Under Assumption~\ref{assump: Intersection Coverage Completeness}, for any valid context mask $\bm{m}$ and any unobserved dimension $i$ (where $\bm{m}_i = 0$), the conditional probability of dimension $i$ being selected as a query is strictly positive:
\begin{equation}
    P((\bm{M}_{\operatorname{qry}})_i = 1 \mid \bm{M}_{\operatorname{ctx}} = \bm{m}) > 0.
\end{equation}
\end{restatable}
The proof can be found in Appendix~\ref{app: Strict Positivity of Query Probabilities via Mask Intersection}. Theorem~\ref{thm:uniform_query_exposure} provides a rigorous mathematical guarantee for our partitioning strategy. However, realizing this guarantee in practice introduces a new prerequisite: we must successfully model and sample from the complex, discrete distribution $p(\bm{M})$. This theoretical requirement directly motivates the need for a powerful discrete generative prior.

\subsection{Modeling Mask Distributions via Bayesian Flow Networks (BFN)} \label{sec: Modeling Mask Distributions via Bayesian Flow Networks}
To fulfill the prerequisite established by Theorem~\ref{thm:uniform_query_exposure}, we require a generative model capable of capturing the intricate spatial dependencies inherent in authentic missing patterns. Our mask generation module is fundamentally built upon Bayesian Flow Networks (BFN)~\citep{graves2025bayesianflownetworks}, a framework naturally suited for modeling discrete binary variables.
To facilitate gradient-based interventions for observation-aligned conditioning (Sec.~\ref{sec: Observation-Aligned Context Mask Generation}), we use a BFN-inspired continuous latent formulation with a discrete data prediction objective. This formulation is motivated by the connection between BFN regression and diffusion score matching~\citep{xue2024unifying}. Exact equivalence to the original BFN further requires matching the latent parameterization and noise schedule (Appendix~\ref{app: Methodological Connections and Distinctions}).


Let $C=2$ denote the number of mask classes, $d$ the number of pixels, and $K>0$ an independent logit scaling factor. For one pixel, $c\in\{1,\dots,C\}$ has one-hot encoding $\bm{e}_c\in\mathbb{R}^{C}$ and continuous target $\bm{x}_0=K\bm{e}_c$. Scaling only affects continuous embedding, while the classification target remains the exact one-hot vector $\bm{e}_c$. We corrupt the continuous embedding according to

\begin{equation}
    \bm{x}_{t} = \alpha_{t}\bm{x}_{0} + \sigma_{t}\bm{\epsilon}, \quad \bm{\epsilon} \sim \mathcal{N}(\bm{0}, \bm{I}).
\end{equation}
For a full mask, the latent stacks $d$ blocks of $C$ logits in $\mathbb{R}^{d \times C}$, with independent standard Gaussian noise on all coordinates. Softmax is applied separately over the $C$ classes at each pixel, and the squared loss below is summed over pixels. Under this Gaussian observation model, adding a common scalar to the logits within any pixel leaves the joint class posterior unchanged. Consequently, the optimal data predictor $\mathbb{E}[\bm{x}_0\mid\bm{x}_t]$ is shift invariant and can be represented using only $\softmax(\bm{x}_t)$, without assuming independent mask pixels (Appendix~\ref{app: Representational Sufficiency of the Softmax Projection}). The full score and noise predictors are not shift invariant: they retain an explicit linear dependence on $\bm{x}_t$. We therefore parameterize the data predictor as $K\hat{\bm{e}}_{\bm{\theta}}(t,\softmax(\bm{x}_t))$ and obtain a reweighted noise matching objective through the corresponding noise predictor (Appendix~\ref{app: Reparameterization of BFN from Noise Matching to Discrete Data Matching})~\citep{zheng2023improved}:
\begin{equation} \label{eq: bfn data matching}
    \mathcal{L}_{\text{DM-discrete}} = \mathbb{E}_{t, c, \bm{\epsilon}} \left[ w(t) \| \hat{\bm{e}}_{\bm{\theta}}(t, \softmax(\bm{x}_{t})) - \bm{e}_{c} \|^{2} \right],
\end{equation}
where $w(t)$ is a time-dependent weight. Exact equality with noise matching weighted by $w_{\mathrm{NM}}(t)$ requires $w(t)=w_{\mathrm{NM}}(t)\alpha_t^2K^2/\sigma_t^2$ at nondegenerate training times. Tweedie's formula motivates the score estimate
\begin{equation}
    \bm{s}_{\bm{\theta}}(t,\bm{x}_t) = \frac{\alpha_t K \hat{\bm{e}}_{\bm{\theta}}(t, \softmax(\bm{x}_t)) - \bm{x}_t}{\sigma_t^2}.
\end{equation}
This equals the true score when the network predicts the exact class posterior. We use this estimate in the probability flow ODE, with the Gaussian terminal approximation $\bm{x}_T\sim\mathcal{N}(\bm{0},\bm{I}_{d \times C})$. At each pixel $i$, we decode $\hat c_i=\arg\max_{j\in\{1,\dots,C\}}(\bm{x}_0)_{i,j}$ and set $\hat M_i=\mathds{1}[\hat c_i=2]$, where classes 1 and 2 denote missing and observed entries, respectively.

\paragraph{Limitation of Unconditional Intersection.} While sampling unconditionally from the learned marginal distribution $p(\bm{M})$ strictly satisfies the positivity guarantee of Theorem~\ref{thm:uniform_query_exposure}, directly applying this unconditional intersection presents a critical practical bottleneck. Because unconditional generation operates entirely independent of the specific valid observations of a given data sample, the spatial overlap between the authentic observation mask and the generated mask can be exceedingly small. Such a minimal intersection produces a highly sparse context region $\bm{M}_{\text{ctx}}$, which inherently exacerbates the structural information gap identified in prior work~\citep{zhou2025incomplete}, leaving the imputation model without sufficient spatial anchors to reconstruct complex physical fields. To overcome this limitation and ensure a sufficiently informative context partition, it is imperative to transition from unconditional sampling to observation-aligned conditional generation. In the following section, we introduce our specific observation-aligned mechanism to achieve this conditional generation.

\subsection{Observation-Aligned Context Mask Generation} \label{sec: Observation-Aligned Context Mask Generation}
Before detailing the practical guidance mechanism, a fundamental theoretical question must be addressed: does restricting the mask generation to align with a specific observation break the strict positivity requirement established in Theorem~\ref{thm:uniform_query_exposure}? Fortunately, Theorem~\ref{thm:strict_positivity_under_ratio_guided} provides the theoretical assurance that under a well-defined intersection constraint, conditional generation strictly preserves this property.
\begin{restatable}
[Strict Positivity of Query Probabilities under Ratio-Guided Partitioning]{theorem}{strictpositivityratioguided} 
\label{thm:strict_positivity_under_ratio_guided}
Let $\bm{M} \in \{0,1\}^d$ be a fixed observation mask. Let the conditionally generated mask $\hat{\bm{M}} \in \{0,1\}^d$ be sampled from the guided distribution $P(\hat{\bm{M}} \mid C_k)$ satisfying the intersection constraint $C_k$ as defined in Assumption~\ref{assump: Constrained Coverage Completeness}. Define the context mask as $\bm{M}_{\text{ctx}} = \hat{\bm{M}} \odot \bm{M}$ and the query mask as $\bm{M}_{\text{qry}} = \bm{M} \odot (1 - \bm{M}_{\text{ctx}})$. Under Assumption~\ref{assump: Constrained Coverage Completeness}, for any valid observed dimension $i$ (where $\bm{M}_i = 1$), the marginal probability of dimension $i$ being assigned as a query region is strictly positive: $P((\bm{M}_{\text{qry}})_i = 1 \mid C_k) > 0$. \footnote{The condition shifts from $\bm{M}_{\text{ctx}}$ (Theorem \ref{thm:uniform_query_exposure}) to $C_k$ to reflect the algorithmic transition from random partitioning to observation-anchored generation. Guaranteeing the marginal probability $P((\bm{M}_{\text{qry}})_i = 1 \mid C_k) > 0$ is mathematically sufficient to ensure a strictly positive query frequency ($p_i > 0$ in Eq.(\ref{eq: linear relation})) across epochs. See Remark~\ref{remark: uniform_query_exposure} in Appendix for a detailed discussion.}
\end{restatable}
The proof can be found in Appendix~\ref{app: Strict Positivity of Query Probabilities under Ratio-Guided Partitioning}. Empowered by this theoretical guarantee, we formulate a practical classifier guidance mechanism during the BFN sampling phase to dynamically anchor the generation process to the valid observations. Our method consists of three key steps:
\paragraph{Stochastic anchor construction.} First, we construct a stochastic binary anchor $\bm{y} \in \{0, 1\}^{d}$ by applying Bernoulli sampling to the valid observation points (where $\bm{M}_i=1$, representing the $i$-th pixel of $\bm{M}$) with a predefined retention ratio $\rho \in (0, 1)$:
\begin{equation}
    \bm{y}_i = \mathds{1}[\bm{r}_i < \rho] \cdot \bm{M}_i, \quad \bm{r}_i \sim \text{Uniform}(0, 1).
\end{equation}
If we directly guide the BFN using all observation points, the generation tends to collapse into a single deterministic output (see Appendix~\ref{app: stochastic anchor} for details). This lacks the diversity needed for robust downstream training. By randomly dropping a fraction of these observation points, we inject essential stochasticity. It forces the BFN to dynamically connect these scattered anchor points using its learned physical priors. Consequently, the model can generate diverse, physically plausible sub-masks for the exact same observation field. 

\paragraph{Globally normalized guidance objective.} Let $\bm{x}_t \in \mathbb{R}^{d \times C}$ be the stacked latent state at a given BFN sampling time step $t$, with softmax applied over the $C=2$ classes at each pixel. Denote the predicted observed-class probabilities by $\hat{\bm{e}}\in[0,1]^d$, where $\hat e_i=[\hat{\bm{e}}_{\bm{\theta}}(t,\softmax(\bm{x}_t))]_{i,2}$. We compute a cross-entropy loss over all $d$ pixels between these probabilities and the anchor $\bm{y}$:
\begin{equation}
    \mathcal{L}_{\text{guidance}}(\bm{x}_t, \bm{y}) = -\frac{1}{d} \sum_{i=1}^{d} \left[ \bm{y}_i \log \hat{\bm{e}}_i + (1 - \bm{y}_i) \log (1 - \hat{\bm{e}}_i) \right].
\end{equation}
Global normalization ensures that the gradient magnitude of the guidance signal remains highly stable across samples with drastically different sparsity levels, preventing gradient explosion or vanishing. 

\paragraph{Latent space intervention.} Finally, we inject this guidance signal during the reverse sampling process of the BFN. Let $t_i$ and $t_{i+1}$ denote consecutive discrete time steps in the reverse time schedule. Denoting $x_{t_{i+1}}^{\text{base}}$ as the unconditional latent state update computed from $\bm{x}_{t_i}$, we intervene by applying a gradient descent step using our guidance loss:
\begin{equation}
    \bm{x}_{t_{i+1}} \leftarrow \bm{x}_{t_{i+1}}^{\text{base}} - w_g \nabla_{\bm{x}_{t_i}} \mathcal{L}_{\text{guidance}}(\bm{x}_{t_i}, \bm{y}),
\end{equation}
where $w_g$ is the guidance scale that acts as a balancing knob. It guarantees that the final generated mask respects the learned occlusion priors while strictly anchoring to the selected sparse observations.

\subsection{Imputation Training with Guided Partitioning} \label{sec: Imputation Training with Guided Partitioning}

\paragraph{Strategic partitioning via generative intersection.}
Building upon the theoretical foundations established in Theorem~\ref{thm:uniform_query_exposure} and the validity guarantees of Theorem~\ref{thm:strict_positivity_under_ratio_guided}, we now formalize our final context-query partitioning strategy. To learn from incomplete observations $\bm{u}_{\text{obs}} = \bm{M} \odot \bm{u}_{0}$, we leverage the observation-aligned binary mask $\hat{\bm{M}}\in\{0,1\}^{d}$ generated by our guided BFN to execute the intersection-based partition.
Specifically, we define context region $\bm{M}_{\text{ctx}}$ as the intersection between actual observed fields $\bm{M}$ and BFN conditional generated fields $\hat{\bm{M}}$. The query region $\bm{M}_{\text{qry}}$ is the remaining portion after excluding the observable area of the context region. Specifically, 
\begin{equation}
    \bm{M}_{\text{ctx}} = \hat{\bm{M}} \odot \bm{M}, \quad \bm{M}_{\text{qry}} = \bm{M} \odot (\bm{1} - \bm{M}_{\text{ctx}}).
\end{equation}

\paragraph{Training objective and information withholding.} During training of the main imputation model, the network $\bm{u}_{\bm{\phi}}$ is fed exclusively with the context elements $\bm{M}_{\text{ctx}} \odot \bm{u}_{\text{obs}}$ and the context mask $\bm{M}_{\text{ctx}}$. The optimization objective is to minimize the Mean Squared Error (MSE) between the network’s prediction and the clean observations, but \textit{only} within the query regions:
\begin{equation}
    \mathcal{L} = \mathbb{E} \left[ \| \bm{M}_{\text{qry}} \odot \left( \bm{u}_{\bm{\phi}}(\bm{M}_{\text{ctx}} \odot \bm{u}_{\text{obs}}, \bm{M}_{\text{ctx}}) - \bm{u}_{\text{obs}} \right) \|^{2} \right].
\end{equation}
By deliberately withholding the query information $\bm{M}_{\text{qry}}$ during the forward pass, we force the model to abandon local pixel memorization and instead learn the spatial correlations between different physical regions. The model must ``fill in the blanks'' by understanding the underlying spatial structure of the physical field. Because the query region is selected by the conditional BFN to match realistic physical occlusions, the model focuses its learning capacity on the most challenging and meaningful spatial transitions.

\begin{figure*}[t]
\vspace{-2mm}
\noindent
\begin{minipage}[t]{0.48\textwidth}
\vspace{0pt} 
\begin{algorithm}[H]
\small 
\caption{Guided BFN Mask Generation}
\label{alg:bfn_mask}
\begin{algorithmic}[1]
\REQUIRE $\bm{M}$, pre-trained BFN $\bm{\hat{e}_\theta}$, $w_g$, $n$, $\rho$
\ENSURE generated mask $\bm{\hat{M}}$
\STATE anchor $\bm{y}_i = \mathds{1}[\bm{r}_i < \rho] \cdot \bm{M}_i$ , $\bm{r}_i \sim \text{Uni.}(0,1)$
\STATE time schedule $\{t_i\}_{i=0}^n$ ($t_{0}=1, t_{n}=0$)
\STATE sample $\bm{x}_{t_{0}} \sim \mathcal{N}(\bm{0}, \bm{I}_{d \times C})$
\FOR{$i = 0, \dots, n-1$}
    \STATE $\hat{\bm{x}}_0\leftarrow K\bm{\hat{e}_{\bm{\theta}}}(t_i,\softmax(\bm{x}_{t_i}))$
    \STATE $\bm{x}_{t_{i+1}}^{\text{base}} \leftarrow \alpha_{t_{i+1}} \hat{\bm{x}}_0 + \frac{\sigma_{t_{i+1}}}{\sigma_{t_i}} ( \bm{x}_{t_{i}}-\alpha_{t_{i}} \hat{\bm{x}}_0)$
    \STATE $\bm{x}_{t_{i+1}} \leftarrow \bm{x}_{t_{i+1}}^{\text{base}} - w_g\nabla_{\bm{x}_{t_i}}\mathcal{L}_\text{guid}(\bm{x}_{t_i}, y)$
\ENDFOR
\STATE $\hat c_i\leftarrow\arg\max_{j=1,2}(\bm{x}_{\text{0}})_{i,j}$ for each pixel $i$
\STATE $\bm{\hat{M}}_i\leftarrow\mathds{1}[\hat c_i=2]$ for each pixel $i$
\RETURN $\bm{\hat{M}}$
\end{algorithmic}
\end{algorithm}
\end{minipage}
\hfill
\begin{minipage}[t]{0.48\textwidth} 
\vspace{0pt} 
\begin{algorithm}[H]
\small
\caption{Train Imputation with Partition}
\label{alg:main_training}
\begin{algorithmic}[1]
    \REQUIRE dataset $\mathcal{D} = \{(\bm{u}_{\text{obs}}^{(i)}, \bm{M}^{(i)})\}_{i=1}^N$
    \ENSURE imputation model parameters $\bm{\phi}$
    \WHILE{not converged}
    \FOR{each batch $(\bm{u}_{\text{obs}}, \bm{M}) \in \mathcal{D}$}
        \STATE $\bm{\hat{M}} \leftarrow$ \text{Algorithm \ref{alg:bfn_mask}}($\bm{u}_{\text{obs}}, \bm{M}, \dots$)
        \STATE $\bm{M}_{\text{ctx}} \leftarrow \bm{\hat{M}} \odot \bm{M}$
        \STATE $\bm{M}_{\text{qry}} \leftarrow \bm{M} \odot (1 - \bm{M}_{\text{ctx}})$
        \STATE $\bm{\hat{u}} \leftarrow \bm{u_{\bm{\phi}}}(\bm{M}_{\text{ctx}} \odot \bm{u}_{\text{obs}}, \bm{M}_{\text{ctx}})$
        \STATE $\mathcal{L} \leftarrow ||\bm{M}_{\text{qry}} \odot (\hat{\bm{u}} - \bm{u}_{\text{obs}})||^2$
        \STATE update $\bm{\phi}$ using gradient descent on $\mathcal{L}$
    \ENDFOR
    \ENDWHILE
\end{algorithmic}
\end{algorithm}
\end{minipage}
\vspace{-3mm}
\end{figure*}

\paragraph{Convergence to conditional expectation.} Under this prediction-driven training objective, the network naturally learns to approximate the conditional expectation $\mathbb{E}[\bm{u}_{0} | \bm{M}_{\text{ctx}} \odot \bm{u}_{\text{obs}}, \bm{M}_{\text{ctx}}]$. 
Crucially, as mathematically guaranteed by Theorem~\ref{thm:strict_positivity_under_ratio_guided}, our guided partitioning inherently satisfies the strict positivity requirement for every valid spatial dimension. By eliminating the heuristic bottleneck identified in Sec.~\ref{sec:preliminary}, this ensures that no part of the observation space is ignored. Over many iterations, the model develops a robust ability to reconstruct missing spatial domains. Ultimately, this enables the framework to transfer its reconstructive capabilities from synthetically masked training subsets to the inherently missing regions of real-world data, successfully bridging the gap without ever requiring access to any fully complete observation.

%% file: 4experiment.tex
\section{Experiments} \label{sec: section experiments}
\subsection{Baselines}\label{sec:experiment-baselines}
To evaluate the effectiveness of our framework, we compare it against the following established diffusion-based reconstruction approaches. \textbf{DINDiff}~\citep{barth2024ensemble} utilizes a denoising diffusion model with hierarchical masking to reconstruct incomplete satellite observations via ensemble averaging. \textbf{Ambient Diffusion}~\citep{daras2023ambientdiffusionlearningclean} reconstructs uncorrupted distributions directly from highly corrupted training data by incorporating known degradation processes into the forward pass. \textbf{MissDiff}~\citep{ouyang2023missdifftrainingdiffusionmodels} handles incomplete datasets by optimizing a unified diffusion objective that effectively marginalizes over unobserved variables. In addition to these evaluated methods, we also explored other baselines in Appendix~\ref{app: baseline_and_method_comparison}. For \textbf{our method}, we first pre-train a dataset-specific BFN model on the observation masks, freezing its parameters during the primary imputation training to maintain stable generative priors. Detailed settings are provided in Appendix~\ref{app: implement_details}.

\subsection{Datasets} \label{sec:experiment-datasets}
We conduct comprehensive experiments on three real-world high-resolution datasets with authentic satellite occlusions~\citep{letraon:hal-03405376}:
\textbf{(i) Black Sea CHL} (Black Sea Chlorophyll-a Dataset) offers multi-sensor observations of the Black Sea at a 1-km resolution, measuring the chlorophyll-a concentration. \textbf{(ii) Baltic Sea NANO} (Baltic Sea Nano-Phytoplankton Dataset) provides the observations of Baltic Sea nano-phytoplankton concentration at a 1-km resolution.  \textbf{(iii) Global Ocean SSS} (Global Sea Surface Salinity Dataset) measures global sea surface salinity from satellite at a 25-km resolution. Regarding the data format, we processed the datasets into partially observed gridded fields and observation mask data pairs. Detailed information is provided in Appendix~\ref{app: dataset settings}.




\subsection{Experimental Results}
\paragraph{Evaluation.} We evaluate the model performance by overlaying another mask over the existing observations. Specifically, we extract another observation mask $\bm{\tilde{M}}$ from the remaining dataset. Overlay it onto the existing observations $\bm{M}_{\text{eval}}$, and use the intersection of $\bm{\tilde{M}} \odot \bm{M}_{\text{eval}}$ as the new observation mask $\bm{M}$. The model will then predict the remaining portion given $\bm{M}$ and $\bm{M} \odot \bm{u}_{\text{obs}}$. The evaluation will be conducted in the region $\bm{M}_{\text{eval}}\odot(1-\bm{M})$ where real values naturally exist as evaluation criteria. To assess reconstruction quality, we use three physically meaningful metrics: \textbf{MSE Loss} measures the mean squared error on masked observed regions using random real-mask overlays. \textbf{PSNR} assesses global signal fidelity. 

\begin{table}[tb]
\centering
\caption{Performance comparison on physical fields imputation tasks with the resolution of $64 \times 64$ across three real-world datasets.}
\label{tab: main results}
\footnotesize
\begin{tabular}{c|cccccc}
\toprule
\multicolumn{1}{c}{\multirow{2}{*}{\textbf{Method}}} & \multicolumn{2}{c}{\textbf{Black Sea CHL}}    & \multicolumn{2}{c}{\textbf{Baltic Sea NANO}}   & \multicolumn{2}{c}{\textbf{Global Ocean SSS}}    \\ \cmidrule(lr){2-3} \cmidrule(lr){4-5} \cmidrule(lr){6-7} 
\multicolumn{1}{c}{} & MSE $\downarrow$ & PSNR $\uparrow$ & MSE $\downarrow$ & PSNR $\uparrow$ & MSE $\downarrow$ & PSNR $\uparrow$ \\ \midrule
AmbientDiff   & 0.228 & 17.717 & 0.365 & 15.466 & 0.286 & 18.022 \\ 
DINDiff       & 0.306 & 15.355 & 0.322 & 15.788 & 0.408 & 16.373 \\ 
MissDiff      & 0.606 & 13.575 & 0.579 & 13.672 & 0.413 & 16.644 \\ 
\midrule
\textbf{Ours} & \textbf{0.193} & \textbf{17.825} & \textbf{0.294} & \textbf{16.101} & \textbf{0.276} & \textbf{18.192} \\
\bottomrule
\end{tabular}
\vspace{-3mm}
\end{table}

\paragraph{Results.} Tab.~\ref{tab: main results} presents the quantitative evaluation of our framework against established diffusion-based imputation baselines in three real-world oceanographic datasets. In addition to the experiments at 64×64 resolution, we also scaled up to $128 \times 128$ and $256 \times 256$ resolution (Tab.~\ref{tab: scale up to 128} and Tab.~\ref{tab: scale up to 256}). Our method consistently achieves superior performance, yielding the lowest MSE Loss and the highest PSNR across all evaluated scenarios. While baselines such as AmbientDiff suffer from severe structural information gaps due to reliance on random spatial masking, and approaches like DINDiff and MissDiff struggle to maintain global fidelity under severe sparsity, our framework robustly reconstructs complex physical fields. By leveraging our proposed context-query partitioning strategy, the model efficiently concentrates its learning capacity on critical missing regions, successfully bridging deterministic observed anchors with highly coherent, physically realistic generated structures.

\subsection{Uncertainty Quantification of Query Distribution}

\begin{figure}[htbp]
    \begin{minipage}[c]{0.55\textwidth}
        \centering
        \begin{subfigure}{0.48\linewidth}
            \centering
            \includegraphics[trim=0 0 65pt 0, clip, height=0.95\linewidth]{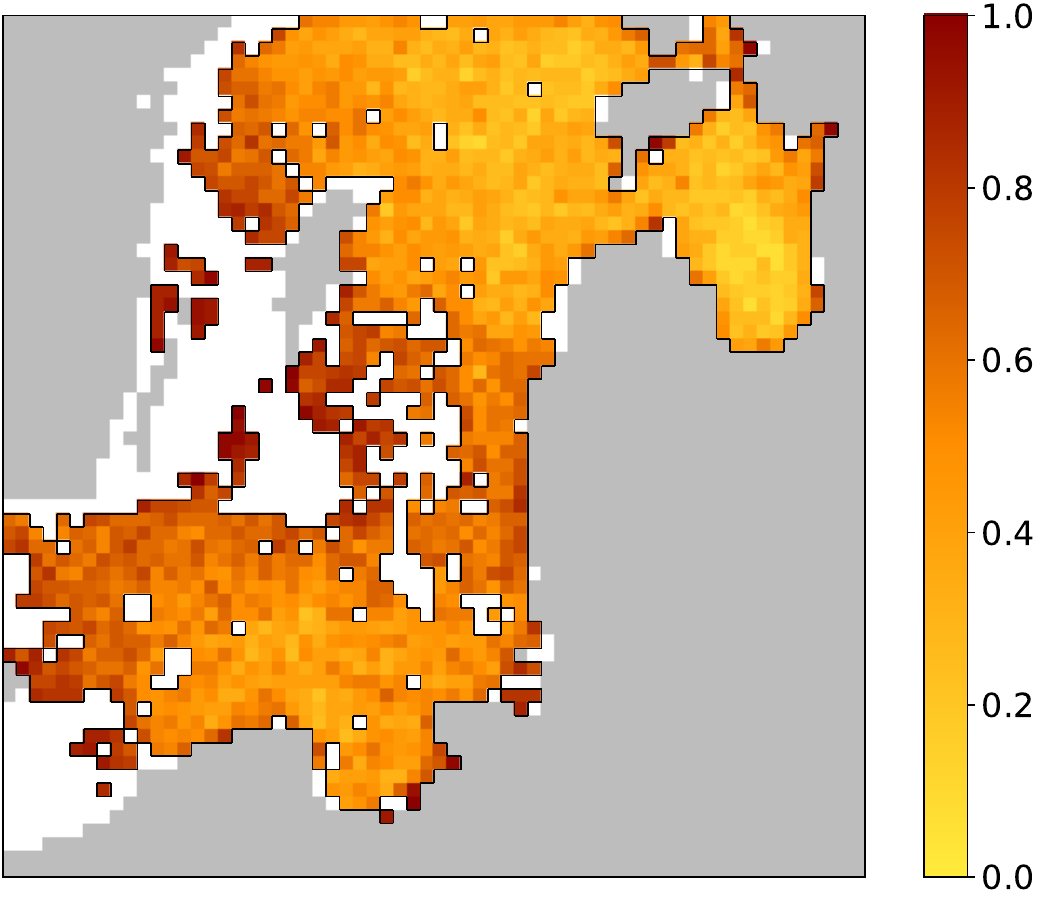}
            \caption{Ours}
            \label{fig:zhengwen_nano_bfn_qry_heatmap_appendix}
        \end{subfigure}\hfill
        \begin{subfigure}{0.48\linewidth}
            \centering
            \includegraphics[height=0.95\linewidth]{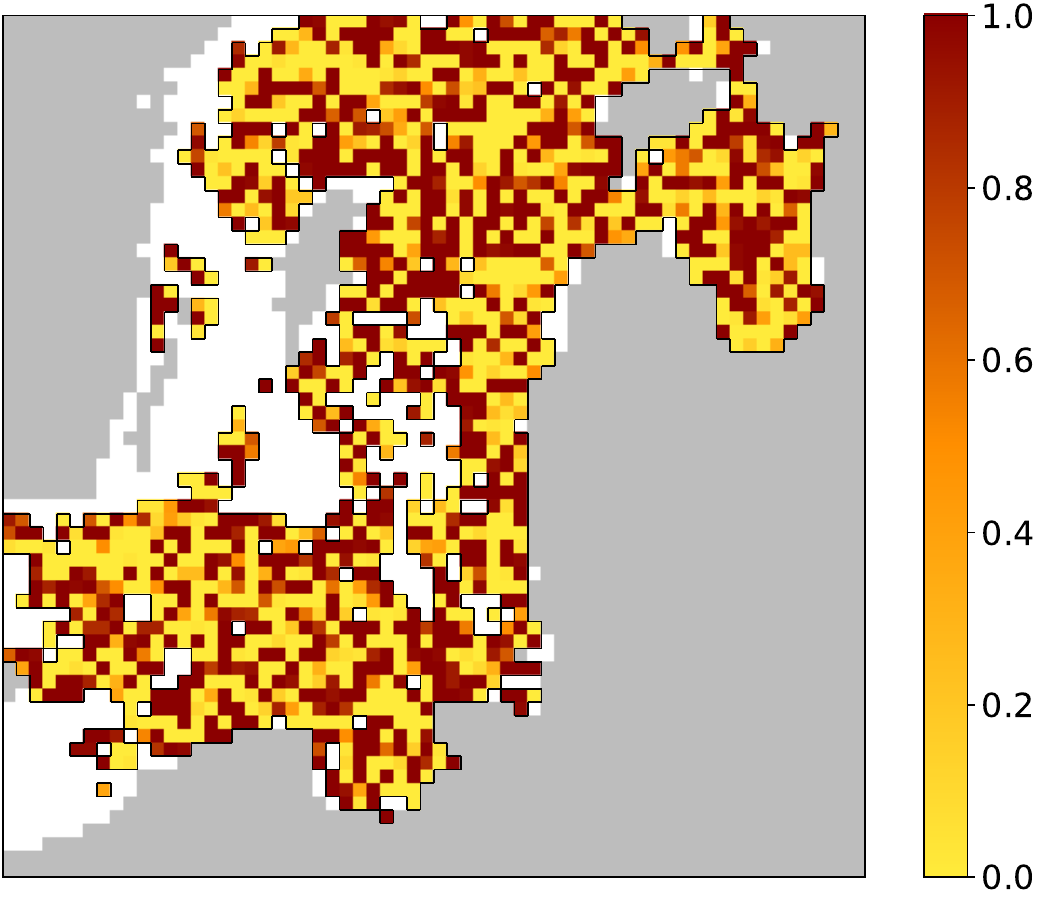}
            \caption{Random Select}
            \label{fig:zhengwen_nano_random_qry_heatmap}
        \end{subfigure}
        \caption{Query probability distributions of our method and naive random selection strategy.}
        \label{fig:zhengwen_qry_heatmap}
    \end{minipage} \hfill 
    \begin{minipage}[c]{0.4\textwidth}
        \centering
        \captionof{table}{Comparison of different context-query partitioning strategies. Here the baseline ``Unconditional'' refers to BFN unconditional generation which reflects the result of Theorem \ref{thm:uniform_query_exposure}.}
        \label{tab: mask method comparison}
        \resizebox{\linewidth}{!}{
            \begin{tabular}{c|cc}
            \toprule
            Methods & MSE $\downarrow$ & PSNR $\uparrow$\\
            \midrule
            Pixel-level       & 0.263 & 17.10 \\
            Block-wise        & 0.251 & 16.43 \\
            Saliency-driven   & 0.317 & 16.28 \\
            Empirical         & 0.225 & 16.62 \\
            Unconditional     & 0.216 & 16.78 \\
            Ours & \textbf{0.193} & \textbf{17.83} \\
            \bottomrule
            \end{tabular}
        }
    \end{minipage}
\end{figure}

To validate that our observation-aligned guided partitioning strategy strictly preserves the positivity of query probabilities as established in Theorem~\ref{thm:strict_positivity_under_ratio_guided}, we design an empirical study to quantify the spatial distribution of query assignments.
\paragraph{Experiment design.} We first train a binary classifier to learn to predict the probability of each pixel being assigned to the query region given only the context mask $\bm{M}_\text{ctx}$. Then we repeatedly sample from our conditionally guided Bayesian Flow Network (BFN) under the intersection constraint $C_k$ (regulated by our retention ratio $\rho$). By generating an ensemble of context-query partitions for the same observation field, we compute the empirical query probability, denoted as $P((\bm{M}_\text{qry})_i=1|C_k)$, across the entire spatial domain.

\paragraph{Verification of strictly positive query probabilities.} To explicitly demonstrate the effectiveness of our partitioning mechanism, we compare the mean $P(qry)$ spatial distribution of our method against a naive random selection baseline in Fig.~\ref{fig:zhengwen_qry_heatmap}. As shown, the naive strategy yields a highly fragmented distribution plagued by ``zero-query'' regions (where $P(qry)$ approaches 0), which inevitably leads to localized generative collapse. In contrast, it demonstrates that the query probability distribution of our method is strictly positive throughout the spatial domain. This empirical evidence directly corroborates Theorem~\ref{thm:strict_positivity_under_ratio_guided}: by dynamically anchoring mask generation to the learned physical prior without perfectly memorizing the deterministic constraints, our framework mathematically and practically guarantees that no valid spatial dimension is systematically ignored during the partitioning process. An extended comparative analysis across additional datasets is provided in Appendix~\ref{app:uq_qry}. Extensive experiments show that the induced positive query exposure consistently maintains non-negligible magnitudes for each dimension in practice, rather than merely non-zero in expectation.

\paragraph{Heuristic reliability indicator.}  In Appendix~\ref{app:learning dynamics basis and analysis}, we detail the relationship between query probability and gradient updates, illustrating that the $P(qry)$ distribution can provide a heuristic indication of the reconstruction reliability during inference.

\subsection{Ablation Study}\label{sec:experiment-ablation}

\paragraph{Context mask selection strategy comparison.} Proper selection of $\bm{M}_{\text{ctx}}$ is crucial to prevent information gaps. Tab.~\ref{tab: mask method comparison} reveals that existing baselines face a critical trade-off: heuristic partitions (Pixel/Block-wise) lose physical coherence, non-adaptive methods ignore spatial topology, and saliency-driven approaches sacrifice large-scale structural anchors, worsening global PSNR. Our strategy resolves trade-offs via cross-entropy guidance, which adaptively optimizes context proportion and placement to balance global anchors with local structure for superior reconstruction (Appendix~\ref{app:context_mask_selection_strategy_comparison}). Additionally, we perform controlled comparisons to further isolate the contribution of the partitioning mechanism itself and examine the robustness across different backbones (Appendix~\ref{app:ablation study}).

\paragraph{Sampling method comparison.}
To reconstruct complete ocean fields from sparse observations, the sampling strategy must effectively bridge the deterministic observed regions and the stochastic generated regions. We evaluate our approach against several sampling methods in Appendix~\ref{app:sampling_method_comparison}.



\begin{figure}[tb]
    \centering
    \begin{minipage}[c]{0.32\textwidth}
        \centering
        \small Black Sea CHL\vspace{1mm}\\
        \begin{minipage}[t]{0.48\linewidth}
            \centering
            \includegraphics[width=\linewidth]{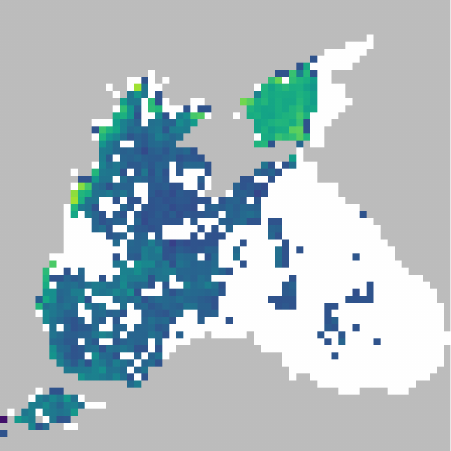}\\
            \scriptsize (a) GT
        \end{minipage}\hfill
        \begin{minipage}[t]{0.48\linewidth}
            \centering
            \includegraphics[width=\linewidth]{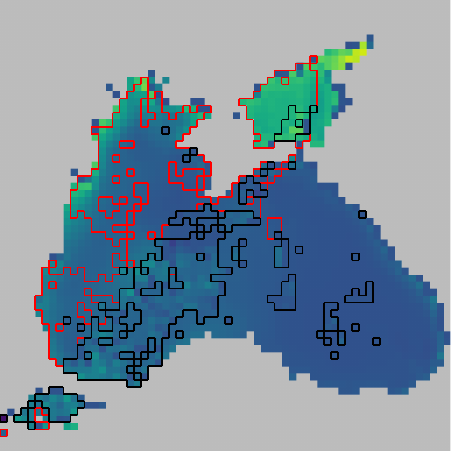}\\
            \scriptsize (b) Imputation
        \end{minipage}
    \end{minipage}\hfill
    \begin{minipage}[c]{0.32\textwidth}
        \centering
        \small Global Ocean SSS\vspace{1mm}\\
        \begin{minipage}[t]{0.48\linewidth}
            \centering
            \includegraphics[width=\linewidth]{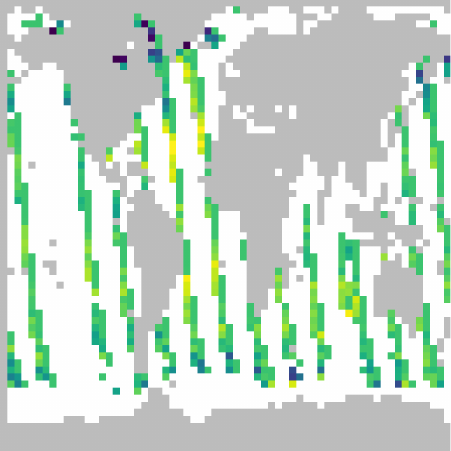}\\
            \scriptsize (a) GT
        \end{minipage}\hfill
        \begin{minipage}[t]{0.48\linewidth}
            \centering
            \includegraphics[width=\linewidth]{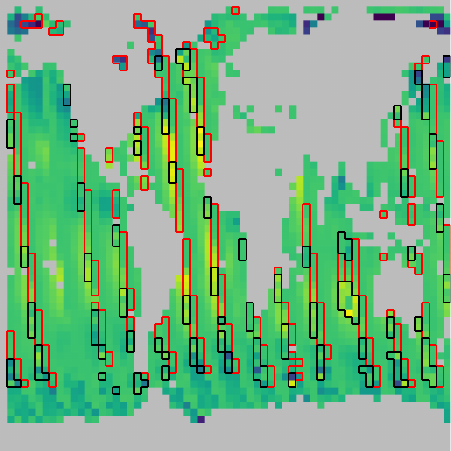}\\
            \scriptsize (b) Imputation
        \end{minipage}
    \end{minipage}\hfill
    \begin{minipage}[c]{0.32\textwidth}
        \centering
        \small Baltic Sea NANO\vspace{1mm}\\
        \begin{minipage}[t]{0.48\linewidth}
            \centering
            \includegraphics[width=\linewidth]{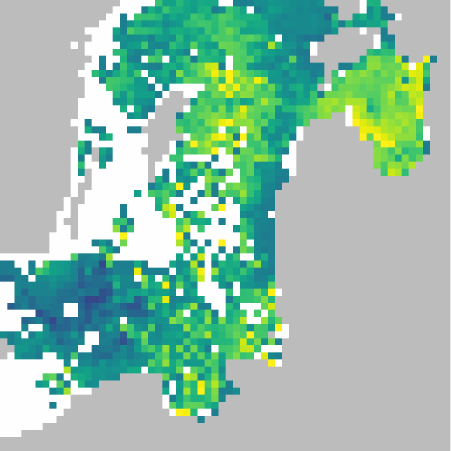}\\
            \scriptsize (a) GT
        \end{minipage}\hfill
        \begin{minipage}[t]{0.48\linewidth}
            \centering
            \includegraphics[width=\linewidth]{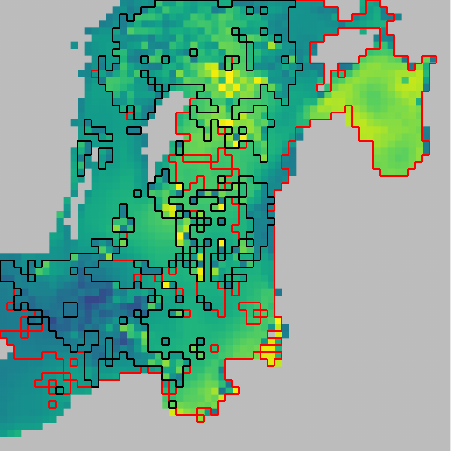}\\
            \scriptsize (b) Imputation
        \end{minipage}
    \end{minipage}
    
    \caption{Visualization of data imputation results on Black Sea CHL, Global Ocean SSS, and Baltic Sea NANO with the resolution of $64 \times 64$ (grouped from left to right). For each dataset, the left panel shows the ground truth (GT) and the right panel displays our imputation result. The gray areas represent land, and the white areas represent the missing regions. The area outlined in black is the visible input region for the model, while the areas framed in red indicate the region used to evaluate the reconstruction results.}
    \label{fig:main_figure_label}
    \vspace{-3mm}
\end{figure}

%% file: 5conclusion.tex
\section{Conclusion}
This paper introduces Observation-Aligned Mask Priors, a novel framework for learning physical fields directly from incomplete data with authentic, structured occlusions. Instead of relying on predefined or heuristic masking, our approach uses BFN to learn the distribution of real observation masks, guiding the generation to align with specific sparse observations to construct dynamic context-query partitions. We mathematically prove that this intersection-based partitioning guarantees a strictly positive query probability for every valid observed dimension, effectively eliminating zero-query dead zones and preventing local generative collapse. Extensive experiments on three real-world oceanographic datasets across multiple resolutions demonstrate that our method consistently outperforms strong diffusion baselines in reconstruction accuracy and global fidelity. Ultimately, this work establishes that learning data-driven occlusion priors is a highly effective alternative to heuristic masking for reconstructing complex physical fields without access to fully observed training data.

%% file: 7appendix.tex
\input{appendix/related}
\input{appendix/bfn_diffusion}

\input{appendix/proofs}

\input{appendix/learning_dynamics}
\input{appendix/experiment}
\input{appendix/context_select_methods}
\input{appendix/sampling_methods}

%% file: appendix/related.tex
\section{Related Work} \label{app: related work}
\subsection{Learning from Incomplete Data}
Recent generative approaches for incomplete data face critical limitations when applied to complex physical fields. Diffusion methods like MissDiff~\citep{ouyang2023missdifftrainingdiffusionmodels} and Generative Geostatistical Modeling~\citep{erdinc2024generativegeostatisticalmodelingincomplete} rely on heuristic masking strategies that implicitly assume Missing At Random (MAR) conditions. Lacking a physics-guided context-query partitioning mechanism, they struggle to capture essential long-range spatial correlations. Similarly, continuous-time mask-aware paradigms such as Impute-MACFM~\citep{liu2025imputemacfmimputationbasedmaskaware} are tailored for tabular data, missing the structural inductive biases required for global fluid continuity. Furthermore, theoretically complex frameworks encounter hard boundaries in physical scenarios. For example, MIRI~\citep{yu2025missingdataimputationreducing} relies on a strict MAR assumption that fundamentally fails in non-random (MNAR) oceanographic occlusions, while incurring prohibitive computational costs. In addition, latent-space diffusion models, such as LSSDM~\citep{liang2024latentspacescorebaseddiffusion}, depend entirely on pre-trained Variational Autoencoders, which currently cannot losslessly encode high-dimensional ocean states without violating governing physical conservation laws. Furthermore,~\citep{zhou2025incomplete} propose a fundamental framework for learning physical systems from incomplete observations by dividing valid regions into context and query components. A more detailed technical overview of this framework is provided in Sec.~\ref{sec:preliminary}.

\subsection{Discrete Generation via Diffusion and Flow Models}
Classic generative models for categorical spaces often struggle to incorporate smooth, gradient-based guidance due to their hard discrete transitions. To address discrete generation, several advanced diffusion and flow frameworks have recently emerged. Multinomial diffusion~\citep{hoogeboom2021argmaxflowsmultinomialdiffusion} extends the standard continuous diffusion framework to categorical distributions by iteratively adding categorical noise through uniform transition matrices. Generalizing this approach, structured discrete diffusion models (D3PM)~\citep{austin2023structureddenoisingdiffusionmodels} introduce specialized transition matrices, such as absorbing and discretized Gaussian states, allowing the forward process to explicitly respect domain priors and ordinal relationships between categories. Concurrently, discrete flow matching~\citep{gat2024discreteflowmatching} introduces a probability path framework utilizing continuous-time Markov chains (CTMCs) to define a probability velocity field for categorical data.

Despite their theoretical elegance, adapting these explicitly discrete frameworks for zero-shot physical imputation is restricted. These discrete frameworks fundamentally rely on classifier-free guidance which strictly necessitates conditional labels during training, or otherwise lack a native, differentiable continuous latent space. Bayesian Flow Networks (BFNs)~\citep{graves2025bayesianflownetworks} solve this hard discrete transition problem by modeling discrete data with continuous parameters on a probability simplex. Recent works~\citep{xue2024unifying} further unify BFNs with stochastic differential equations to highlight their continuous dynamics. However, standard BFNs mostly focus on unconditional generation. In physical dynamic imputation tasks, the model must dynamically align with sparse observations. While some gradient-guided BFNs exist~\citep{qiu2024empower}, they rely on heavy historical trajectory corrections. Therefore, existing discrete generation frameworks lack a mechanism for instantaneously gradient-based guidance, making it difficult to seamlessly guide discrete mask generation with real physical observations.

\subsection{Diffusion for PDE Problem}
Generative modeling for partial differential equations (PDEs) has shifted from purely data-driven mapping to physics-informed diffusion frameworks. Standard models like Fourier Neural Operator (FNO)~\citep{li2021fourierneuraloperatorparametric} effectively learn solution operators in frequency domains but typically require complete, high-resolution training data to capture the underlying physical dynamics. To improve long-term rollouts and capture high-frequency details, diffusion-based solvers such as Dyffusion~\citep{cachay2023dyffusiondynamicsinformeddiffusionmodel} and PDE-Refiner~\citep{lippe2023pderefinerachievingaccuratelong} leverage score-based generation and multi-step refinement, enhancing the statistical fidelity of complex fluid forecasting. Furthermore, Physics-Informed Diffusion Models (PIDMs)~\citep{bastek2025physicsinformeddiffusionmodels} explicitly embed PDE constraints directly into the generative process, reducing residual errors by ensuring that the sampled trajectories strictly adhere to governing physical laws. For inverse problems and partial observations, approaches like DiffusionPDE~\citep{huang2024diffusionpdegenerativepdesolvingpartial} and conditional score-based assimilation models~\citep{shysheya2024conditionaldiffusionmodelspde} incorporate physics-informed losses and observation likelihoods to guide the denoising process. Recent frameworks like Ambient Physics~\citep{majid2026ambientphysicstrainingneural} further attempt to learn the joint distribution of coefficient-solution pairs directly from partial observations by randomly masking known measurements during training.

However, these methods primarily treat sparse inputs as static constraints and often depend on pre-defined, idealistic masking patterns. However, when faced with the complex, non-random occlusions typical of oceanographic data, they lack the adaptive partitioning needed to maintain physical consistency across unknown boundaries. Consequently, bridging the gap between unstructured sparse measurements and structured latent dynamics remains a significant challenge for existing PDE-based diffusion frameworks.

%% file: appendix/bfn_diffusion.tex
\section{Methodological Connections and Distinctions} \label{app: Methodological Connections and Distinctions}
This section contextualizes our proposed mask generation module within the broader landscape of probabilistic generative models, specifically detailing its theoretical connections to and differences from the original BFNs and their recent continuous-time reinterpretation.

\subsection{Connections and Differences with Original BFNs}
Our mask generation module is fundamentally inspired by BFNs introduced by~\citep{graves2025bayesianflownetworks}.

\paragraph{Theoretical connections.}
\begin{itemize}
    \item \textbf{Continuous modeling of discrete data:} Both frameworks address the challenge of modeling discrete variables (such as our binary masks) by operating on a continuous parameter space. In the original BFN, the network inputs for discrete data lie on the probability simplex, making the process natively differentiable. We similarly project our noisy latent variables onto the probability simplex, ensuring the network processes continuous probabilities $\softmax(\bm{x}_t)$ rather than discrete tokens.
    \item \textbf{Data Prediction Objective:} Original BFNs train the network to output the parameters of a clean distribution based on Bayesian-updated input parameters. Our framework adopts a mathematically analogous data prediction formulation, directly outputting predicted class probabilities $\hat{\bm{e}}_{\bm{\theta}}$ to estimate the clean discrete categories.
\end{itemize}

\paragraph{Key differences.}
\begin{itemize}
    \item \textbf{Forward process formulation:} The original BFN framework defines the generative process strictly through Bayesian inference, iteratively updating the parameters of an input distribution based on noisy ``sender'' samples. In contrast, our framework explicitly recasts these dynamics into a continuous diffusion-style forward process. We treat the discrete categories as continuous targets using a scaled logit representation $\bm{x}_0 = K\bm{e}_c$ and directly corrupt this continuous target by injecting Gaussian noise.
    \item \textbf{Guidance and intervention:} BFNs were primarily formulated to optimize data compression via maximum likelihood. Our reinterpretation is specifically designed to facilitate gradient-based intervention in the latent space during sampling. This allows us to introduce a classifier guidance mechanism that dynamically aligns the generated masks with sparse physical observations, a capability not natively explored in the standard BFN formulation.
\end{itemize}

\subsection{Connections and Differences with SDE-Unified BFNs}
Our continuous reformulation of BFNs shares conceptual similarities with recent work by~\citep{xue2024unifying}, which unifies BFNs and Diffusion Models (DMs) through SDEs.

\paragraph{Theoretical connections.}
\begin{itemize}
    \item \textbf{Latent Variable Perspective:} \citep{xue2024unifying} demonstrates that BFNs on discrete data implicitly operate as a linear SDE on a set of continuous latent variables $\bm{z}(t)$, which the original BFN formulation marginalizes out. Our formulation explicitly adopts this continuous latent perspective, defining a forward diffusion process $\bm{x}_t$ on scaled logits.
    \item \textbf{Alignment with Diffusion Objectives:} \citep{xue2024unifying} formally prove that the continuous-time regression loss of a discrete BFN is a reparameterized Denoising Score Matching (DSM) loss. Our framework embraces this alignment, explicitly reparameterizing the noise matching objective into a direct discrete data matching objective $\mathcal{L}_{\text{DM-discrete}}$.
\end{itemize}

\paragraph{Key differences.}
\begin{itemize}
    \item \textbf{Primary Objective:} The primary goal of \citep{xue2024unifying}'s SDE unification is to derive specialized, high-order ODE and SDE solvers (BFN-Solvers) to significantly accelerate the sampling speed of BFNs with a limited number of function evaluations. Conversely, our continuous formulation is not engineered for sampling acceleration, but rather to provide a principled pathway for physics-guided conditioning.
    \item \textbf{Application in the Pipeline:} While \citep{xue2024unifying} aims to improve BFNs as standalone generative models for general text and image datasets, our framework utilizes the BFN specifically as a dynamic spatial partitioning engine. The outputs of our guided BFN dictate the context ($\bm{M}_{\text{ctx}}$) and query ($\bm{M}_{\text{qry}}$) regions for a secondary continuous diffusion model. This ensures the primary model focuses its learning capacity on reconstructing missing physical field structure rather than simply matching unstructured data.
\end{itemize}

\subsection{Connections and Differences with Gradient-Guided BFNs (MolJO)}
Recently, \citep{qiu2024empower} proposed MolJO, which extends BFNs with gradient-based guidance for structure-based molecule optimization (SBMO). While both works leverage BFN's capacity for guided generation, our objectives and architectural choices diverge significantly.

\paragraph{Theoretical connections.}
\begin{itemize}
    \item \textbf{Gradient guidance in $\bm{\theta}$ space:} Both methods capitalize on the insight that BFN's Bayesian posterior parameter $\bm{\theta}$ constitutes a continuous, differentiable space naturally suited for gradient-based optimization. We similarly exploit this property to apply classifier guidance during sampling, allowing the incorporation of external constraints into the mask generation process.
    \item \textbf{Hybrid continuous-discrete modeling:} MolJO addresses the challenge of jointly optimizing continuous coordinates and discrete atom types by guiding their respective $\bm{\theta}$ components. Analogously, our framework represents discrete binary masks via continuous probability simplex parameters (logits), enabling seamless gradient flow for discrete structure generation.
\end{itemize}

\paragraph{Key differences.}
\begin{itemize}
    \item \textbf{Guidance objective and domain:} MolJO targets multi-objective molecular property optimization (e.g., binding affinity, synthesizability) within 3D protein pockets. In contrast, our guidance is physics-driven, designed to align spatial mask sampling with sparse physical observations (e.g., scattered ocean sensor measurements) for scientific data imputation, rather than drug-like property enhancement.
    \item \textbf{Temporal dependency and correction:} MolJO introduces a sophisticated \textit{backward correction} strategy with a sliding window mechanism to align gradients across generative steps and preserve SE(3)-equivariance for molecular geometries. Our framework eschews such historical correction; instead, we focus on instantaneous physics-guided conditioning to dynamically partition spatial contexts for downstream reconstruction.
    \item \textbf{System integration:} While MolJO serves as the terminal generative model producing final molecular structures, our guided BFN operates as a \textit{spatial routing module} within a larger pipeline. The generated masks explicitly dictate context ($\bm{M}_{\text{ctx}}$) and query ($\bm{M}_{\text{qry}}$) regions for a secondary continuous diffusion model, rather than constituting the end-product themselves.
\end{itemize}

%% file: appendix/proofs.tex
\section{Theorems and Proofs} \label{app: Theorems and Proofs}

\input{appendix/proofs/softmax}
\input{appendix/proofs/bfn_noise_data_reparameterization}

\input{appendix/proofs/theorem1}
\input{appendix/proofs/theorem2}

%% file: appendix/proofs/softmax.tex
\subsection{Representational Sufficiency of the Softmax Projection} \label{app: Representational Sufficiency of the Softmax Projection}

We first characterize functions that depend only on softmax probabilities, then show that the optimal data predictor in our Gaussian latent model has this property. Throughout, $C$ denotes the number of classes, $K>0$ the logit scaling factor, and $d$ the number of pixels. Our implementation uses $C=2$ and $K=5.4$, giving a true-class probability $\softmax(K\bm{e}_c)_c\approx0.9955$ for the binary case. This is just an embedding choice. 

\begin{theorem}[Softmax factorization] \label{thm: BFN net softmax}
Let $h:\mathbb{R}^{C}\to\mathbb{R}^{m}$ satisfy $h(\bm{x}+a\bm{1})=h(\bm{x})$ for every $\bm{x}\in\mathbb{R}^{C}$ and $a\in\mathbb{R}$. Define the relative interior of the probability simplex by
\begin{equation}
    \Delta^{C-1}_{\circ}\coloneqq\operatorname{ri}(\Delta^{C-1})
    =\left\{\bm{p}\in\mathbb{R}^{C}:p_j>0,\ \sum_{j=1}^{C}p_j=1\right\}.
\end{equation}
There exists a unique function $g:\Delta^{C-1}_{\circ}\to\mathbb{R}^{m}$ such that $h(\bm{x})=g(\softmax(\bm{x}))$ for all $\bm{x}\in\mathbb{R}^{C}$.
\end{theorem}

\begin{proof}
For $\bm{p}\in\Delta^{C-1}_{\circ}$, set $g(\bm{p})=h(\log\bm{p})$, where the logarithm acts elementwise. Since
\begin{equation}
    \log\softmax(\bm{x})
    =\bm{x}-\log\left(\sum_{j=1}^{C}e^{x_j}\right)\bm{1},
\end{equation}
shift invariance gives $g(\softmax(\bm{x}))=h(\bm{x})$. Moreover, $\softmax(\log\bm{p})=\bm{p}$ for every $\bm{p}\in\Delta^{C-1}_{\circ}$, so any such factorization must satisfy $g(\bm{p})=h(\log\bm{p})$. This proves uniqueness on the relative interior.
\end{proof}

\paragraph{Boundary of the simplex.}
Finite logits never map to the boundary, so the theorem does not determine boundary values. If a continuous extension to the closed simplex exists, it is unique, but continuity of $h$ alone does not guarantee its existence. For example, $h(\bm{x})=x_1-x_2$ is continuous and shift invariant, whereas its factorization $g(\bm{p})=\log(p_1/p_2)$ has no finite continuous extension to the entire boundary. For $d$ pixels, the same proof applies blockwise: invariance under an independent shift within each pixel's $C$ logits gives a unique factorization through the product $(\Delta^{C-1}_{\circ})^d$.

\begin{lemma}[Shift invariance of the optimal data predictor] \label{lem: bfn posterior shift}
Let $c\in\{1,\dots,C\}$ have prior probabilities $\pi_j=P(c=j)$, and let $\bm{x}_0=K\bm{e}_c$. At a fixed time $t$, suppose
\begin{equation}
    \bm{x}_t=\alpha_tK\bm{e}_c+\sigma_t\bm{\epsilon},
    \qquad \bm{\epsilon}\sim\mathcal{N}(\bm{0},\bm{I}_C),
\end{equation}
where $\bm{\epsilon}$ is independent of $c$ and $\sigma_t>0$. Then the class posterior and the optimal squared-error data predictor $D_t^*(\bm{x})=\mathbb{E}[\bm{x}_0\mid\bm{x}_t=\bm{x}]$ are invariant under $\bm{x}\mapsto\bm{x}+a\bm{1}$. The same conclusion holds for a full mask under independent Gaussian noise on the stacked logits, with arbitrary dependencies among mask pixels and independent shifts within each pixel's logits.
\end{lemma}

\begin{proof}
Expanding the Gaussian likelihood and cancelling factors common to all classes gives
\begin{equation} \label{eq: bfn class posterior}
    P(c=j\mid\bm{x}_t=\bm{x})
    =\frac{\pi_j\exp(\lambda_t x_j)}{\sum_{\ell=1}^{C}\pi_\ell\exp(\lambda_t x_\ell)},
    \qquad \lambda_t=\frac{\alpha_tK}{\sigma_t^2}.
\end{equation}
Replacing every $x_j$ by $x_j+a$ multiplies the numerator and denominator by the same positive factor $\exp(\lambda_t a)$. The posterior therefore remains unchanged, as does
\begin{equation} \label{eq: bfn invariant data predictor}
    D_t^*(\bm{x})=K\big(P(c=j\mid\bm{x}_t=\bm{x})\big)_{j=1}^{C}.
\end{equation}
For a full mask, let $\bm{c}=(c_1,\dots,c_d)$ have any joint prior $\pi(\bm{c})$, and let $x_{i,j}$ denote the logit for class $j$ at pixel $i$. The same expansion yields
\begin{equation}
    P(\bm{c}\mid\bm{x}_t=\bm{x})
    \propto\pi(\bm{c})\exp\left(\lambda_t\sum_{i=1}^{d}x_{i,c_i}\right).
\end{equation}
Under $x_{i,j}\mapsto x_{i,j}+a_i$, every configuration receives the common factor $\exp(\lambda_t\sum_i a_i)$, which cancels upon normalization. Thus the joint posterior, its pixelwise marginals, and the posterior mean of the stacked target are invariant. No factorization of $\pi(\bm{c})$ is required.
\end{proof}

Combining Lemma~\ref{lem: bfn posterior shift} with Theorem~\ref{thm: BFN net softmax} justifies parameterizing the data predictor as $K\hat{\bm{e}}_{\bm{\theta}}(t,\softmax(\bm{x}_t))$. At the population optimum, each output block is the posterior class-probability vector. The input softmax probabilities need not themselves equal this posterior. The representational sufficiency claim concerns the optimal data predictor; the full score and noise predictors require the explicit latent terms derived next.

%% file: appendix/proofs/bfn_noise_data_reparameterization.tex
\subsection{Reparameterization of BFN from Noise Matching to Discrete Data Matching} \label{app: Reparameterization of BFN from Noise Matching to Discrete Data Matching}

We derive the discrete data matching objective in Eq.(\ref{eq: bfn data matching}) using the shift-invariant data predictor established in Appendix~\ref{app: Representational Sufficiency of the Softmax Projection}. Let $C$ be the number of classes and $K>0$ an independent logit scale ($C=2$ and $K=5.4$ in our implementation). For one pixel, $c\in\{1,\dots,C\}$, $\bm{e}_c\in\mathbb{R}^{C}$, and $\bm{x}_0=K\bm{e}_c$. The forward process is
\begin{equation} \label{eq: bfn forward}
    \bm{x}_t=\alpha_t\bm{x}_0+\sigma_t\bm{\epsilon},
    \qquad \bm{\epsilon}\sim\mathcal{N}(\bm{0},\bm{I}_C),
\end{equation}
with $\bm{\epsilon}$ independent of $c$. All formulas below apply at nondegenerate training times with $\alpha_t>0$ and $\sigma_t>0$; singular endpoints require separate limits. For a full mask, we stack $d$ one-hot vectors, replace $\bm{I}_C$ by $\bm{I}_{d \times C}$, and sum squared errors over all pixel-class coordinates.

\paragraph{Population-optimal predictors.}
For a positive time weight, squared-error noise and data matching have conditional-expectation minimizers~\citep{zheng2023improved}:
\begin{subequations} \label{eq: diffusion optimal}
\begin{align}
    \bm{\epsilon}_t^*(\bm{x})
    &=\mathbb{E}[\bm{\epsilon}\mid\bm{x}_t=\bm{x}]
    =-\sigma_t\bm{s}_t^*(\bm{x}), \label{eq: noise predictor}\\
    D_t^*(\bm{x})
    &=\mathbb{E}[\bm{x}_0\mid\bm{x}_t=\bm{x}]
    =\frac{\bm{x}+\sigma_t^2\bm{s}_t^*(\bm{x})}{\alpha_t}, \label{eq: data predictor}
\end{align}
\end{subequations}
where $\bm{s}_t^*(\bm{x})=\nabla_{\bm{x}}\log p_t(\bm{x})$ is the true score. Lemma~\ref{lem: bfn posterior shift} establishes $D_t^*(\bm{x}+a\bm{1})=D_t^*(\bm{x})$. Rearranging Eq.(\ref{eq: data predictor}) therefore gives
\begin{equation} \label{eq: bfn score shift}
\begin{aligned}
    \bm{s}_t^*(\bm{x}+a\bm{1})
    &=\bm{s}_t^*(\bm{x})-\frac{a}{\sigma_t^2}\bm{1},\\
    \bm{\epsilon}_t^*(\bm{x}+a\bm{1})
    &=\bm{\epsilon}_t^*(\bm{x})+\frac{a}{\sigma_t}\bm{1}.
\end{aligned}
\end{equation}
Thus the full score and noise predictors are not shift invariant. A network receiving only $\softmax(\bm{x})$ cannot represent either of them without an additional explicit dependence on $\bm{x}$.

\paragraph{Data and noise parameterizations.}
Theorem~\ref{thm: BFN net softmax} and Lemma~\ref{lem: bfn posterior shift} instead justify the data parameterization
\begin{equation} \label{eq: bfn data parameterization}
    D_{\bm{\theta}}(t,\bm{x})
    =K\hat{\bm{e}}_{\bm{\theta}}(t,\softmax(\bm{x})),
\end{equation}
where the network outputs a class-probability vector in $\Delta^{C-1}$ at each pixel. Define the corresponding noise predictor by
\begin{equation} \label{eq: bfn noise parameterization}
    \bm{\epsilon}_{\bm{\theta}}(t,\bm{x})
    \coloneqq\frac{\bm{x}-\alpha_tD_{\bm{\theta}}(t,\bm{x})}{\sigma_t}.
\end{equation}
The explicit $\bm{x}/\sigma_t$ term preserves the required shift behavior. This family contains the population-optimal noise predictor whenever the data network represents the posterior mean; the claim does not require an arbitrary noise predictor to be shift invariant.

\paragraph{Exact loss reparameterization.}
For every training sample and every parameter value, Eqs.(\ref{eq: bfn forward}) and~(\ref{eq: bfn noise parameterization}) imply
\begin{equation}
    \bm{\epsilon}_{\bm{\theta}}(t,\bm{x}_t)-\bm{\epsilon}
    =\frac{\alpha_tK}{\sigma_t}
    \left(\bm{e}_c-\hat{\bm{e}}_{\bm{\theta}}(t,\softmax(\bm{x}_t))\right).
\end{equation}
Consequently,
\begin{equation} \label{eq: bfn exact loss equivalence}
\begin{aligned}
    \mathcal{J}_{\mathrm{noise}}(\bm{\theta})
    &\coloneqq\mathbb{E}\left[w_{\mathrm{NM}}(t)
    \|\bm{\epsilon}_{\bm{\theta}}(t,\bm{x}_t)-\bm{\epsilon}\|^2\right]\\
    &=\mathbb{E}\left[w_{\mathrm{NM}}(t)\frac{\alpha_t^2K^2}{\sigma_t^2}
    \|\hat{\bm{e}}_{\bm{\theta}}(t,\softmax(\bm{x}_t))-\bm{e}_c\|^2\right].
\end{aligned}
\end{equation}
This is exactly Eq.(\ref{eq: bfn data matching}) when
\begin{equation} \label{eq: bfn loss weight}
    w(t)=w_{\mathrm{NM}}(t)\frac{\alpha_t^2K^2}{\sigma_t^2}.
\end{equation}
These identities assume finite expectations and the same time-sampling distribution. Other strictly positive time weights have the same pointwise population minimizer when the predictor is unrestricted, but need not give identical objectives or the same optimum within a finite-capacity model.

\paragraph{Score estimation for sampling.}
The corresponding score estimate is
\begin{equation} \label{eq: bfn score estimate}
    \bm{s}_{\bm{\theta}}(t,\bm{x})
    =-\frac{\bm{\epsilon}_{\bm{\theta}}(t,\bm{x})}{\sigma_t}
    =\frac{\alpha_tK\hat{\bm{e}}_{\bm{\theta}}(t,\softmax(\bm{x}))-\bm{x}}{\sigma_t^2}.
\end{equation}
It equals $\nabla_{\bm{x}}\log p_t(\bm{x})$ at the population optimum, and otherwise provides an estimate for diffusion sampling. For any network parameters, it obeys the score transformation in Eq.(\ref{eq: bfn score shift}) because its learned data-prediction component is shift invariant. This separation of a learned probability predictor and an explicit linear latent term is consistent with the discrete BFN score parameterization of~\citet{xue2024unifying}.

%% file: appendix/proofs/theorem1.tex
\subsection{Strict Positivity of Query Probabilities via Mask Intersection} \label{app: Strict Positivity of Query Probabilities via Mask Intersection}
\begin{assumption}[Intersection Coverage Completeness] \label{assump: Intersection Coverage Completeness}
    Let $S = \{\bm{x} \in \{0,1\}^d \mid P(\bm{x}) > 0\}$ denote the support of the distribution $P(\bm{M})$. For any valid context mask $\bm{m}$ (i.e., $\exists \bm{x}, \bm{y} \in S$ such that $\bm{x} \odot \bm{y} = \bm{m}$) and any spatial dimension $i$ where $\bm{m}_i = 0$, there exists at least one pair of masks $(\bm{M}_a, \bm{M}_b) \in S \times S$ such that:
    \begin{enumerate}
        \item $\bm{M}_a \odot \bm{M}_b = \bm{m}$ (Intersectability)
        \item $(\bm{M}_a)_i = 1$ (Local Coverage)
    \end{enumerate}
\end{assumption}

\uniformqueryexposure*
\begin{proof}
By the definition of conditional probability, we can expand the target expression as:
\begin{equation}
    P((\bm{M}_{\text{qry}})_i = 1 \mid \bm{M}_{\text{ctx}} = \bm{m}) = \frac{P((\bm{M}_{\text{qry}})_i = 1 \text{ and } \bm{M}_{\text{ctx}} = \bm{m})}{P(\bm{M}_{\text{ctx}} = \bm{m})}
\end{equation}
First, we establish that the denominator is strictly positive. Since $\bm{m}$ is defined as a valid context mask, there exist $\bm{x}, \bm{y} \in S$ such that $\bm{x} \odot \bm{y} = \bm{m}$. By the definition of the support set $S$, $P(\bm{x}) > 0$ and $P(\bm{y}) > 0$. By independence condition, the probability of this specific combination is:
\begin{equation}
    P(\bm{M}_1 = \bm{x}, \bm{M}_2 = \bm{y}) = P(\bm{x}) P(\bm{y}) > 0
\end{equation}
Therefore, the marginal probability of the context mask is bounded from below by this specific combination, implying $P(\bm{M}_{\text{ctx}} = \bm{m}) > 0$.

Next, we analyze the numerator, which represents the joint probability of two events occurring simultaneously: $\bm{M}_1 \odot \bm{M}_2 = \bm{m}$ and $(\bm{M}_{\text{qry}})_i = 1$. We substitute the definition of the query mask into the second condition:
\begin{equation}
    (\bm{M}_{\text{qry}})_i = (\bm{M}_1)_i \cdot (1 - (\bm{M}_{\text{ctx}})_i)
\end{equation}
Given the condition $\bm{M}_{\text{ctx}} = \bm{m}$ and the premise that we are evaluating an unobserved dimension where $\bm{m}_i = 0$, the equation simplifies to:
\begin{equation}
    (\bm{M}_{\text{qry}})_i = (\bm{M}_1)_i \cdot (1 - 0) = (\bm{M}_1)_i
\end{equation}
Thus, the numerator fundamentally computes the probability of the joint event: $\{\bm{M}_1 \odot \bm{M}_2 = \bm{m} \text{ and } (\bm{M}_1)_i = 1\}$.

By Assumption~\ref{assump: Intersection Coverage Completeness}, because $\bm{m}$ is a valid context mask and $\bm{m}_i = 0$, there is guaranteed to exist a specific pair of masks $(\bm{M}_a, \bm{M}_b) \in S \times S$ satisfying $\bm{M}_a \odot \bm{M}_b = \bm{m}$ and $(\bm{M}_a)_i = 1$.

Since $\bm{M}_a \in S$ and $\bm{M}_b \in S$, their independent marginal probabilities are strictly positive ($P(\bm{M}_a) > 0$ and $P(\bm{M}_b) > 0$). Under the independence condition, the probability of sampling this exact pair is:
\begin{equation}
    P(\bm{M}_1 = \bm{M}_a, \bm{M}_2 = \bm{M}_b) = P(\bm{M}_a) P(\bm{M}_b) > 0
\end{equation}

Because we have found at least one valid combination of masks in the distribution that perfectly satisfies the conditions of the numerator, and the probability of sampling this combination is strictly greater than zero, the total probability of the numerator must also be strictly positive:
\begin{equation}
    P((\bm{M}_{\text{qry}})_i = 1 \text{ and } \bm{M}_{\text{ctx}} = \bm{m}) \geq P(\bm{M}_1 = \bm{M}_a, \bm{M}_2 = \bm{M}_b) > 0
\end{equation}

Finally, dividing a strictly positive numerator by a strictly positive denominator yields a strictly positive result:
\begin{equation}
    \frac{P((\bm{M}_{\text{qry}})_i = 1 \text{ and } \bm{M}_{\text{ctx}} = \bm{m})}{P(\bm{M}_{\text{ctx}} = \bm{m})} > 0
\end{equation}
Hence, $P((\bm{M}_{\text{qry}})_i = 1 \mid \bm{M}_{\text{ctx}} = \bm{m}) > 0$, completing the proof.
\end{proof}

%% file: appendix/proofs/theorem2.tex
\subsection{Strict Positivity of Query Probabilities under Ratio-Guided Partitioning} \label{app: Strict Positivity of Query Probabilities under Ratio-Guided Partitioning}

\begin{assumption}[Constrained Coverage Completeness] \label{assump: Constrained Coverage Completeness}
 Let $\bm{M} \in \{0,1\}^d$ be a fixed incomplete observation mask. Let $C_k$ denote the hard constraint that a conditionally generated mask $\hat{\bm{M}}$ must share exactly $k$ valid intersecting pixels with the observation, i.e., $|\hat{\bm{M}} \odot \bm{M}| = k$, where $0 < k < |\bm{M}|$. Let $S_{\text{guided}} = \{\hat{\bm{M}} \in \{0,1\}^d \mid P(\hat{\bm{M}} \mid C_k) > 0\}$ be the support set of the conditionally guided mask distribution. We assume that the physical mask prior does not deterministically collapse on any single spatial coordinate under this volume constraint. Specifically, for any spatial dimension $i$ where $\bm{M}_i = 1$, there exists at least one mask $\hat{\bm{M}}_a \in S_{\text{guided}}$ such that $(\hat{\bm{M}}_a)_i = 0$.
\end{assumption}

This assumption is physically well-founded. In authentic oceanographic observations, missing patterns, such as those driven by dynamic cloud cover or shifting satellite trajectories, exhibit immense spatial diversity. Constraining the intersection volume $k$ merely restricts the \textit{capacity} of the context region without eliminating the inherent \textit{stochasticity} of the physical prior. A well-trained generative model naturally synthesizes multiple distinct, physically plausible occlusion topologies that satisfy the exact same volume constraint, ensuring that no single spatial dimension is deterministically anchored across all valid generations.

\strictpositivityratioguided*
\begin{proof} 
By the definitions of the context and query masks, we can explicitly express the query mask $\bm{M}_{\text{qry}}$ in terms of the fixed observation $\bm{M}$ and the conditionally generated mask $\hat{\bm{M}}$:
\begin{equation}
    \bm{M}_{\text{qry}} = \bm{M} \odot (1 - \hat{\bm{M}} \odot \bm{M})
\end{equation}
Since $\bm{M}$ is a binary mask (i.e., $\bm{M}_i \in \{0, 1\}$), the element-wise multiplication yields $\bm{M} \odot \bm{M} = \bm{M}$. Thus, the equation simplifies to:
\begin{equation}
    \bm{M}_{\text{qry}} = \bm{M} \odot (1 - \hat{\bm{M}})
\end{equation}
We are evaluating a specific spatial dimension $i$ where the authentic observation is valid, meaning $\bm{M}_i = 1$. Substituting this into the simplified query equation for dimension $i$, we obtain:
\begin{equation}
    (\bm{M}_{\text{qry}})_i = \bm{M}_i \cdot (1 - \hat{\bm{M}}_i) = 1 \cdot (1 - \hat{\bm{M}}_i) = 1 - \hat{\bm{M}}_i
\end{equation}
This deterministic relationship reveals that under a fixed observation $\bm{M}$, the event $\{ (\bm{M}_{\text{qry}})_i = 1 \}$ is mathematically equivalent to the event $\{ \hat{\bm{M}}_i = 0 \}$. Consequently, their probabilities under the guided distribution are identical:
\begin{equation}
    P((\bm{M}_{\text{qry}})_i = 1 \mid C_k) = P(\hat{\bm{M}}_i = 0 \mid C_k)
\end{equation}
According to Assumption~\ref{assump: Constrained Coverage Completeness}, because $0 < k < |\bm{M}|$, the constrained support set $S_{\text{guided}}$ preserves sufficient generative diversity. Therefore, there exists at least one valid mask $\hat{\bm{M}}_a \in S_{\text{guided}}$ that satisfies $(\hat{\bm{M}}_a)_i = 0$.

By the definition of the support set $S_{\text{guided}}$, any mask within it has a strictly positive probability of being sampled:
\begin{equation}
    P(\hat{\bm{M}} = \hat{\bm{M}}_a \mid C_k) > 0
\end{equation}
The marginal probability $P(\hat{\bm{M}}_i = 0 \mid C_k)$ is formulated as the sum of the probabilities of all generated masks in $S_{\text{guided}}$ where the $i$-th dimension equals 0. Since $\hat{\bm{M}}_a$ is a member of this set and possesses a strictly positive probability, the entire sum must be strictly bounded away from zero:
\begin{equation}
    P(\hat{\bm{M}}_i = 0 \mid C_k) = \sum_{\hat{\bm{m}} \in S_{\text{guided}}, \hat{\bm{m}}_i = 0} P(\hat{\bm{M}} = \hat{\bm{m}} \mid C_k) \geq P(\hat{\bm{M}} = \hat{\bm{M}}_a \mid C_k) > 0
\end{equation}
It trivially follows that:
\begin{equation}
    P((\bm{M}_{\text{qry}})_i = 1 \mid C_k) > 0
\end{equation}
This completes the proof. 
\end{proof}

\begin{remark}[Connection to Expected Gradients and Theorem~\ref{thm:uniform_query_exposure}] \label{remark: uniform_query_exposure}
Readers may notice a subtle shift in the condition from $\bm{M}_{\text{ctx}}$ in Theorem~\ref{thm:uniform_query_exposure} to $C_k$ in Theorem~\ref{thm:strict_positivity_under_ratio_guided}. This reflects the fundamental architectural shift from random bipartite splitting to observation-anchored generation. In the unconditional case (Theorem~\ref{thm:uniform_query_exposure}), $\bm{M}_{\text{ctx}}$ is highly stochastic, so we must mathematically guarantee that \textit{any} valid randomly formed context leaves a non-zero query probability to avoid zero-gradient dead zones. However, under the guided framework, the actual observation $\bm{M}$ is fixed for a given data sample. Both $\bm{M}_{\text{ctx}}$ and $\bm{M}_{\text{qry}}$ are jointly and deterministically derived from a single generated mask $\hat{\bm{M}}$ subject to the guidance constraint $C_k$. Therefore, guaranteeing the marginal probability $P((\bm{M}_{\text{qry}})_i = 1 \mid C_k) > 0$ is mathematically sufficient. It ensures that across multiple training epochs, every valid observation point has a strictly positive frequency of being queried ($p_i > 0$ in Eq.(\ref{eq: linear relation}), thereby satisfying the prerequisite for non-vanishing expected gradients.
\end{remark}

We also need to clarify that the exact-overlap constraint used in Theorem~\ref{thm:strict_positivity_under_ratio_guided} is an analytical abstraction for characterizing the support condition required for positive query exposure, rather than an exact sampling operation implemented by Algorithm~\ref{alg:bfn_mask}. In practice, Algorithm~\ref{alg:bfn_mask} constructs stochastic anchors through Bernoulli retention and applies a soft cross-entropy guidance signal during BFN sampling. For the high-dimensional observation masks considered in our experiments, each sample contains a large number of valid observed coordinates, so Bernoulli retention produces a non-empty stochastic anchor set with overwhelming probability. In particular, the event that all valid observations are simultaneously discarded is vanishingly unlikely under our practical settings; hence, the required condition $k>0$ is effectively always satisfied in practice rather than constituting an additional restrictive assumption. At the same time, the stochastic anchor set remains non-empty without overwhelmingly high probability of deterministically retaining any individual observed coordinate across all sampling instances. Thus, the practical procedure preserves the stochastic non-degeneracy that motivates the support condition in Theorem~\ref{thm:strict_positivity_under_ratio_guided}, rather than enforcing a fixed-cardinality constraint explicitly. The theorem can therefore be interpreted as identifying a sufficient support-level condition under which observation-guided partitioning preserves positive query exposure, while the empirical query-frequency analysis in Appendix~\ref{app:uq_qry} verifies that this condition is maintained by the implemented soft-guidance sampler in practice.

%% file: appendix/learning_dynamics.tex
\section{Learning Dynamics Basis and Analysis }
\label{app:learning dynamics basis and analysis}
\paragraph{Preliminary property.} As illustrated in previous work~\citep{zhou2025incomplete}, Eq.(\ref{eq: strict positive condition}) is fundamentally dictated by the learning dynamics, where the expected squared gradient magnitude for feature $i$ scales linearly with its query probability:
\begin{equation}  \label{eq: linear relation}
    \mathbb{E}\left[\left(\frac{\partial \mathcal{L}}{\partial (\bm{u}_{\bm{\phi}})_i}\right)^2\right] = 4p_i \mathbb{E}\left[ \left((\bm{u}_{\bm{\phi}})_i - (\bm{u}_{\text{obs}})_i \right)^2 \Big| (\bm{M}_{\text{qry}})_i = 1\right],
\end{equation}
where $p_i = P((\bm{M}_{\text{qry}})_i = 1 | \bm{M}_{\text{ctx}})$ directly controls the frequency of parameter updates. If $p_i = 0$, the given dimension receives no gradient updates, causing the model to predict arbitrary and meaningless values in that region.
\paragraph{Heuristic reliability indicator.} As established in our Eq.(\ref{eq: linear relation}), the expected gradient update for a specific feature dimension scales linearly with its query probability. This dictates a direct relationship between sampling probability and learning frequency: regions with a higher $P(qry)$ are exposed to the loss function more frequently during training.Because the prediction model is optimized more heavily in these high-frequency regions, the $P(qry)$ distribution effectively captures the model's spatial familiarity with the data. Although not a formally calibrated statistical uncertainty metric, $P(qry)$ acts as a strong proxy for epistemic confidence. During inference, regions that exhibited high query probabilities in the training phase can generally be interpreted as having higher predictive reliability. Consequently, this distribution provides an accessible and heuristic indicator to assess the relative trustworthiness of the reconstructed physical fields across different spatial domains.

%% file: appendix/experiment.tex
\section{Supplementary Experiments}
\input{appendix/experiments/dataset_setting}

\input{appendix/experiments/method_comparison}
\input{appendix/experiments/implement_details}
\input{appendix/experiments/complete_results}
\input{appendix/experiments/ablation}
\input{appendix/experiments/context_select_comparison}
\input{appendix/experiments/sampling_method_comparison}
\input{appendix/experiments/qry_uq}
\input{appendix/experiments/stochastic_anchor}
\input{appendix/experiments/visualization}

%% file: appendix/experiments/dataset_setting.tex
\subsection{Dataset Settings} \label{app: dataset settings}
All real-world oceanographic datasets utilized in our experiments are sourced from the E.U. Copernicus Marine Environment Monitoring Service (CMEMS). We specifically selected Level-3 (L3) satellite observation products. Unlike Level-4 (L4) products, which are already synthetically interpolated and gap-filled, L3 products contain unadulterated, authentic missingness patterns caused by real-world physical constraints such as cloud occlusions, sensor swath gaps, and orbital trajectories. This makes them ideal testbeds for evaluating the imputation of complex, structured missing fields.

\textbf{Black Sea CHL} (Black Sea Chlorophyll-a Dataset)~\citep{letraon:hal-03405376}. This dataset is derived from the CMEMS product \texttt{OCEANCOLOUR\_BLK\_BGC\_L3\_MY\_009\_153} (\url{https://doi.org/10.48670/moi-00303}). It provides daily multi-sensor merged objective observations of Chlorophyll-a (CHL) concentration in the Black Sea at a 1-km spatial resolution. Chlorophyll-a distributions exhibit highly complex, non-linear spatial gradients driven by mesoscale eddies and coastal currents. Coupled with the extensive and contiguous cloud occlusions typical of the Black Sea region, this dataset challenges an imputation model's ability to recover fine-grained, localized ecological structures from sparse and fragmented observations. We picked up approximately 10,000 samples for our experiments.

\textbf{Baltic Sea NANO} (Baltic Sea Nano-Phytoplankton Dataset)~\citep{letraon:hal-03405376}. Sourced from the CMEMS product \texttt{OCEANCOLOUR\_BAL\_BGC\_L3\_MY\_009\_133} (\url{https://doi.org/10.48670/moi-00296}), this dataset provides daily regional ocean color observations, specifically targeting nano-phytoplankton concentrations at a 1-km resolution. The Baltic Sea possesses a highly irregular topology with numerous archipelagos and narrow straits. Consequently, the missingness patterns are governed not only by atmospheric occlusions but also by complex, rigid coastal boundaries. This presents a unique challenge for imputation tasks, requiring models to reconstruct physically plausible fields that strictly respect intricate geographical constraints rather than applying naive spatial smoothing across landmasses. The dataset used in our experiments contains roughly 5,000 samples.

\textbf{Global Ocean SSS} (Global Sea Surface Salinity Dataset)~\citep{letraon:hal-03405376}. This dataset utilizes the CMEMS global product \texttt{MULTIOBS\_GLO\_PHY\_SSS\_L3\_MYNRT\_015\_014} (\url{https://doi.org/10.48670/mds-00368}). It offers daily global Sea Surface Salinity measurements at a coarser 25-km spatial resolution. In contrast to regional datasets, this global dataset involves macroscopic spatial correlations dominated by large-scale ocean circulation and climatic forcing. Furthermore, the missingness typology fundamentally shifts from primarily cloud-induced irregular holes to distinct satellite orbital tracks and systematic inter-swath gaps. This dataset is ideal for validating the scalability of imputation algorithms and their capacity to maintain global physical consistency under highly structured, geometry-driven masking topologies. We selected around 6,000 samples as our dataset.

%% file: appendix/experiments/method_comparison.tex
\subsection{Baseline and Method Comparison} \label{app: baseline_and_method_comparison}
Here we also considered other baselines but ultimately omitted them from direct empirical comparisons. We excluded AmbientGAN~\citep{Bora2018AmbientGANGM} and MisGAN~\citep{li2018misgan} because purely adversarial training suffers from severe instability, making it exceedingly difficult to converge on highly sparse datasets. Furthermore, we reviewed recent generative frameworks tailored for partial observations, including Ambient Physics~\citep{majid2026ambientphysicstrainingneural},  MSM~\citep{park2025measurementscorebaseddiffusionmodel}, and the context-query diffusion approach~\citep{zhou2025incomplete}. They fundamentally rely on similar overly simplistic or manually designed mask distributions (e.g., uniform random dropout or regular block masking). These idealistic assumptions usually fail to capture the complex, missing-not-at-random (MNAR) topology of authentic physical occlusions. And we equivalently demonstrate their core idea by comparing different masking strategies in Sec.~\ref{sec:experiment-ablation}. A more direct and comprehensive comparison with the most related methods is shown in Tab.~\ref{tab: other method comparison}.

\begin{table}[htbp]
\centering
\caption{Comparison with the most related methods. We use $\times$ to denote that it cannot be achieved or is extremely costly or engineering-difficult. And use $\checkmark$ to indicate that this can be achieved.}
\label{tab: other method comparison}
\resizebox{0.8\textwidth}{!}{%
\begin{tabular}{@{}cccc@{}}
\toprule
\textbf{Methods} & \textbf{MNAR Assumption} & \textbf{Non-heuristic} & \textbf{High Dimensional} \\
\midrule
\textbf{ImputeMACFM}~\citep{liu2025imputemacfmimputationbasedmaskaware} & \checkmark & \texttimes & \texttimes \\
\addlinespace[0.2em]
\textbf{MIRI}~\citep{yu2025missingdataimputationreducing} & \texttimes & \texttimes & \texttimes \\
\addlinespace[0.2em]
\textbf{LSSDM}~\citep{liang2024latentspacescorebaseddiffusion} & \texttimes & \checkmark & \texttimes  \\
\addlinespace[0.2em]
\textbf{AmbientDiff}~\citep{daras2023ambientdiffusionlearningclean} & \texttimes & \checkmark & \checkmark \\
\addlinespace[0.2em]
\textbf{MissDiff}~\citep{ouyang2023missdifftrainingdiffusionmodels} & \texttimes & \checkmark & \checkmark \\
\addlinespace[0.2em]
\textbf{DINDiff}~\citep{barth2024ensemble} & \checkmark & \checkmark & \checkmark \\
\addlinespace[0.2em]
\textbf{Ambient Physics}~\citep{majid2026ambientphysicstrainingneural} & \texttimes & \texttimes & \checkmark \\
\addlinespace[0.2em]
\textbf{MSM}~\citep{park2025measurementscorebaseddiffusionmodel} & \texttimes & \texttimes & \checkmark \\
\addlinespace[0.2em]
\textbf{Ours} & \checkmark & \checkmark & \checkmark \\
\bottomrule
\end{tabular}%
}
\end{table}

%% file: appendix/experiments/implement_details.tex
\subsection{Implementation Details} \label{app: implement_details}
In all experiments, considering both the stability of the guidance and the diversity of the generated data, we empirically set $\rho = 0.8$. Setting the value of $\rho$ too high will result in insufficient diversity of conditional generation, while setting it too low will make it difficult to align with the guidance target. Regarding BFN sampling steps $n$, we find that 10 to 20 ODE sampling steps are usually sufficient to achieve the desired effect. As the number of sampling steps increases, the improvement in effect becomes relatively limited. For guidance scale $w_g$, a moderate interval is between 80 and 200 and usually needs to be adapted to the selected sampling steps. The value of the guidance scale seems larger than that in traditional classifier guidance because we perform global normalization on our guidance objective. This parameter is not particularly sensitive, a reasonable setting is sufficient. 

Our training and test sets are completely separated in time with no overlap, and each has been internally shuffled. This ensures that no data leakage occurs between data from adjacent dates. Also, to explore the best performance of the baseline AmbientDiff, we attempt several combinations of context and query ratio in Tab.~\ref{tab: ambientdiff multi ratio}. 

\begin{table}[ht]
\centering
\caption{Exploration of suitable parameters for AmbientDiff on the Black Sea CHL dataset with the resolution of $64 \times 64$.}
\label{tab: ambientdiff multi ratio}
\footnotesize
\begin{tabular}{cc|cccccc}
\toprule
\multirow{2}{*}{\textbf{Context Ratio}} & \multirow{2}{*}{\textbf{Query Ratio}} & \multicolumn{2}{c}{\textbf{Black Sea CHL}} & \multicolumn{2}{c}{\textbf{Baltic Sea NANO}} & \multicolumn{2}{c}{\textbf{Global Ocean SSS}}    \\ \cmidrule(lr){3-4} \cmidrule(lr){5-6} \cmidrule(lr){7-8} 
&  & MSE $\downarrow$ & PSNR $\uparrow$ & MSE $\downarrow$ & PSNR $\uparrow$ & MSE $\downarrow$ & PSNR $\uparrow$ \\ \midrule
10\% & 10\%   & 0.243 & 16.775 & 0.376 & 15.652 & 0.294 & 17.992 \\ 
30\% & 30\%   & \textbf{0.228} & \textbf{17.717} & \textbf{0.365} & \textbf{15.466} & \textbf{0.286} & \textbf{18.022} \\ 
50\% & 50\%   & 0.275 & 16.926 & 0.439 & 14.776 & 0.292 & 17.957 \\ 
70\% & 70\%   & 0.344 & 16.171 & 0.495 & 14.323 & 0.323 & 17.596 \\ 
90\% & 90\%   & 0.428 & 15.439 & 0.619 & 13.488 & 0.344 & 17.363 \\ 
\bottomrule
\end{tabular}

\end{table}

%% file: appendix/experiments/complete_results.tex
\subsection{Complete Results}
To further demonstrate the scalability and robustness of our framework across different spatial scales, we provide additional experimental results at $128 \times 128$ and $256 \times 256$ resolutions on all three datasets in Tab.~\ref{tab: scale up to 128} and~\ref{tab: scale up to 256}. This demonstrates that our method remains effective and stands out more at high resolutions.

\begin{table}[ht]
\centering
\caption{Performance comparison on physical fields imputation tasks with the resolution of $128 \times 128$ across three real-world datasets.}
\label{tab: scale up to 128}
\footnotesize
\begin{tabular}{c|cccccc}
\toprule
\multicolumn{1}{c}{\multirow{2}{*}{\textbf{Method}}} & \multicolumn{2}{c}{\textbf{Black Sea CHL}}    & \multicolumn{2}{c}{\textbf{Baltic Sea NANO}}   & \multicolumn{2}{c}{\textbf{Global Ocean SSS}}    \\ \cmidrule(lr){2-3} \cmidrule(lr){4-5} \cmidrule(lr){6-7} 
\multicolumn{1}{c}{} & MSE $\downarrow$ & PSNR $\uparrow$ & MSE $\downarrow$ & PSNR $\uparrow$ & MSE $\downarrow$ & PSNR $\uparrow$ \\ \midrule
AmbientDiff   & 0.268 & 20.432 & 0.466 & 15.648 & 0.440 & 19.165 \\ 
DINDiff       & 0.275 & 18.800 & 0.385 & 16.275 & 0.578 & 17.967 \\ 
MissDiff      & 0.706 & 16.702 & 0.753 & 13.480 & 0.683 & 17.342 \\ 
\midrule
\textbf{Ours} & \textbf{0.210} & \textbf{21.020} & \textbf{0.333} & \textbf{16.730} & \textbf{0.379} & \textbf{19.799} \\
\bottomrule
\end{tabular}
\end{table}

\begin{table}[ht]
\centering
\caption{Performance comparison on physical fields imputation tasks with the resolution of $256 \times 256$ across three real-world datasets.}
\label{tab: scale up to 256}
\footnotesize
\begin{tabular}{c|cccccc}
\toprule
\multicolumn{1}{c}{\multirow{2}{*}{\textbf{Method}}} & \multicolumn{2}{c}{\textbf{Black Sea CHL}}    & \multicolumn{2}{c}{\textbf{Baltic Sea NANO}}   & \multicolumn{2}{c}{\textbf{Global Ocean SSS}}    \\ \cmidrule(lr){2-3} \cmidrule(lr){4-5} \cmidrule(lr){6-7} 
\multicolumn{1}{c}{} & MSE $\downarrow$ & PSNR $\uparrow$ & MSE $\downarrow$ & PSNR $\uparrow$ & MSE $\downarrow$ & PSNR $\uparrow$ \\ \midrule
AmbientDiff   & 0.256 & 20.523 & 0.487 & 16.234 & 0.524 & 20.578 \\ 
DINDiff       & 0.266 & 21.389 & 0.449 & 17.222 & 0.567 & 19.610 \\ 
MissDiff      & 0.655 & 16.887 & 0.818 & 14.147 & 0.731 & 19.187 \\ 
\midrule
\textbf{Ours} & \textbf{0.160} & \textbf{21.889} & \textbf{0.320} & \textbf{17.616} & \textbf{0.441} & \textbf{20.801} \\
\bottomrule
\end{tabular}
\end{table}

%% file: appendix/experiments/ablation.tex
\subsection{Supplementary Experiments} \label{app:ablation study}
To make a more fair and clean comparison, we provide more experiments with a strict complementary context-query setup and similar ratios, aligning them with our method. This ensures, to the greatest extent possible, that the ratio of context and query closely matches that of our approach on different datasets. And there is no overlap between context and query, thereby eliminating the information leakage as a potential confounding factor in the comparison. Based on the real-time curves of the context and query ratio recorded during training, we observed that the context proportion fluctuates around 50\% for the Black Sea CHL and Baltic Sea NANO dataset, and around 70\% for the Global Ocean SSS dataset, with the magnitude of fluctuation generally remaining within 10\%. And the query is complementary to the context. The results in Tabs.~\ref{tab:ratio_chl}, \ref{tab:ratio_nano}, and \ref{tab:ratio_sss} clearly demonstrate that the effectiveness of our method is not determined by the ratio. The performance gains achieved by adjusting the ratio are far smaller than those yielded by our method based on all experiments. This makes a clean isolation to support that the bottleneck mainly exists in the context-query partition strategy rather than the context-query ratio. 

To assess the generalizability of our guided partitioning strategy across different neural architectures, we also evaluate our method using two distinct backbones: Vision Transformer (ViT)~\citep{dosovitskiy2021imageworth16x16words} and Karras UNet~\citep{karras2024analyzingimprovingtrainingdynamics}. The results presented in Tab.~\ref{tab: backbone comparison} indicate that while both architectures achieve comparable performance on the Black Sea Chlorophyll-a Dataset at a resolution of 64×64, while UNet achieves superior performance, producing a lower MSE and a higher PSNR compared to ViT. This performance gap suggests that the hierarchical inductive bias of UNet is well-suited for capturing the multi-scale spatial dependencies inherent in oceanic physical fields.

\begin{table}[tb]
    \centering
        \caption{Performance comparison of different sampling strategies on the 64×64 Black Sea CHL Dataset.}
        \label{tab: sample method comparison}
            \begin{tabular}{c|cc}
            \toprule
            Methods & MSE $\downarrow$ & PSNR $\uparrow$ \\
            \midrule
            Direct Projection & 0.193 & 17.825 \\
            Proximal Prediction & \textbf{0.189} & 17.669 \\
            Iterative Conditioning & 0.199 & 18.330 \\
            Resampling Inpainting & 0.199 & 18.330 \\
            Recursive Jump Inpainting & 0.199 & \textbf{18.359} \\
            \bottomrule
            \end{tabular}
\end{table}

\begin{table}[tb]
    \centering
    \caption{Backbone performance comparison on the Black Sea CHL dataset with the resolution of $64 \times 64$.}
    \label{tab: backbone comparison}
    \begin{tabular}{c|cc}
    \toprule
        Methods & MSE $\downarrow$ & PSNR $\uparrow$ \\
        \midrule
            Vit & 0.225 & 16.701 \\
            Unet & \textbf{0.193} & \textbf{17.825} \\
        \bottomrule
    \end{tabular}
    \label{tab:mask-method-comparison}
\end{table}

\begin{table}[t]
\centering
\caption{Comparison on the Black Sea CHL dataset under controlled context-query ratios.}
\label{tab:ratio_chl}
\small
\setlength{\tabcolsep}{3.5pt}
\renewcommand{\arraystretch}{1.15}
\begin{tabular}{@{}lcccc@{}}
\toprule
\textbf{Method}
& \textbf{Context}
& \textbf{Query}
& \textbf{MSE}$\downarrow$
& \textbf{PSNR}$\uparrow$
\\
\midrule
AmbientDiff 
& 50\% & 50\% & 0.263 & 17.097 \\
Ours 
& $\sim$50\% & $\sim$50\% & \textbf{0.193} & \textbf{17.825} \\
\bottomrule
\end{tabular}
\end{table}

\begin{table}[t]
\centering
\caption{Comparison on the Baltic Sea NANO dataset under controlled context-query ratios.}
\label{tab:ratio_nano}
\small
\setlength{\tabcolsep}{3.5pt}
\renewcommand{\arraystretch}{1.15}
\begin{tabular}{@{}lcccc@{}}
\toprule
\textbf{Method}
& \textbf{Context}
& \textbf{Query}
& \textbf{MSE}$\downarrow$
& \textbf{PSNR}$\uparrow$
\\
\midrule
AmbientDiff 
& 50\% & 50\% & 0.441 & 15.021 \\
Ours 
& $\sim$50\% & $\sim$50\% & \textbf{0.294} & \textbf{16.101} \\
\bottomrule
\end{tabular}
\end{table}

\begin{table}[t]
\centering
\caption{Comparison on the Global Ocean SSS dataset under controlled context-query ratios.}
\label{tab:ratio_sss}
\small
\setlength{\tabcolsep}{3.5pt}
\renewcommand{\arraystretch}{1.15}
\begin{tabular}{@{}lcccc@{}}
\toprule
\textbf{Method}
& \textbf{Context}
& \textbf{Query}
& \textbf{MSE}$\downarrow$
& \textbf{PSNR}$\uparrow$
\\
\midrule
AmbientDiff 
& 70\% & 30\% & 0.331 & 17.455 \\
Ours 
& $\sim$70\% & $\sim$30\% & \textbf{0.276} & \textbf{18.192} \\
\bottomrule
\end{tabular}
\end{table}

%% file: appendix/experiments/context_select_comparison.tex
\subsection{Context Mask Selection Strategy Comparison.} \label{app:context_mask_selection_strategy_comparison}
We conduct a comprehensive comparison of our approach with various context mask selection strategies here.
\begin{itemize}
    \item \textbf{Pixel-level Partition:} Each pixel in the observed region is independently sampled using a Bernoulli distribution with a constant probability $p$ (where $p$ is the ratio of required points to total points), producing a spatially uniform subset as $\bm{M}_{\text{ctx}}$.
    \item \textbf{Block-wise Partition:} The observed region is partitioned into a grid of non-overlapping blocks. A subset of these blocks is randomly sampled with a constant probability $p$ (where $p$ is the ratio of selected blocks to total blocks), and their intersection with the ground-truth mask produces a subset as $\bm{M}_{\text{ctx}}$.
    \item \textbf{Saliency-driven Partition:} This heuristic strategy relies on sorting local gradient magnitudes to allocate the available context points. Sample half of $\bm{M}_{\text{ctx}}$ from high-gradient regions and the other half from low-gradient regions.
    \item \textbf{Empirical Distribution Partition:} Randomly sample a real observation mask from the remaining dataset, then consider the intersection of this mask and $\bm{M}_{\text{real}}$ as $\bm{M}_{\text{ctx}}$.
    \item \textbf{Unconditional Prior Partition:} It employs a mask sampled unconditionally by a BFN which was pre-trained to capture the dataset's observation mask distribution (Sec.~\ref{sec: Modeling Mask Distributions via Bayesian Flow Networks}). The intersection of this sampled mask and $\bm{M}_{\text{real}}$ is used as $\bm{M}_{\text{ctx}}$.   
    \item \textbf{Ours:} Our framework introduces a cross-entropy guidance mechanism during the BFN sampling phase to dynamically anchor the generation process (Sec.~\ref{sec: Observation-Aligned Context Mask Generation}). This produces a tailored mask specific to the current data point, and $\bm{M}_{\text{ctx}}$ is the intersection of this conditionally generated mask and $\bm{M}_{\text{real}}$.
\end{itemize}
To make a fair and clean comparison, we provide experiments with a strict complementary context-query setup and under similar context ratios. Based on the real-time curves of the context and query ratio recorded during training, we observed that the context proportion fluctuates around 50\% for the Black Sea CHL ith the magnitude of fluctuation generally remaining within 10\%. Thus, for Pixel-level and Block-wise, we use the context ratio to 0.5 for better comparison.  

%% file: appendix/experiments/sampling_method_comparison.tex
\subsection{Sampling Method Comparison.} \label{app:sampling_method_comparison}
We conduct a comprehensive comparison of our approach with various sampling methods here.
\begin{itemize}
    \item \textbf{Direct Projection:} Perform a single-step projection to approximate the conditional expectation $\mathbb{E}[\bm{u}_{\text{0}} | \bm{u}_{\text{obs}}, \bm{M}]$. By mapping the uncorrupted partial observations directly onto the data manifold via ensemble averaging, it minimizes global reconstruction error but often produces blurry results in high-frequency regions.
    \item \textbf{Proximal Prediction:} A variant of the single-step approach that applies an infinitesimal noise level before projection. This initial noise injection acts as a proximal constraint that aligns with the diffusion model's prior, while the final step still tightly anchors the solution to the deterministic manifold of the observed data, better mitigating blurriness and ensuring higher fidelity to local pixel values.
    \item \textbf{Iterative Conditioning:} Decomposes the complex conditional expectation into a diffusion expectation capturing the global prior and an imputation expectation enforcing structural constraints. By blending the two components through a monotonically increasing weight schedule over multiple steps, it effectively balances global physical dynamics with local observation consistency.
    \item \textbf{Resampling Inpainting:} Harmonize observation boundaries by alternating between denoising and a resampling step where noise is added back to the current estimate. This iterative loop allows the model to better align the generated content with the fixed observed regions through multiple stochastic trials.
    \item \textbf{Recursive Jump Inpainting:} An extension of the Resampling Inpainting framework that utilizes profound temporal jumps back to earlier noise levels. By re-injecting significant variance, the model alleviates local boundary inconsistencies and more effectively reconciles the global physical dynamics with the sparse observation constraints.
\end{itemize}
The quantitative results in Tab.~\ref{tab: sample method comparison} demonstrate a fundamental trade-off between minimizing expectation error and preserving global signal fidelity. The single-step methods achieve the lowest global reconstruction loss as theoretically established in~\citep{zhou2025incomplete}. Iterative approaches, such as Iterative Conditioning and Resampling Inpainting, significantly improve PSNR but albeit at a slight cost to the global loss. This improvement aligns with the theoretical foundation of Resampling Inpainting, that is, the iterative forward-backward diffusion process continuously harmonizes the joint distribution of the known context and the generated content. Notably, the Recursive Jump Inpainting achieves the best overall structural fidelity (highest PSNR). By introducing large noise jumps, this profound re-injection strategy prevents the model from settling into local boundary inconsistencies, ultimately yielding the most seamless and physically natural transitions across the mask boundaries. Detailed algorithm implementation for all sampling methods are provided in Appendix~\ref{app: sampling methods}. 

%% file: appendix/experiments/qry_uq.tex
\subsection{Uncertainty Quantification of Query Distribution}
\label{app:uq_qry}
To further validate the efficacy of our partitioning mechanism, we conduct a comparative analysis of the spatial query probability distribution between our proposed method and a naive random selection baseline across all three datasets. We set the number of ensemble samplings to 200 and conduct multiple experiments to make the results more generalizable and reliable. As shown in Fig.~\ref{fig:comparison_qry_heatmap}, a naive random selection strategy yields a highly fragmented query distribution with numerous ``dead zones'' where the selection probability approaches zero. These ``dead zones'' receive virtually no updates, often leading to localized generative collapse. In contrast, observation-aligned partitioning strategy maintains consistently positive query frequencies across the whole valid spatial locations. Our extensive practical experiments complements the marginal positivity guarantee of Theorem~\ref{thm:strict_positivity_under_ratio_guided} by showing that the induced positive query exposure consistently maintains non-negligible magnitudes globally in practice, rather than merely non-zero in expectation. Consequently, all valid locations can receive effective supervision. This reliable update frequency effectively broadens the model's learning capacity across the physical field.

\begin{figure}[htbp]
    \centering
    \begin{subfigure}{0.28\textwidth}
        \centering
        \caption{Ours}
        \includegraphics[trim=0 0 2.5cm 0, clip, width=\linewidth]{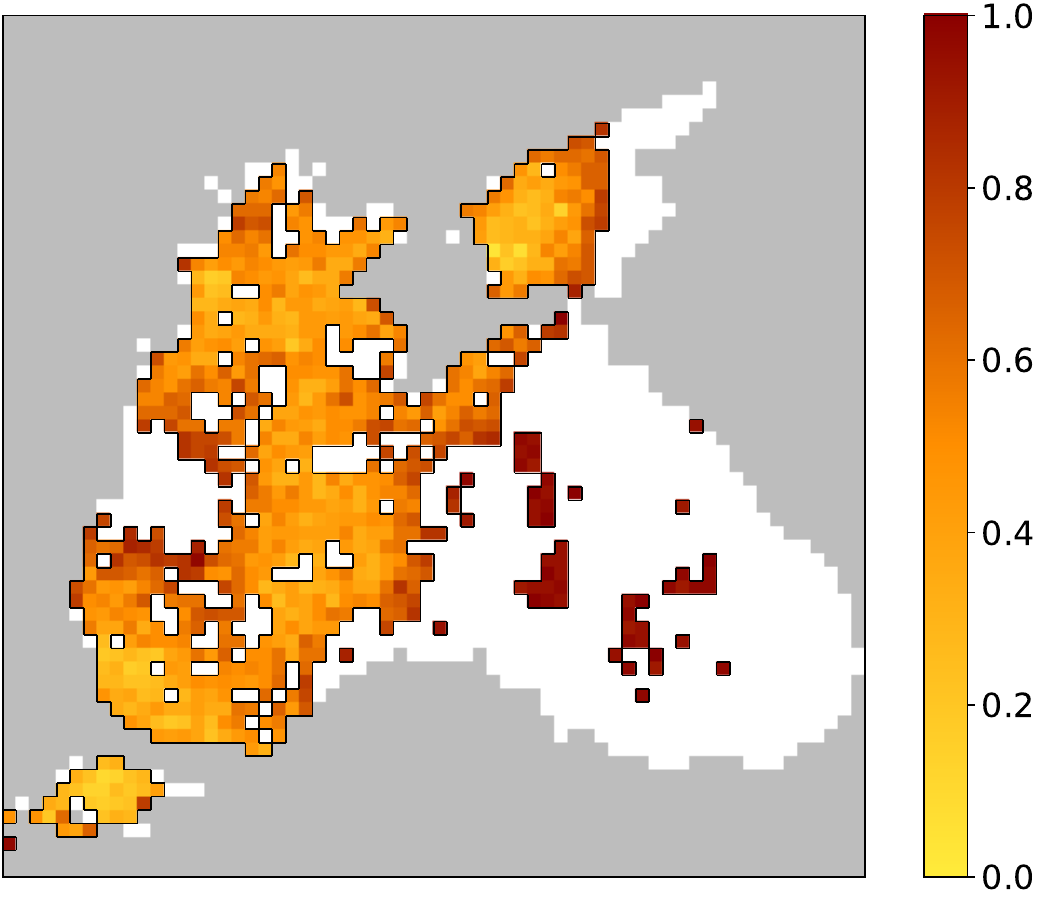}
        \label{fig:black_sea_bfn_qry_heatmap_appendix}
    \end{subfigure}\hfill
    \begin{subfigure}{0.28\textwidth}
        \centering
        \caption{Ours}
        \includegraphics[trim=0 0 2.5cm 0, clip, width=\linewidth]{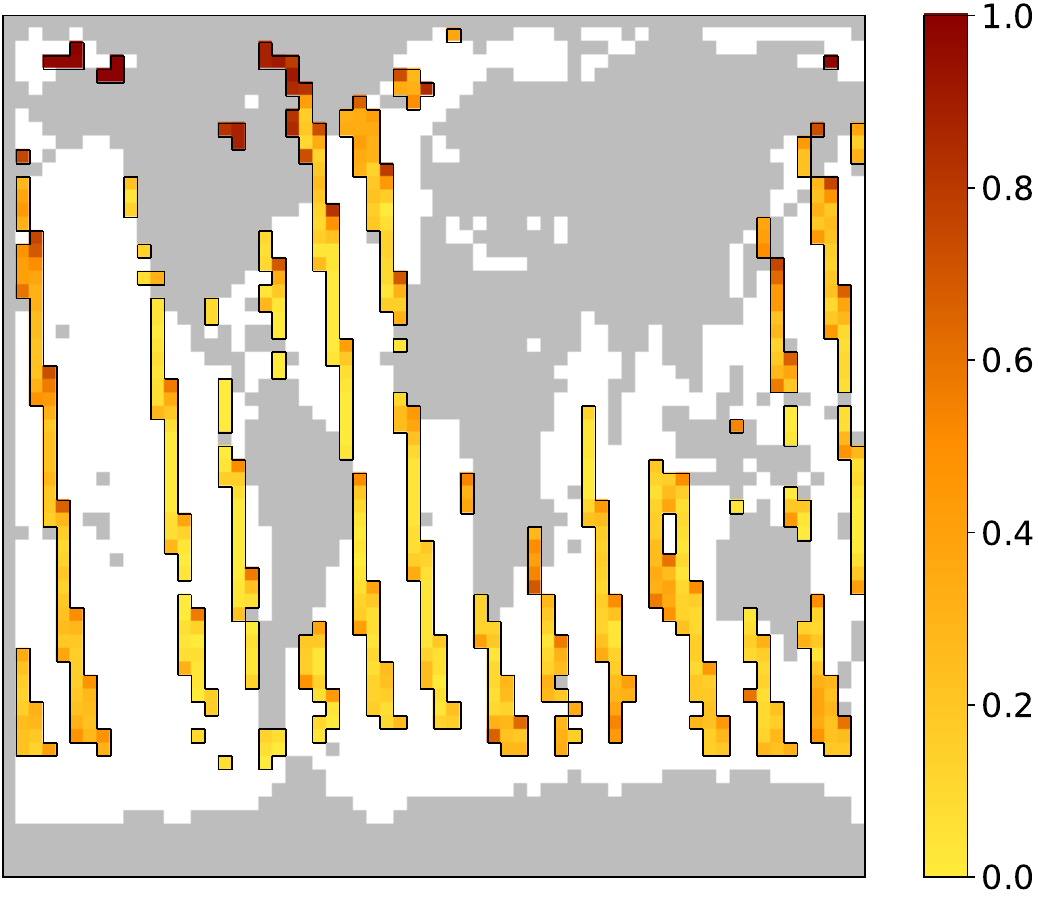}
        \label{fig:smos_bfn_qry_heatmap_appendix}
    \end{subfigure}\hfill
    \begin{subfigure}{0.327\textwidth}
        \centering
        \caption{Ours}
        \includegraphics[width=\linewidth]{figs/query_visual/nano_bfn_qry_heatmap.pdf}
        \label{fig:nano_bfn_qry_heatmap_appendix}
    \end{subfigure}
    
    \vspace{0.5mm} 
    
    \begin{subfigure}{0.28\textwidth}
        \centering
        \caption{Random Select}
        \includegraphics[trim=0 0 2.5cm 0, clip, width=\linewidth]{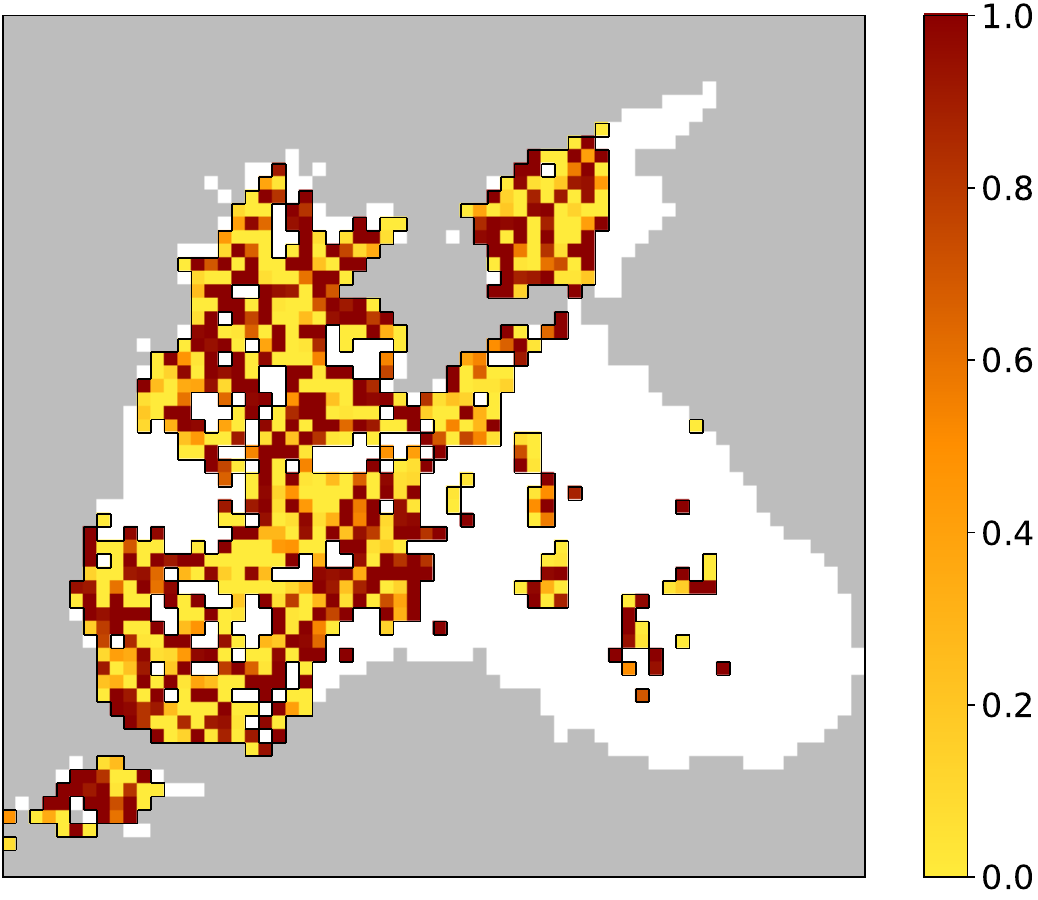}
        \label{fig:black_sea_random_qry_heatmap}
    \end{subfigure}\hfill
    \begin{subfigure}{0.28\textwidth}
        \centering
        \caption{Random Select}
        \includegraphics[trim=0 0 2.5cm 0, clip, width=\linewidth]{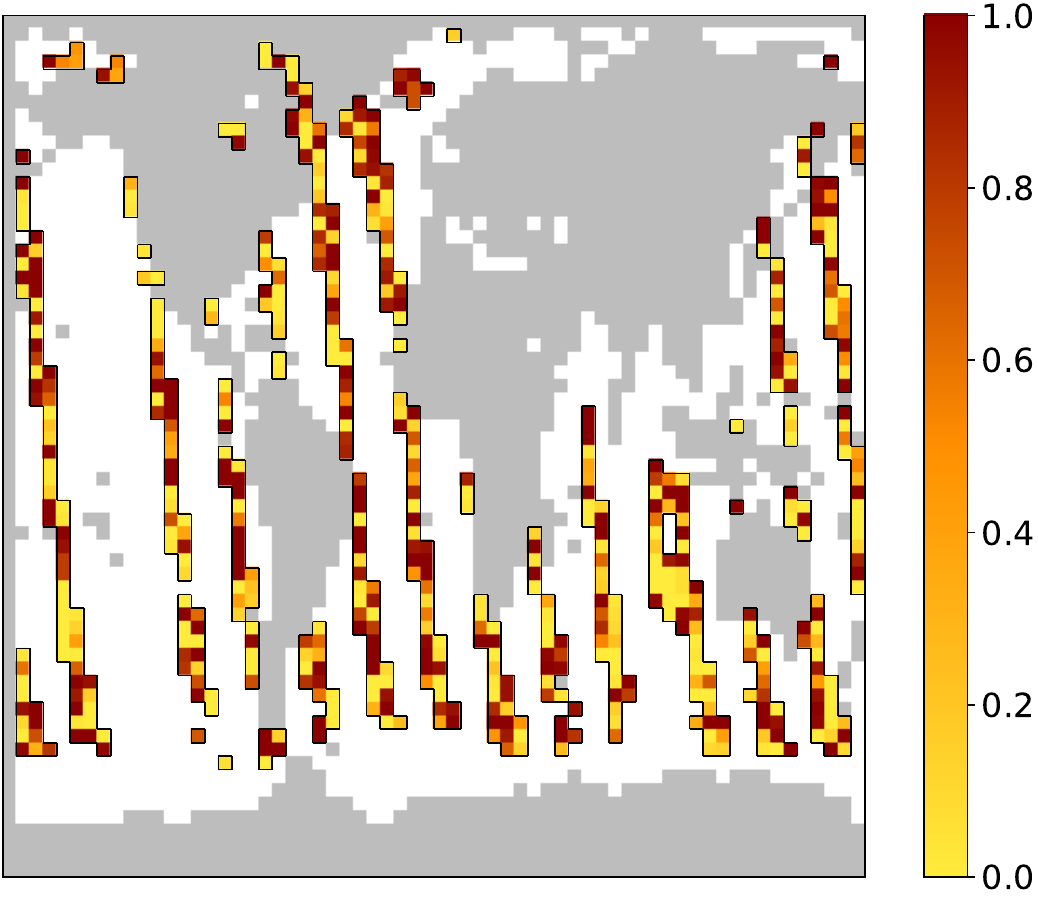}
        \label{fig:smos_random_qry_heatmap}
    \end{subfigure}\hfill
    \begin{subfigure}{0.327\textwidth}
        \centering
        \caption{Random Select}
        \includegraphics[width=\linewidth]{figs/query_visual/nano_random_qry_heatmap.pdf}
        \label{fig:nano_random_qry_heatmap}
    \end{subfigure}
    
    \caption{Comparison of spatial query probability distributions between our proposed method and naive random selection strategy across the Black Sea CHL, Global Ocean SSS, and Baltic Sea NANO datasets with the resolution of $64 \times 64$ (from left to right).}
    \label{fig:comparison_qry_heatmap}
\end{figure}

%% file: appendix/experiments/stochastic_anchor.tex
\subsection{Effectiveness of Stochastic Anchor} \label{app: stochastic anchor}
To demonstrate the effectiveness of the stochastic anchor construction mechanism, we conduct a comparison analysis on $\rho=1$ (i.e. do not use this mechanism ), $\rho=0.8$ and $\rho=0.6$ respectively. As shown in Fig.~\ref{fig:stochastic_anchor_figure} , we visualize the mean and variance of the mask conditionally generated by the BFN after ensemble sampling 200 times. It is easy to observe that without stochastic anchor ($p=1.0$), the sampled masks exhibit minimal variance, which severely limits the spatial diversity of the samples. While, when $p$ is set to $0.6$, the proportion of the context region is reduced, resulting in insufficient information to guide the model effectively. Balancing the need for spatial diversity and sufficient context information, $p=0.8$ provides a better trade-off, ensuring adequate mask variance while preserving enough information.

\begin{figure}[htbp]
    \centering
    \begin{subfigure}{0.28\textwidth}
        \centering
        \caption{Mean}
        \includegraphics[trim=0 0 2cm 0, clip, width=\linewidth]{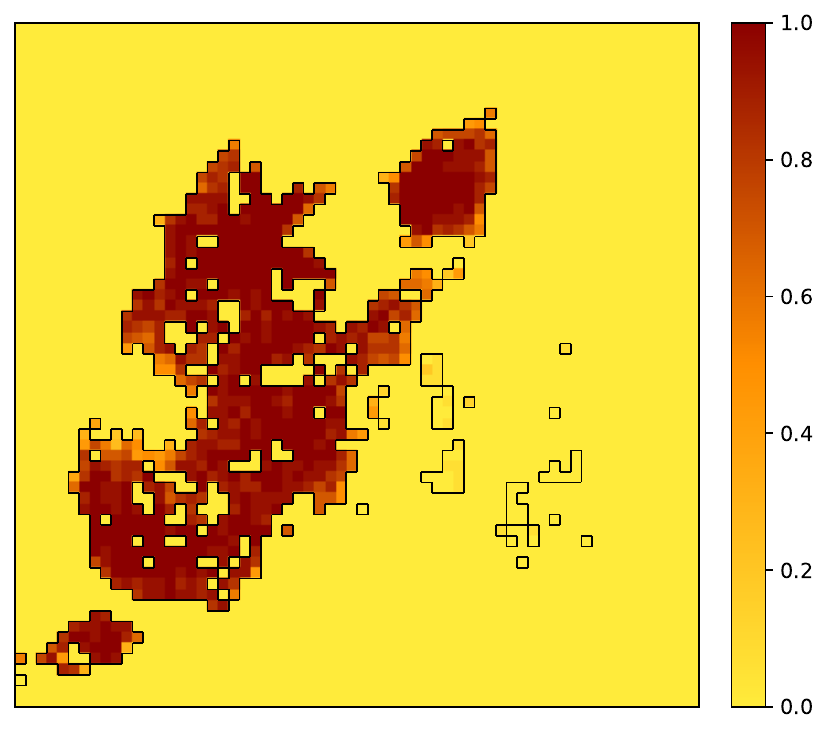}
        \label{fig:mean_context_mask_probability_1}
    \end{subfigure}\hfill
    \begin{subfigure}{0.28\textwidth}
        \centering
        \caption{Mean}
        \includegraphics[trim=0 0 2cm 0, clip, width=\linewidth]{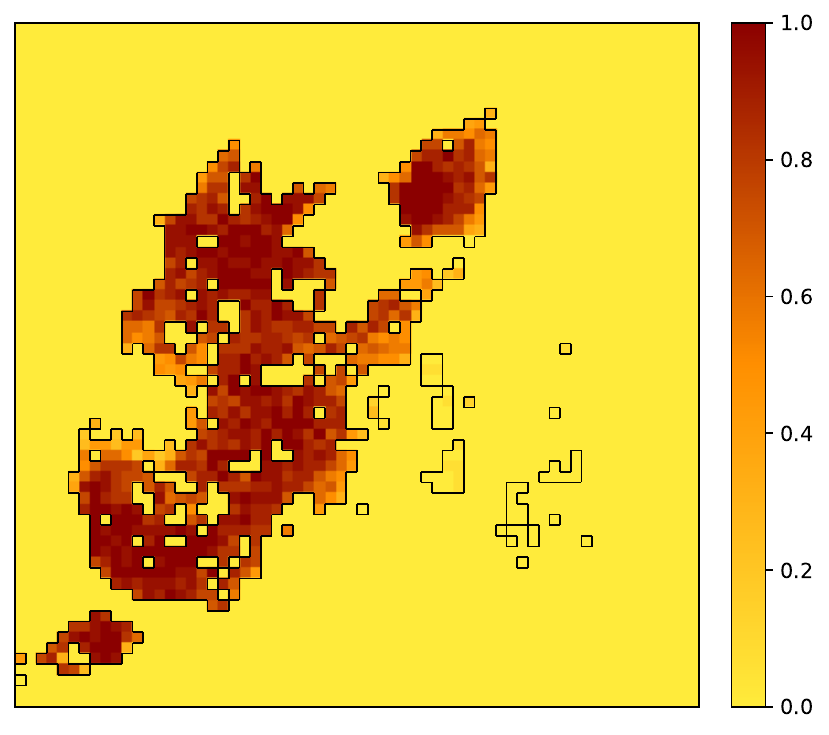}
        \label{fig:mean_context_mask_probability_0.8}
    \end{subfigure}\hfill
    \begin{subfigure}{0.327\textwidth}
        \centering
        \caption{Mean}
        \includegraphics[width=\linewidth]{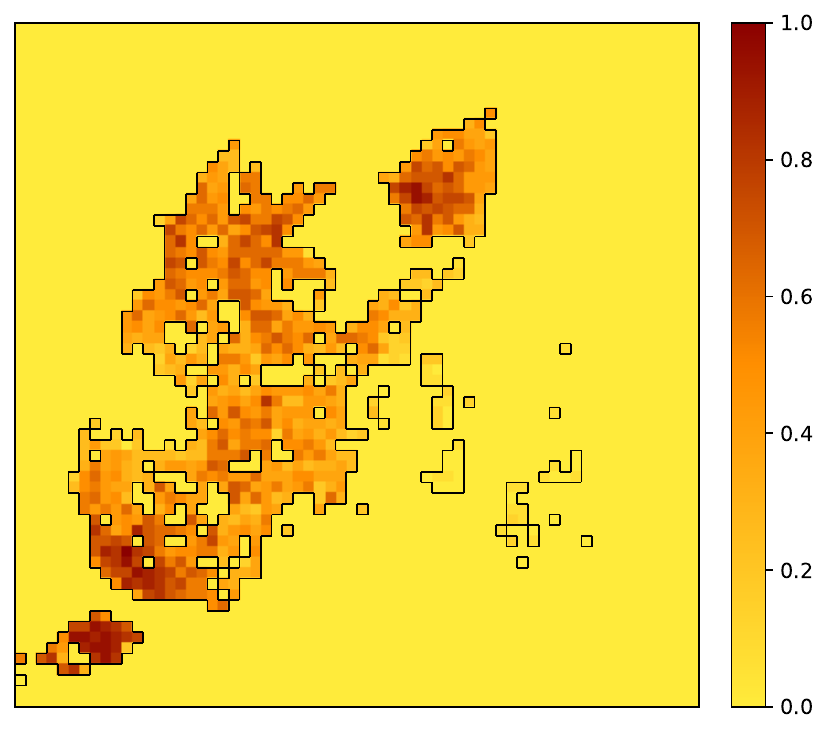}
        \label{fig:mean_context_mask_probability_0.6}
    \end{subfigure}
    
    \vspace{0.5mm} 
    
    \begin{subfigure}{0.28\textwidth}
        \centering
        \caption{Std}
        \includegraphics[trim=0 0 2cm 0, clip, width=\linewidth]{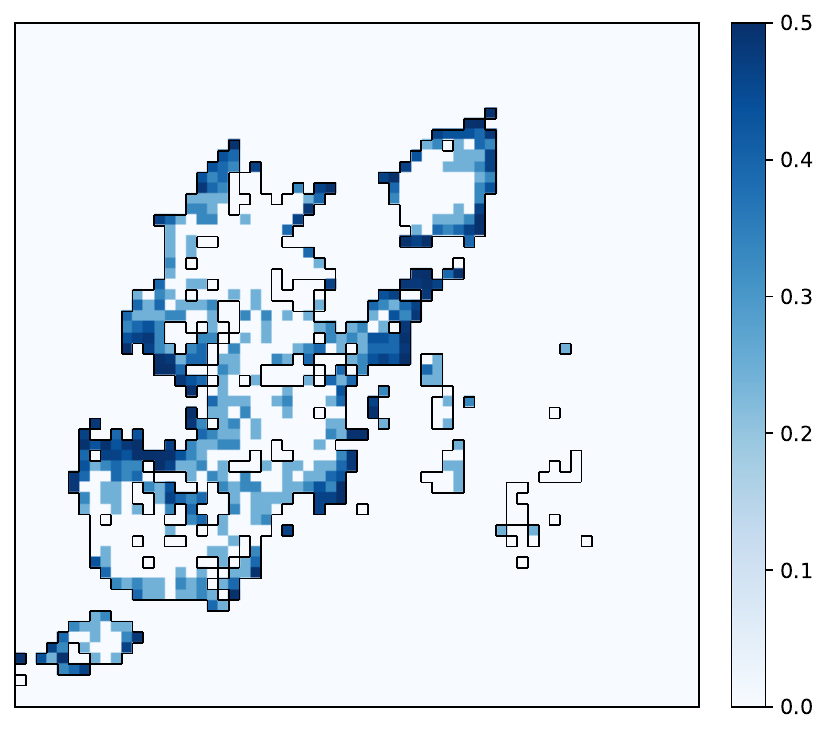}
        \label{fig:std_context_mask_probability_1}
    \end{subfigure}\hfill
    \begin{subfigure}{0.28\textwidth}
        \centering
        \caption{Std}
        \includegraphics[trim=0 0 2cm 0, clip, width=\linewidth]{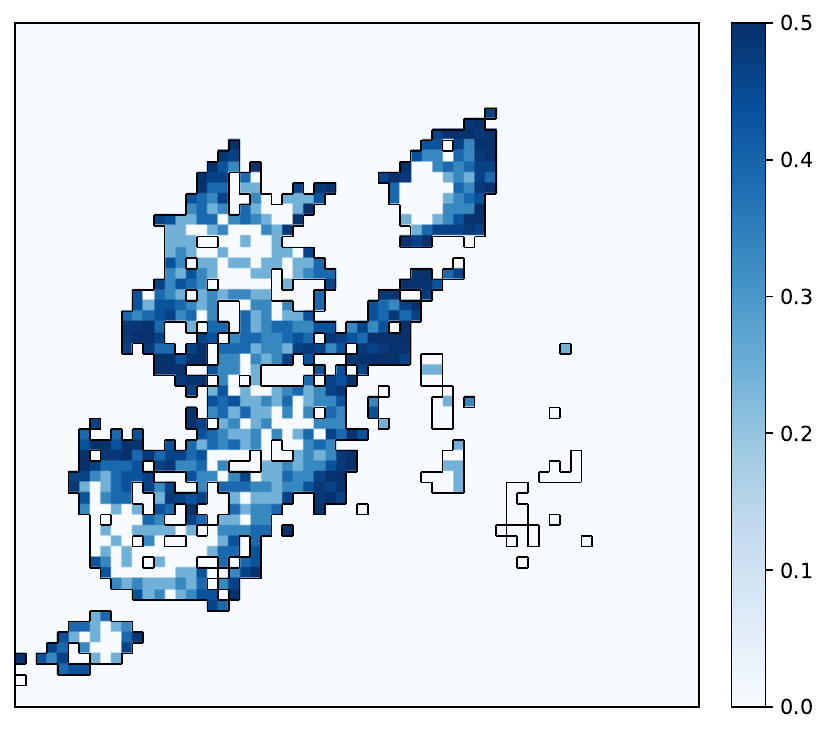}
        \label{fig:std_context_mask_probability_0.8}
    \end{subfigure}\hfill
    \begin{subfigure}{0.327\textwidth}
        \centering
        \caption{Std}
        \includegraphics[width=\linewidth]{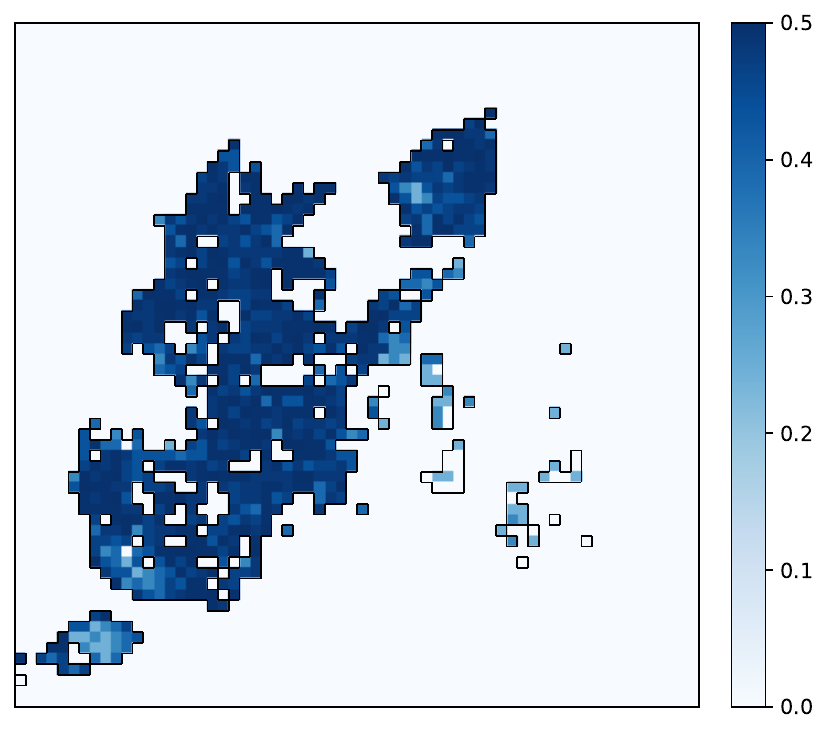}
        \label{fig:std_context_mask_probability_0.6}
    \end{subfigure}
    
    \caption{Comparison of different setting of $\rho$ in conditional BFN sampling on the $ 64 \times 64 $ Black Sea CHL dataset. The value of $\rho$ is set to be 1, 0.8, 0.6 from left to right.}
    \label{fig:stochastic_anchor_figure}
\end{figure}

%% file: appendix/experiments/visualization.tex
\subsection{Visualization of Generated Samples}
By statistically analyzing the masks within the dataset, we defined positions where no valid observations were ever recorded as ``land'', and represented them in gray during visualization to distinguish them from the unobserved fields, which are depicted in white. The areas enclosed by the red and black wireframes together constitute the observable region of the sample. During evaluation, the area enclosed by the black wireframe is the visible part provided to the model, which performs global reconstruction based on it. The red area is an artificially masked region used to evaluate the model's reconstruction performance. Here we provide visual comparison with the resolution of $64 \times 64$, $128 \times 128$ and $256 \times 256$ respectively in Fig.~\ref{fig:full_64_figure}, Fig.~\ref{fig:full_128_figure} and Fig.~\ref{fig:full_256_figure}

\begin{figure}[htbp]
    \centering
    \begin{subfigure}{0.24\textwidth}
        \centering
        \caption{GT}
        \includegraphics[width=\linewidth]{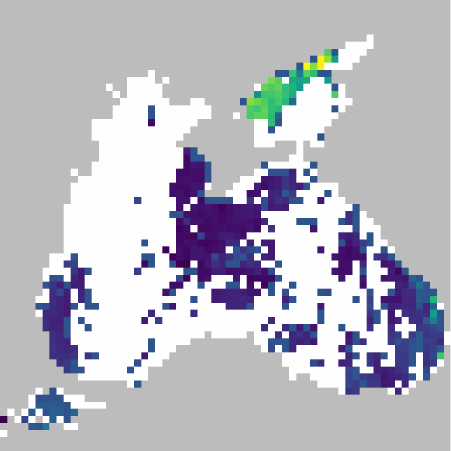}
        \label{fig:appendix_black_sea_gt}
    \end{subfigure}\quad
    \begin{subfigure}{0.24\textwidth}
        \centering
        \caption{GT}
        \includegraphics[width=\linewidth]{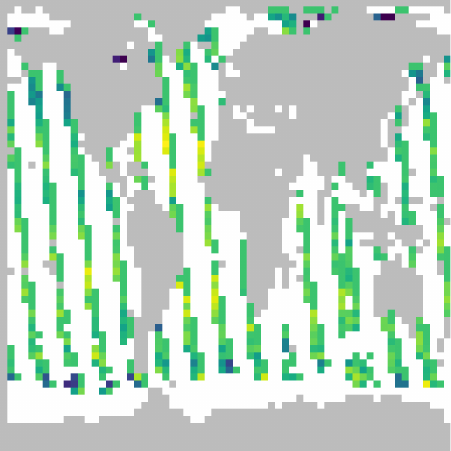}
        \label{fig:appendix_smos_gt}
    \end{subfigure}\quad
    \begin{subfigure}{0.24\textwidth}
        \centering
        \caption{GT}
        \includegraphics[width=\linewidth]{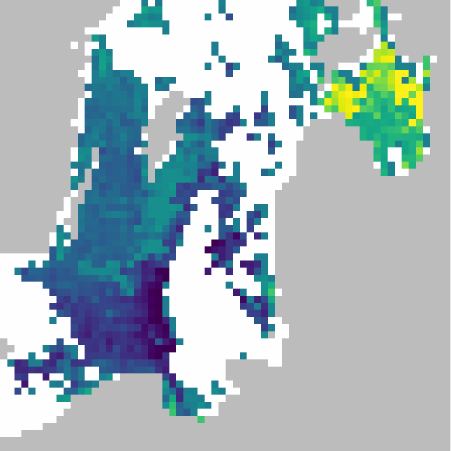}
        \label{fig:appendix_nano_gt}
    \end{subfigure}
    
    \vspace{0.4mm} 
    
    \begin{subfigure}{0.24\textwidth}
        \centering
        \caption{AmbientDiff}
        \includegraphics[width=\linewidth]{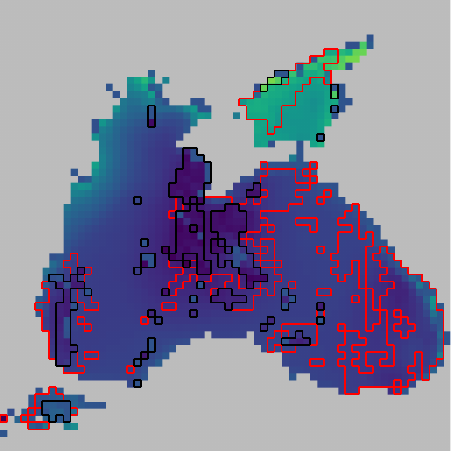}
        \label{fig:appendix_black_sea_ambient}
    \end{subfigure}\quad
    \begin{subfigure}{0.24\textwidth}
        \centering
        \caption{AmbientDiff}
        \includegraphics[width=\linewidth]{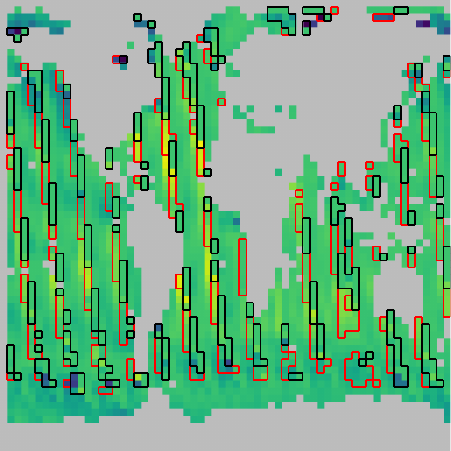}
        \label{fig:appendix_smos_ambient}
    \end{subfigure}\quad
    \begin{subfigure}{0.24\textwidth}
        \centering
        \caption{AmbientDiff}
        \includegraphics[width=\linewidth]{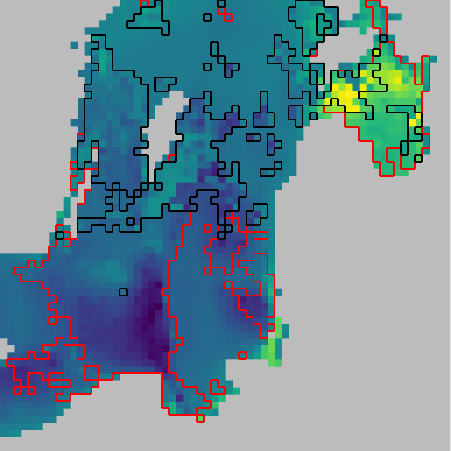}
        \label{fig:appendix_nano_ambient}
    \end{subfigure}

    \vspace{0.4mm} 
    
    \begin{subfigure}{0.24\textwidth}
        \centering
        \caption{DINDiff}
        \includegraphics[width=\linewidth]{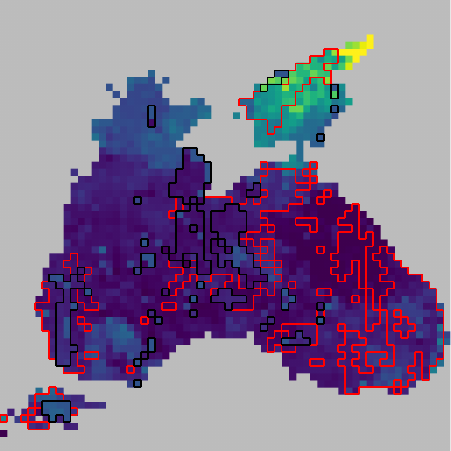}
        \label{fig:appendix_black_sea_dindiff}
    \end{subfigure}\quad
    \begin{subfigure}{0.24\textwidth}
        \centering
        \caption{DINDiff}
        \includegraphics[width=\linewidth]{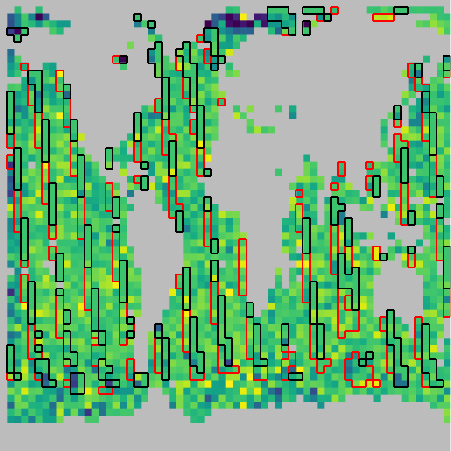}
        \label{fig:appendix_smos_dindiff}
    \end{subfigure}\quad
    \begin{subfigure}{0.24\textwidth}
        \centering
        \caption{DINDiff}
        \includegraphics[width=\linewidth]{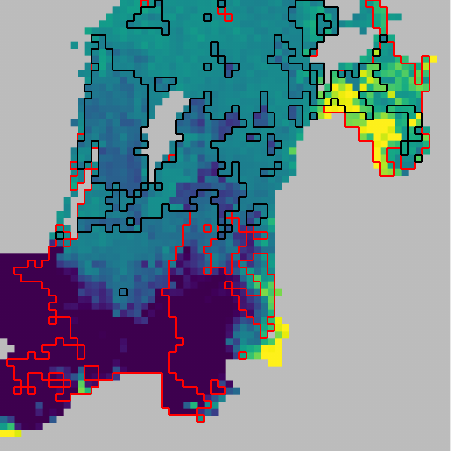}
        \label{fig:appendix_nano_dindiff}
    \end{subfigure}

    \vspace{0.4mm} 
    
    \begin{subfigure}{0.24\textwidth}
        \centering
        \caption{MissDiff}
        \includegraphics[width=\linewidth]{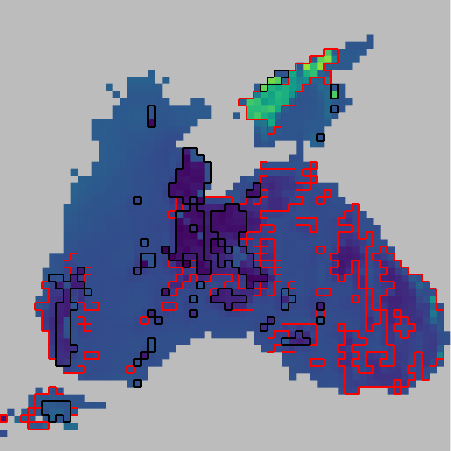}
        \label{fig:appendix_black_sea_missdiff}
    \end{subfigure}\quad
    \begin{subfigure}{0.24\textwidth}
        \centering
        \caption{MissDiff}
        \includegraphics[width=\linewidth]{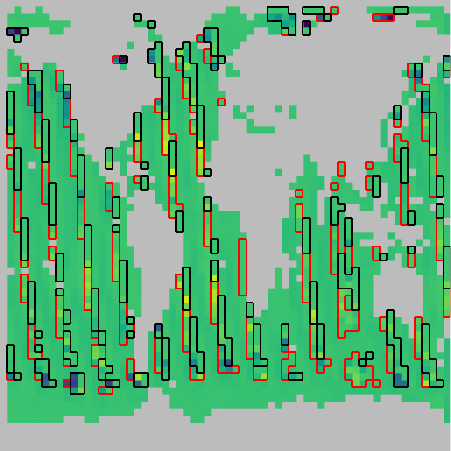}
        \label{fig:appendix_smos_missdiff}
    \end{subfigure}\quad
    \begin{subfigure}{0.24\textwidth}
        \centering
        \caption{MissDiff}
        \includegraphics[width=\linewidth]{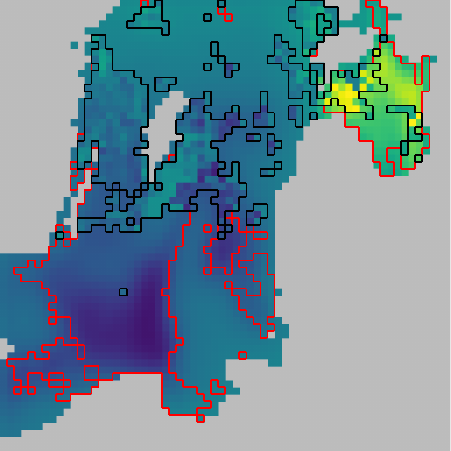}
        \label{fig:appendix_nano_missdiff}
    \end{subfigure}

    \vspace{0.4mm} 
    
    \begin{subfigure}{0.24\textwidth}
        \centering
        \caption{Ours}
        \includegraphics[width=\linewidth]{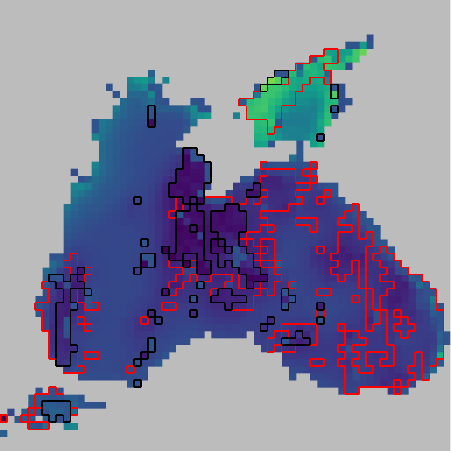}
        \label{fig:appendix_black_sea_ours}
    \end{subfigure}\quad
    \begin{subfigure}{0.24\textwidth}
        \centering
        \caption{Ours}
        \includegraphics[width=\linewidth]{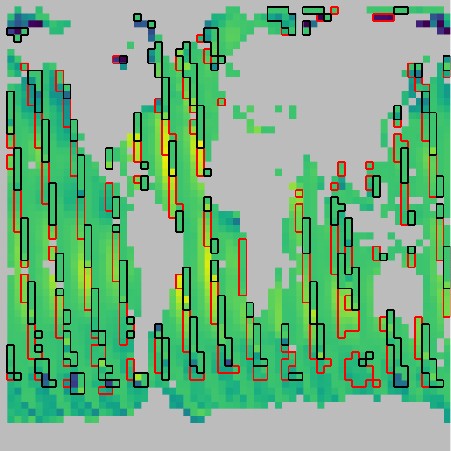}
        \label{fig:appendix_smos_ours}
    \end{subfigure}\quad
    \begin{subfigure}{0.24\textwidth}
        \centering
        \caption{Ours}
        \includegraphics[width=\linewidth]{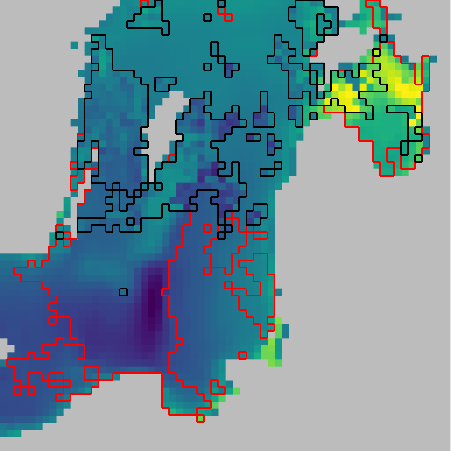}
        \label{fig:appendix_nano_ours}
    \end{subfigure}
    
    \caption{Visualization of imputation results on Black Sea CHL, Global Ocean SSS, and Baltic Sea NANO with the resolution of $64 \times 64$ (from left to right).}
    \label{fig:full_64_figure}
\end{figure}

\begin{figure}[htbp]
    \centering
    \begin{subfigure}{0.24\textwidth}
        \centering
        \caption{GT}
        \includegraphics[width=\linewidth]{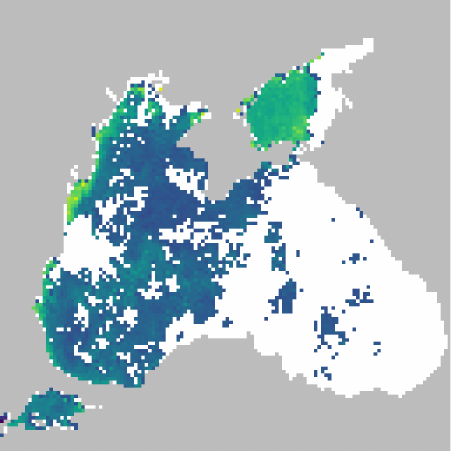}
        \label{fig:black_sea_128_gt}
    \end{subfigure}\quad
    \begin{subfigure}{0.24\textwidth}
        \centering
        \caption{GT}
        \includegraphics[width=\linewidth]{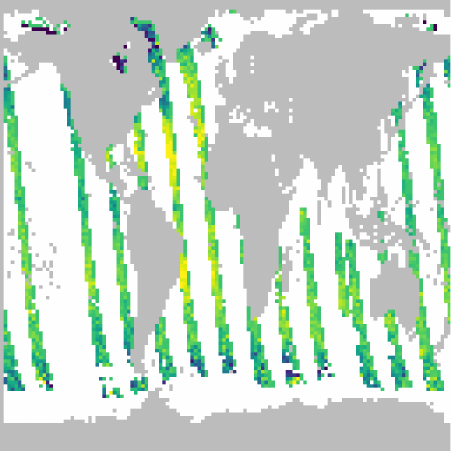}
        \label{fig:smos_128_gt}
    \end{subfigure}\quad
    \begin{subfigure}{0.24\textwidth}
        \centering
        \caption{GT}
        \includegraphics[width=\linewidth]{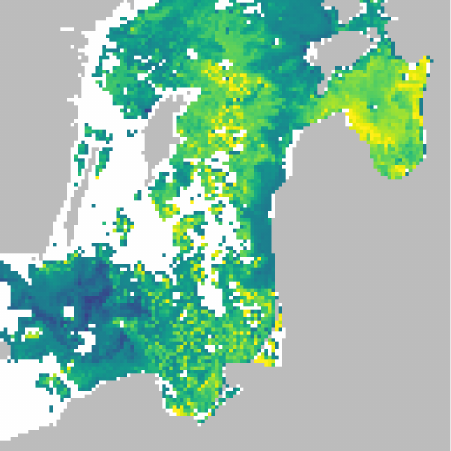}
        \label{fig:nano_128_gt}
    \end{subfigure}
    
    \vspace{0.4mm} 
    
    \begin{subfigure}{0.24\textwidth}
        \centering
        \caption{AmbientDiff}
        \includegraphics[width=\linewidth]{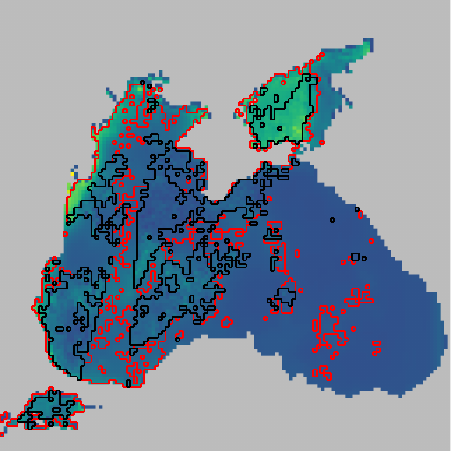}
        \label{fig:black_sea_128_ambient}
    \end{subfigure}\quad
    \begin{subfigure}{0.24\textwidth}
        \centering
        \caption{AmbientDiff}
        \includegraphics[width=\linewidth]{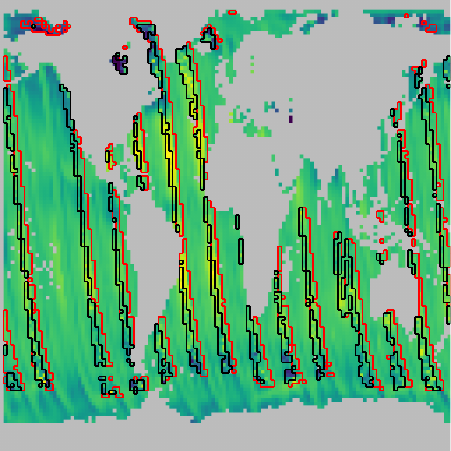}
        \label{fig:smos_128_ambient}
    \end{subfigure}\quad
    \begin{subfigure}{0.24\textwidth}
        \centering
        \caption{AmbientDiff}
        \includegraphics[width=\linewidth]{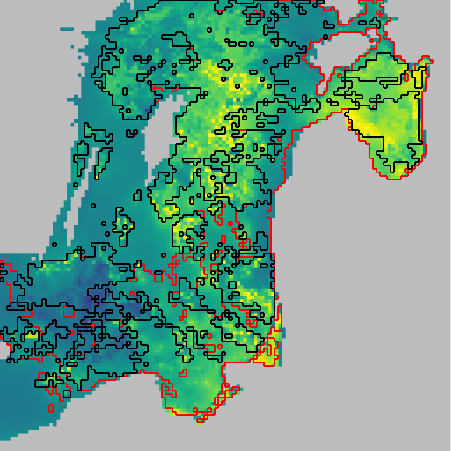}
        \label{fig:nano_128_ambient}
    \end{subfigure}

    \vspace{0.4mm} 
    
    \begin{subfigure}{0.24\textwidth}
        \centering
        \caption{DINDiff}
        \includegraphics[width=\linewidth]{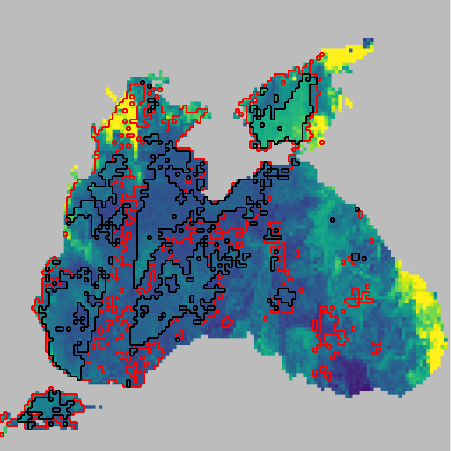}
        \label{fig:black_sea_128_dindiff}
    \end{subfigure}\quad
    \begin{subfigure}{0.24\textwidth}
        \centering
        \caption{DINDiff}
        \includegraphics[width=\linewidth]{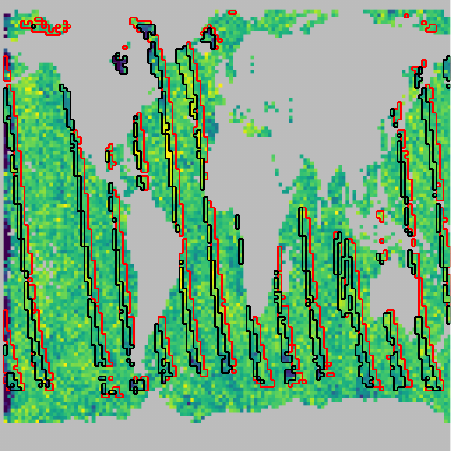}
        \label{fig:smos_128_dindiff}
    \end{subfigure}\quad
    \begin{subfigure}{0.24\textwidth}
        \centering
        \caption{DINDiff}
        \includegraphics[width=\linewidth]{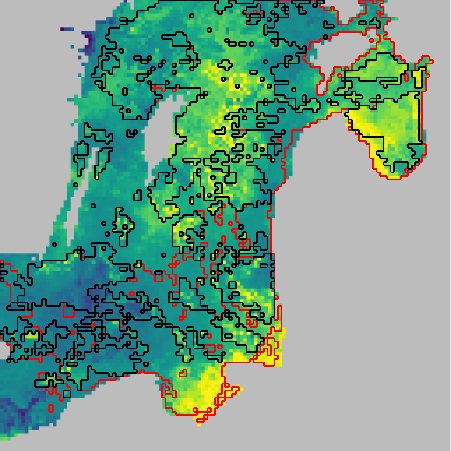}
        \label{fig:nano_128_dindiff}
    \end{subfigure}

    \vspace{0.4mm} 
    
    \begin{subfigure}{0.24\textwidth}
        \centering
        \caption{MissDiff}
        \includegraphics[width=\linewidth]{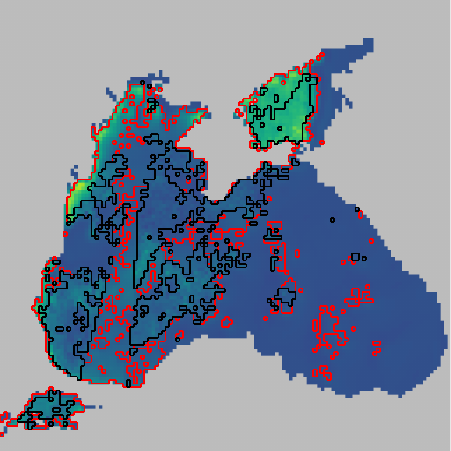}
        \label{fig:black_sea_128_missdiff}
    \end{subfigure}\quad
    \begin{subfigure}{0.24\textwidth}
        \centering
        \caption{MissDiff}
        \includegraphics[width=\linewidth]{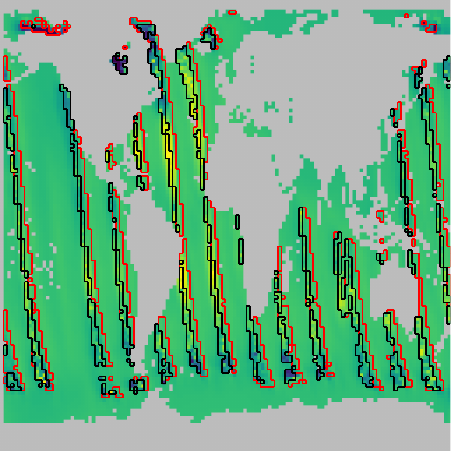}
        \label{fig:smos_128_missdiff}
    \end{subfigure}\quad
    \begin{subfigure}{0.24\textwidth}
        \centering
        \caption{MissDiff}
        \includegraphics[width=\linewidth]{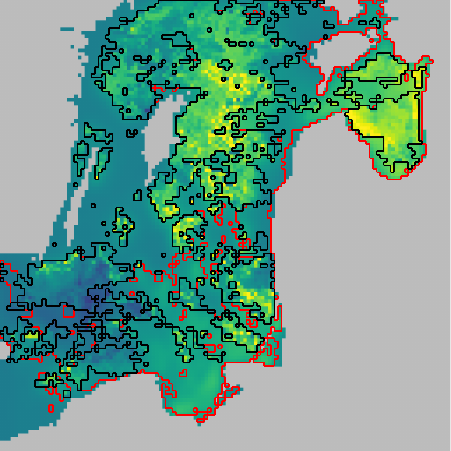}
        \label{fig:nano_128_missdiff}
    \end{subfigure}

    \vspace{0.4mm} 
    
    \begin{subfigure}{0.24\textwidth}
        \centering
        \caption{Ours}
        \includegraphics[width=\linewidth]{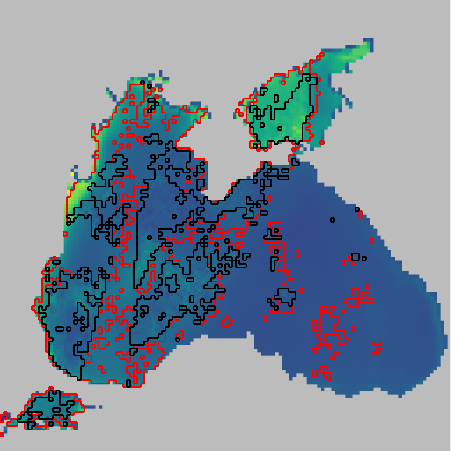}
        \label{fig:black_sea_128_ours}
    \end{subfigure}\quad
    \begin{subfigure}{0.24\textwidth}
        \centering
        \caption{Ours}
        \includegraphics[width=\linewidth]{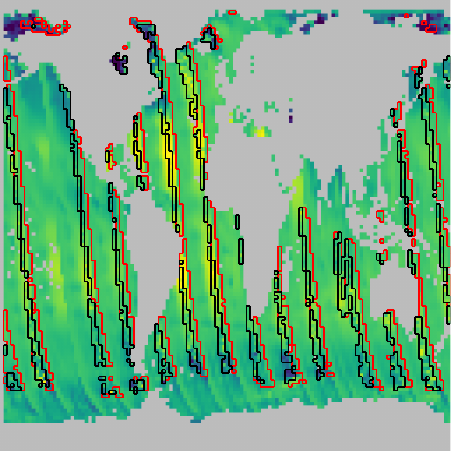}
        \label{fig:smos_128_ours}
    \end{subfigure}\quad
    \begin{subfigure}{0.24\textwidth}
        \centering
        \caption{Ours}
        \includegraphics[width=\linewidth]{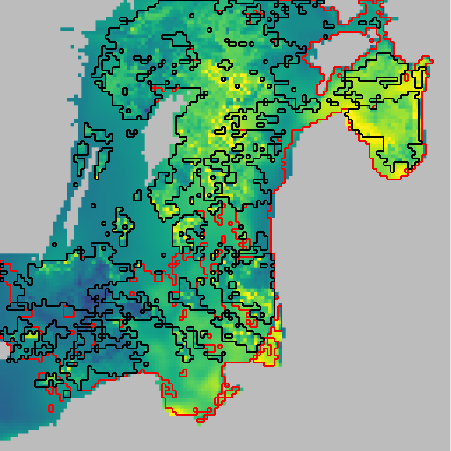}
        \label{fig:nano_128_ours}
    \end{subfigure}
    
    \caption{Visualization of imputation results on Black Sea CHL, Global Ocean SSS, and Baltic Sea NANO with the resolution of $128 \times 128$ (from left to right).}
    \label{fig:full_128_figure}
\end{figure}

\begin{figure}[htbp]
    \centering
    \begin{subfigure}{0.24\textwidth}
        \centering
        \caption{GT}
        \includegraphics[width=\linewidth]{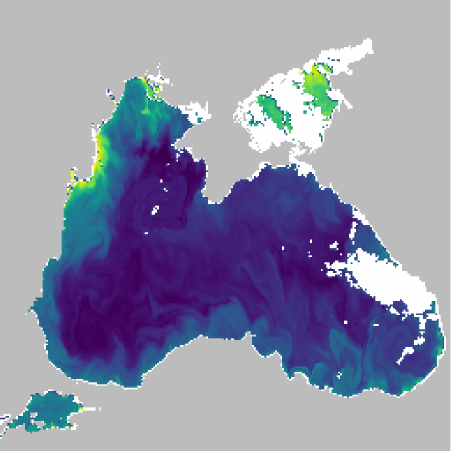}
        \label{fig:black_sea_256_gt}
    \end{subfigure}\quad
    \begin{subfigure}{0.24\textwidth}
        \centering
        \caption{GT}
        \includegraphics[width=\linewidth]{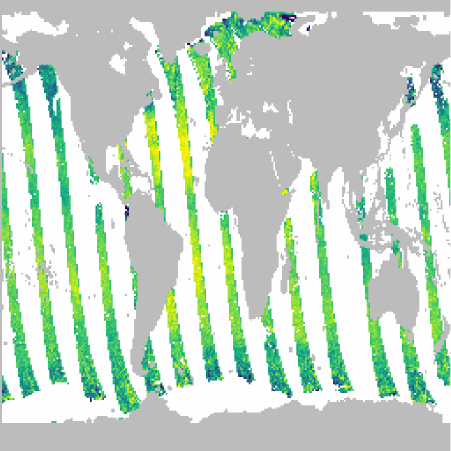}
        \label{fig:smos_256_gt}
    \end{subfigure}\quad
    \begin{subfigure}{0.24\textwidth}
        \centering
        \caption{GT}
        \includegraphics[width=\linewidth]{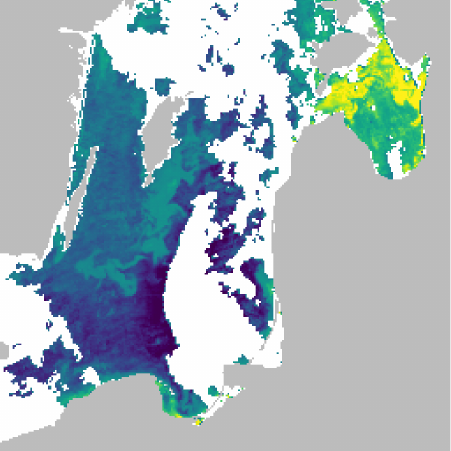}
        \label{fig:nano_256_gt}
    \end{subfigure}
    
    \vspace{0.4mm} 
    
    \begin{subfigure}{0.24\textwidth}
        \centering
        \caption{AmbientDiff}
        \includegraphics[width=\linewidth]{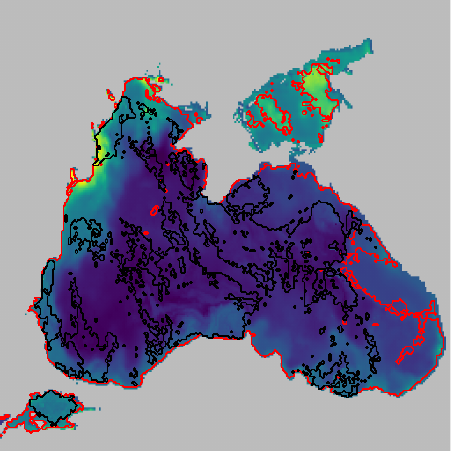}
        \label{fig:black_sea_256_ambient}
    \end{subfigure}\quad
    \begin{subfigure}{0.24\textwidth}
        \centering
        \caption{AmbientDiff}
        \includegraphics[width=\linewidth]{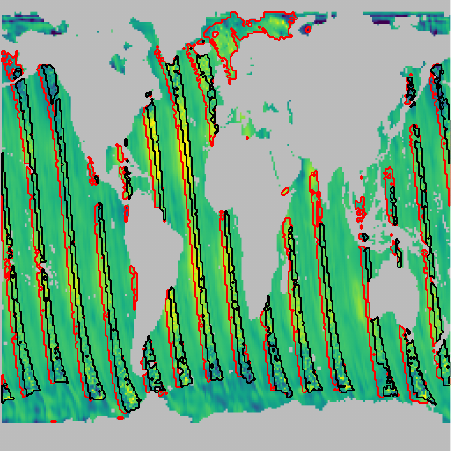}
        \label{fig:smos_256_ambient}
    \end{subfigure}\quad
    \begin{subfigure}{0.24\textwidth}
        \centering
        \caption{AmbientDiff}
        \includegraphics[width=\linewidth]{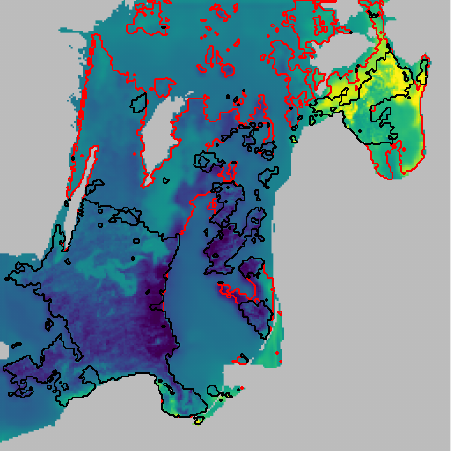}
        \label{fig:nano_256_ambient}
    \end{subfigure}

    \vspace{0.4mm} 
    
    \begin{subfigure}{0.24\textwidth}
        \centering
        \caption{DINDiff}
        \includegraphics[width=\linewidth]{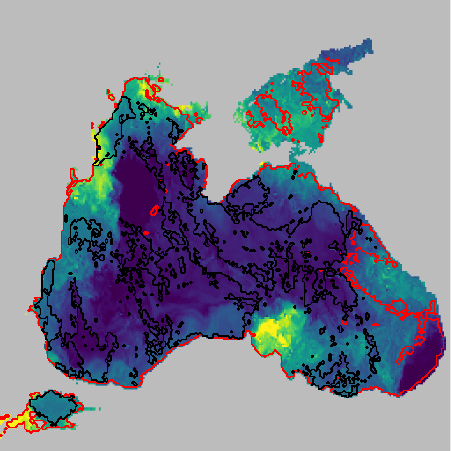}
        \label{fig:black_sea_256_dindiff}
    \end{subfigure}\quad
    \begin{subfigure}{0.24\textwidth}
        \centering
        \caption{DINDiff}
        \includegraphics[width=\linewidth]{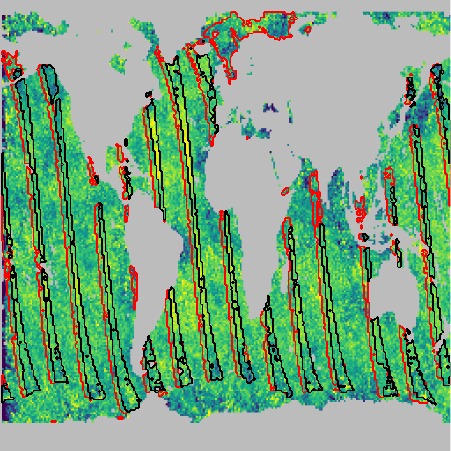}
        \label{fig:smos_256_dindiff}
    \end{subfigure}\quad
    \begin{subfigure}{0.24\textwidth}
        \centering
        \caption{DINDiff}
        \includegraphics[width=\linewidth]{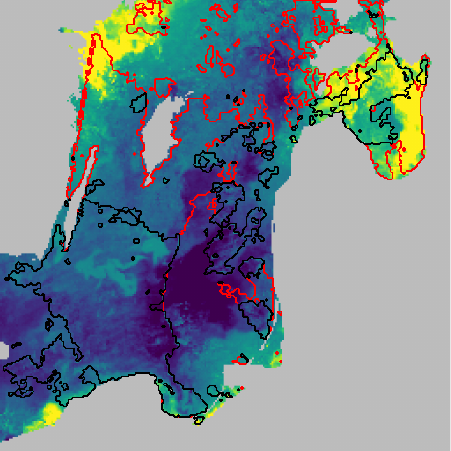}
        \label{fig:nano_256_dindiff}
    \end{subfigure}

    \vspace{0.4mm} 
    
    \begin{subfigure}{0.24\textwidth}
        \centering
        \caption{MissDiff}
        \includegraphics[width=\linewidth]{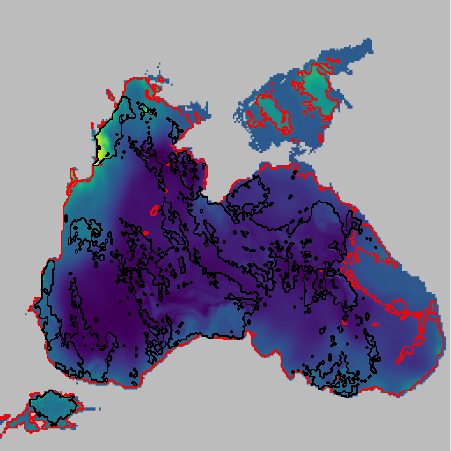}
        \label{fig:black_sea_256_missdiff}
    \end{subfigure}\quad
    \begin{subfigure}{0.24\textwidth}
        \centering
        \caption{MissDiff}
        \includegraphics[width=\linewidth]{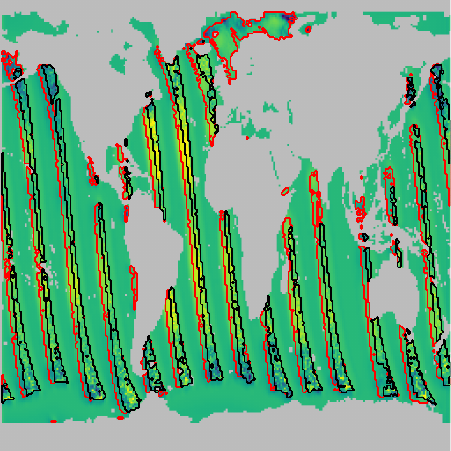}
        \label{fig:smos_256_missdiff}
    \end{subfigure}\quad
    \begin{subfigure}{0.24\textwidth}
        \centering
        \caption{MissDiff}
        \includegraphics[width=\linewidth]{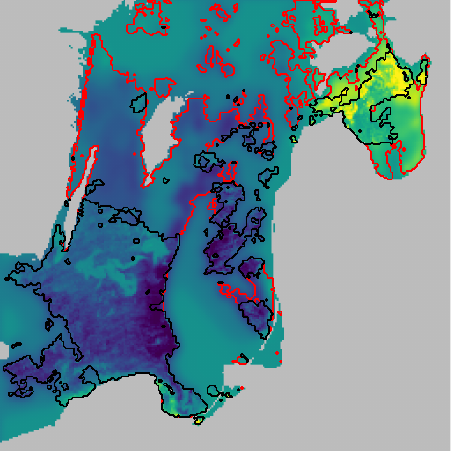}
        \label{fig:nano_256_missdiff}
    \end{subfigure}

    \vspace{0.4mm} 
    
    \begin{subfigure}{0.24\textwidth}
        \centering
        \caption{Ours}
        \includegraphics[width=\linewidth]{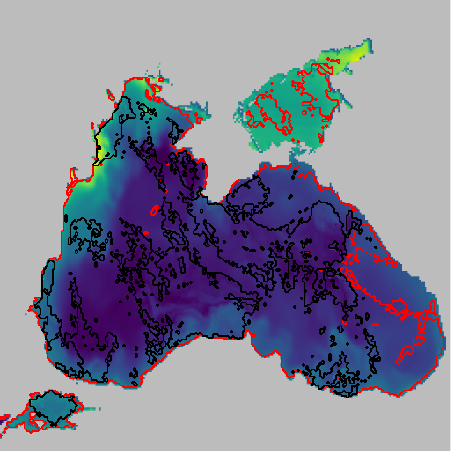}
        \label{fig:black_sea_256_ours}
    \end{subfigure}\quad
    \begin{subfigure}{0.24\textwidth}
        \centering
        \caption{Ours}
        \includegraphics[width=\linewidth]{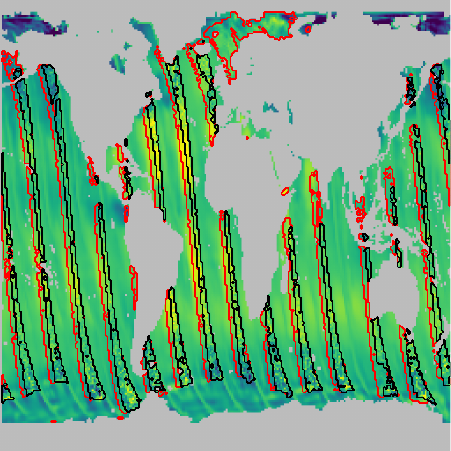}
        \label{fig:smos_256_ours}
    \end{subfigure}\quad
    \begin{subfigure}{0.24\textwidth}
        \centering
        \caption{Ours}
        \includegraphics[width=\linewidth]{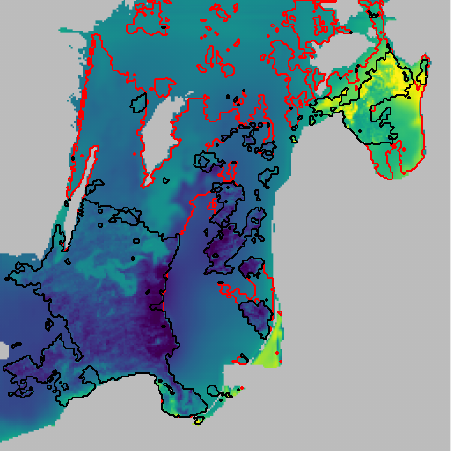}
        \label{fig:nano_256_ours}
    \end{subfigure}
    
    \caption{Visualization of imputation results on Black Sea CHL, Global Ocean SSS, and Baltic Sea NANO with the resolution of $256 \times 256$ (from left to right).}
    \label{fig:full_256_figure}
\end{figure}

%% file: appendix/context_select_methods.tex
\section{Implementation of Context Mask Selection Strategies} \label{app: context mask selection strategies}
We provide concrete algorithmic implementations of those context mask selection strategies mentioned in Sec.~\ref{sec:experiment-ablation}. Algorithm~\ref{alg:pixel_partition} presents the \textbf{Pixel-level Partition} strategy, where each point in the observation mask $\bm{M}$ is independently selected as
context or query through Bernoulli sampling with ratios $r_\text{ctx}=0.3$ and $r_\text{qry}=0.3$ here. Algorithm~\ref{alg:block_partition} shows the \textbf{Block-wise Partition} strategy, where a subset of blocks is randomly sampled with ratios $r_\text{ctx}=0.5$ and $r_\text{qry}=0.5$ here, and their intersection with the observation mask $\bm{M}$ forms the context and query masks, respectively. Algorithm~\ref{alg: saliency-driven partition} outlines the \textbf{Saliency-driven Partition} strategy, a heuristic approach that allocates context points based on sorting local gradient magnitudes. Algorithm~\ref{alg: empirical distribution partition} details the \textbf{Empirical Distribution Partition} strategy, which samples a real observation mask from the dataset to intersect with the current observation. Algorithm~\ref{alg: generative prior partition} demonstrates the \textbf{Unconditional Prior Partition} strategy, employing an unconditionally sampled mask from a pre-trained BFN to intersect with the current observation. As for \textbf{Our Proposed Partition} strategy, we have provided the implementation in Algorithm~\ref{alg:bfn_mask} and Algorithm~\ref{alg:main_training}.

\begin{algorithm}[H]
\caption{Pixel-Level Partition}
\label{alg:pixel_partition}
\begin{algorithmic}[1]
\REQUIRE Observation mask $M\in\{0,1\}^{d}$, context ratio $r_{\mathrm{ctx}}\in(0,1)$
\ENSURE Context mask $M_{\mathrm{ctx}}$, query mask $M_{\mathrm{qry}}$
\STATE Initialize $M_{\mathrm{ctx}}\leftarrow \mathbf{0}$
\FOR{each spatial index $i\in\{1,\ldots,d\}$}
    \IF{$M_i=1$}
        \STATE Sample $u\sim\mathrm{Uniform}(0,1)$
        \IF{$u<r_{\mathrm{ctx}}$}
            \STATE $(M_{\mathrm{ctx}})_i\leftarrow 1$
        \ENDIF
    \ENDIF
\ENDFOR
\STATE $M_{\mathrm{qry}}\leftarrow M\odot(1-M_{\mathrm{ctx}})$
\RETURN $M_{\mathrm{ctx}},M_{\mathrm{qry}}$
\end{algorithmic}
\end{algorithm}

\begin{algorithm}[H]
\caption{Block-Wise Partition}
\label{alg:block_partition}
\begin{algorithmic}[1]
\REQUIRE Observation mask $M\in\{0,1\}^{H\times W}$, block grid dimension $d$,
context ratio $r_{\mathrm{ctx}}\in(0,1]$
\ENSURE Context mask $M_{\mathrm{ctx}}$, query mask $M_{\mathrm{qry}}$
\STATE Initialize $M_{\mathrm{ctx}}\leftarrow \mathbf{0}^{H\times W}$
\STATE Partition the $H\times W$ domain into a $d\times d$ grid of blocks
$\{B_{i,j}\}_{i,j=1}^{d}$
\STATE $K\leftarrow d^2$
\STATE $N_{\mathrm{ctx}}\leftarrow \lfloor r_{\mathrm{ctx}}K\rfloor$
\STATE Uniformly sample a subset $S_{\mathrm{ctx}}$ of $N_{\mathrm{ctx}}$ blocks:
\[
S_{\mathrm{ctx}}
\leftarrow
\mathrm{Sample}\!\left(
\{B_{i,j}\}_{i,j=1}^{d},
N_{\mathrm{ctx}}
\right)
\]
\FOR{each block $B_{i,j}\in S_{\mathrm{ctx}}$}
    \STATE $M_{\mathrm{ctx}}[\mathrm{region}(B_{i,j})]
    \leftarrow
    M[\mathrm{region}(B_{i,j})]$
\ENDFOR
\STATE $M_{\mathrm{qry}}\leftarrow M\odot(1-M_{\mathrm{ctx}})$
\RETURN $M_{\mathrm{ctx}},M_{\mathrm{qry}}$
\end{algorithmic}
\end{algorithm}

\begin{algorithm}[H]
\caption{Saliency-driven Partition}
\label{alg: saliency-driven partition}
\begin{algorithmic}[1]
\REQUIRE observation mask $\bm{M} \in \{0, 1\}^d$, observed image data $\bm{x}_{\text{obs}}$, context ratio $r_{\text{ctx}} \in (0, 1)$
\ENSURE context mask $\bm{M}_{\text{ctx}}$, query mask $\bm{M}_{\text{qry}}$
\STATE initialize $\bm{M}_{\text{ctx}} \leftarrow \mathbf{0}, \bm{M}_{\text{qry}} \leftarrow \mathbf{0}$
\STATE $\bm{S}_{\text{obs}} \leftarrow \{i \in \{1, \dots, d\} \mid M_i = 1\}$ 
\STATE $N_{\text{obs}} \leftarrow |\bm{S}_{\text{obs}}|$
\STATE $N_{\text{ctx}} \leftarrow \lfloor r_{\text{ctx}} \cdot N_{\text{obs}} \rfloor$ 
\STATE $N_{\text{half}} \leftarrow \lfloor N_{\text{ctx}} / 2 \rfloor$ 
\STATE compute local gradient magnitudes $G_i = \|\nabla x_{\text{obs}, i}\|$ for all $i \in \bm{S}_{\text{obs}}$
\STATE sort $\bm{S}_{\text{obs}}$ in descending order based on $G_i$ into an ordered list $L$
\STATE $\bm{S}_{\text{high}} \leftarrow$ first half of elements in $L$ 
\STATE $\bm{S}_{\text{low}} \leftarrow$ second half of elements in $L$ 
\STATE randomly sample a subset $\bm{S}_{\text{ctx\_high}} \subseteq \bm{S}_{\text{high}}$ of size $N_{\text{half}}$
\STATE randomly sample a subset $\bm{S}_{\text{ctx\_low}} \subseteq \bm{S}_{\text{low}}$ of size $(N_{\text{ctx}} - N_{\text{half}})$
\STATE $\bm{S}_{\text{ctx}} \leftarrow \bm{S}_{\text{ctx\_high}} \cup \bm{S}_{\text{ctx\_low}}$ 
\FOR{each spatial index $i \in \{1, \dots, d\}$}
    \IF{$\bm{M}_i = 1$}
        \IF{$i \in \bm{S}_{\text{ctx}}$}
            \STATE $(\bm{M}_{\text{ctx}})_i \leftarrow 1$
        \ELSE
            \STATE $(\bm{M}_{\text{qry}})_i \leftarrow 1$ 
        \ENDIF
    \ENDIF
\ENDFOR
\RETURN $\bm{M}_{\text{ctx}}, \bm{M}_{\text{qry}}$
\end{algorithmic}
\end{algorithm}

\begin{algorithm}[H]
\caption{Empirical Distribution Partition}
\label{alg: empirical distribution partition}
\begin{algorithmic}[1]
\REQUIRE observation mask $\bm{M} \in \{0, 1\}^d$, dataset of real observation masks $\mathcal{D}_{\text{mask}}$
\ENSURE context mask $\bm{M}_{\text{ctx}}$, query mask $\bm{M}_{\text{qry}}$
\STATE sample a real observation mask $\bm{M}_{\text{sample}} \sim \mathcal{D}_{\text{mask}}$
\STATE initialize $\bm{M}_{\text{ctx}} \leftarrow \mathbf{0}, \bm{M}_{\text{qry}} \leftarrow \mathbf{0}$
\FOR{each spatial index $i \in \{1, \dots, d\}$}
    \IF{$\bm{M}_i = 1$}
        \IF{$(\bm{M}_{\text{sample}})_i = 1$}
            \STATE $(\bm{M}_{\text{ctx}})_i \leftarrow 1$ 
        \ELSE
            \STATE $(\bm{M}_{\text{qry}})_i \leftarrow 1$ 
        \ENDIF
    \ENDIF
\ENDFOR
\RETURN $\bm{M}_{\text{ctx}}, \bm{M}_{\text{qry}}$
\end{algorithmic}
\end{algorithm}

\begin{algorithm}[H]
\caption{Unconditional Prior Partition}
\label{alg: generative prior partition}
\begin{algorithmic}[1]
\REQUIRE observation mask $\bm{M} \in \{0, 1\}^d$, pre-trained BFN model $\mathcal{B}_{\text{mask}}$
\ENSURE context mask $\bm{M}_{\text{ctx}}$, query mask $\bm{M}_{\text{qry}}$
\STATE unconditionally sample a $M_{\text{sample}} \sim \mathcal{B}_{\text{mask}}$
\STATE initialize $\bm{M}_{\text{ctx}} \leftarrow \mathbf{0}, \bm{M}_{\text{qry}} \leftarrow \mathbf{0}$
\FOR{each spatial index $i \in \{1, \dots, d\}$}
    \IF{$\bm{M}_i = 1$}
        \IF{$(\bm{M}_{\text{sample}})_i = 1$}
            \STATE $(\bm{M}_{\text{ctx}})_i \leftarrow 1$ 
        \ELSE
            \STATE $(\bm{M}_{\text{qry}})_i \leftarrow 1$ 
        \ENDIF
    \ENDIF
\ENDFOR
\RETURN $\bm{M}_{\text{ctx}}, \bm{M}_{\text{qry}}$
\end{algorithmic}
\end{algorithm}

%% file: appendix/sampling_methods.tex
\section{Implementation of Sampling Methods} \label{app: sampling methods}
We provide detailed algorithmic implementations of various sampling methods evaluated in Sec.~\ref{sec:experiment-ablation}. Inspired by the methodologies introduced in~\citep{zhou2025incomplete}, we first present three fundamental approaches. Algorithm~\ref{alg: direct projection} demonstrates the \textbf{Direct Projection} method, which performs a one-step projection to minimize global reconstruction error. Algorithm~\ref{alg: proximal prediction} details \textbf{Proximal Prediction} method, introducing a single step denoising process with an infinitesimal noise level to better align with the diffusion prior and mitigate blurriness. Algorithm~\ref{alg: iterative conditioning} details the \textbf{Iterative Conditioning} method. Notably, to heuristically blend the diffusion expectation and imputation expectation, we employ a simple linear schedule ($\omega_s = s$) as the time-varying weight. Furthermore, drawing inspiration from the resampling mechanisms proposed in RePaint~\citep{lugmayr2022repaint}, we present two advanced sampling strategies. Algorithm~\ref{alg:resampling inpainting} describes the implementation of \textbf{Resampling Inpainting} method, which harmonizes observation boundaries by alternating between denoising and stochastic resampling steps. Algorithm~\ref{alg: recursive jump iterative} demonstrates the \textbf{Recursive Jump Inpainting} method, which utilizes profound temporal jumps to re-inject significant variance, effectively reconciling global physical dynamics with sparse observation constraints.

\begin{algorithm}[H]
\caption{Direct Projection}
\label{alg: direct projection}
\begin{algorithmic}[1]
\REQUIRE partially observed data $\bm{x}_{\text{obs}}$, mask $\bm{M}$, trained model $\bm{u}_{\bm{\phi}}$, ensemble size $K$
\ENSURE imputed complete data $\bm{x}_0$
\STATE Generate $K$ random context masks $\{\bm{M}_{\text{ctx}}^{(k)}\}_{k=1}^K \subseteq \bm{M}$
\STATE $\bm{x}^* \leftarrow \frac{1}{K} \sum_{k=1}^K \bm{u}_{\bm{\phi}}\left(\bm{M}_{\text{ctx}}^{(k)} \odot \bm{x}_{\text{obs}}, \bm{M}_{\text{ctx}}^{(k)}\right)$
\STATE $\bm{x}_{0} \leftarrow \bm{M} \odot \bm{x}_{\text{obs}} + (1 - \bm{M}) \odot \bm{x}^*$
\RETURN $\bm{x}_0$
\end{algorithmic}
\end{algorithm}

\begin{algorithm}[H]
\caption{Proximal Prediction}
\label{alg: proximal prediction}
\begin{algorithmic}[1]
\REQUIRE partially observed data $\bm{x}_{\text{obs}}$, mask $\bm{M}$, trained model $\bm{u}_{\bm{\phi}}$, minimal timestep $\delta \ll 1$, ensemble size $K$
\ENSURE imputed complete data $\bm{x}_0$
\STATE $\bm{\epsilon} \sim \mathcal{N}(0, I)$
\STATE $\bm{x}_{\text{obs},\delta} \leftarrow \alpha_\delta \bm{x}_{\text{obs}} + \sigma_\delta \bm{\epsilon}$
\STATE Generate $K$ random context masks $\{\bm{M}_{\text{ctx}}^{(k)}\}_{k=1}^K \subseteq \bm{M}$
\STATE $\bm{x}_{\delta} \leftarrow \frac{1}{K} \sum_{k=1}^K \bm{u}_{\bm{\phi}}\left(\delta, \bm{M}_{\text{ctx}}^{(k)} \odot \bm{x}_{\text{obs},\delta}, \bm{M}_{\text{ctx}}^{(k)}\right)$
\STATE $\bm{x}_0 \leftarrow \bm{M} \odot \bm{x}_{\text{obs}} + (1 - \bm{M}) \odot \bm{x}_{\delta}$
\RETURN $\bm{x}_0$
\end{algorithmic}
\end{algorithm}

\begin{algorithm}[H]
\caption{Iterative Conditioning}
\label{alg: iterative conditioning}
\begin{algorithmic}[1]
\REQUIRE partially observed data $\bm{x}_{\text{obs}}$, mask $\bm{M}$, trained model $\bm{u}_{\bm{\phi}}$, ensemble size $K$, weight schedule $\omega_s=s$
\ENSURE imputed complete data $\bm{x}_0$
\STATE Generate $K$ random context masks $\{\bm{M}_{\text{ctx}}^{(k)}\}_{k=1}^K \subseteq \bm{M}$
\STATE $E_{\text{imp}} \leftarrow \frac{1}{K} \sum_{k=1}^K \bm{u}_{\bm{\phi}}\left(\bm{M}_{\text{ctx}}^{(k)} \odot \bm{x}_{\text{obs}}, \bm{M}_{\text{ctx}}^{(k)}\right)$
\STATE initialize: $\bm{x}_T \sim \mathcal{N}(\mathbf{0}, \bm{I})$
\FOR{diffusion steps from $s$ to $t$}
    \STATE sample a $\bm{M}_{\text{ctx}}$ 
    \STATE $E_{\text{diff}} \leftarrow \bm{u}_\phi\left(s, \bm{M}_{\text{ctx}} \odot \bm{x}_s, \bm{M}_{\text{ctx}}\right)$
    \STATE $\hat{\bm{x}}_0 \leftarrow \omega_s E_{\text{diff}} + (1 - \omega_s) E_{\text{imp}}$
    \STATE $\bm{\epsilon}_{\text{unobs}} \leftarrow \frac{\bm{x}_s - \alpha_s \hat{\bm{x}}_0}{\sigma_s}$
    \STATE $\bm{\epsilon}_{\text{obs}} \leftarrow \frac{\bm{x}_s - \alpha_s \bm{x}_{\text{obs}}}{\sigma_s}$
    \STATE $\bm{\epsilon}_{\text{full}} \leftarrow \bm{M} \odot \bm{\epsilon}_{\text{obs}} + (1 - \bm{M}) \odot \bm{\epsilon}_{\text{unobs}}$
    \STATE $\bm{x}_t \leftarrow \text{DiffusionODE}(s, t, \bm{x}_s, \bm{\epsilon}_{\text{full}})$
\ENDFOR

\STATE $\bm{x}_0 \leftarrow \bm{M} \odot \bm{x}_{\text{obs}} + (1 - \bm{M}) \odot \bm{x}_t$
\RETURN $\bm{x}_0$
\end{algorithmic}
\end{algorithm}

\begin{algorithm}[H]
\caption{Resampling Inpainting}
\label{alg:resampling inpainting}
\begin{algorithmic}[1]
\REQUIRE partially observed data $\bm{x}_{\text{obs}}$, mask $\bm{M}$, trained model $\bm{u}_{\bm{\phi}}$, ensemble size $K$, weight schedule $\omega_s=s$, jump length $j$, jump frequency $f$
\ENSURE imputed complete data $\bm{x}_0$
\STATE Generate $K$ random context masks $\{\bm{M}_{\text{ctx}}^{(k)}\}_{k=1}^K \subseteq \bm{M}$
\STATE $E_{\text{imp}} \leftarrow \frac{1}{K} \sum_{k=1}^K \bm{u}_\phi\left(\bm{M}_{\text{ctx}}^{(k)} \odot \bm{x}_{\text{obs}}, \bm{M}_{\text{ctx}}^{(k)}\right)$
\STATE initialize: $\bm{x}_T \sim \mathcal{N}(0, I)$
\STATE $s \leftarrow T$
\STATE $c \leftarrow 0$
\WHILE{$s > 0$}
    \STATE $t \leftarrow s - 1$
    \STATE sample a $\bm{M}_{\text{ctx}}$
    \STATE $E_{\text{diff}} \leftarrow \bm{u}_{\bm{\phi}}(s, \bm{M}_{\text{ctx}} \odot \bm{x}_s, \bm{M}_{\text{ctx}})$
    \STATE $\hat{x}_0 \leftarrow \omega_s E_{\text{diff}} + (1 - \omega_s) E_{\text{imp}}$
    \STATE $\bm{\epsilon}_{\text{unobs}} \leftarrow \frac{\bm{x}_s - \alpha_s \bm{\hat{x}}_0}{\sigma_s}$
    \STATE $\bm{\epsilon}_{\text{obs}} \leftarrow \frac{\bm{x}_s - \alpha_s \bm{x}_{\text{obs}}}{\sigma_s}$
    \STATE $\bm{\epsilon}_{\text{full}} \leftarrow \bm{M} \odot \bm{\epsilon}_{\text{obs}} + (1 - \bm{M}) \odot \bm{\epsilon}_{\text{unobs}}$
    \STATE $\bm{x}_t \leftarrow \text{DiffusionODE}(s, t, \bm{x}_s, \bm{\epsilon}_{\text{full}})$
    \STATE $c \leftarrow c + 1$
    \IF{$c \mathbin{\%} f == 0$ \AND $t > 0$}
        \STATE $s_{\text{new}} \leftarrow \min(t + j, T)$
        \STATE $\bm{x}_{s_{\text{new}}} \leftarrow \text{ForwardDiffusion}(t, s_{\text{new}}, \bm{x}_t)$
        \STATE $s \leftarrow s_{\text{new}}$
    \ELSE
        \STATE $s \leftarrow t$
    \ENDIF
\ENDWHILE

\STATE $\bm{x}_0 \leftarrow \bm{M} \odot \bm{x}_{\text{obs}} + (1 - \bm{M}) \odot \bm{x}_t$
\RETURN $\bm{x}_0$
\end{algorithmic}
\end{algorithm}

\begin{algorithm}[H]
\caption{Recursive Jump Inpainting}
\label{alg: recursive jump iterative}
\begin{algorithmic}[1]
\REQUIRE partially observed data $\bm{x}_{\text{obs}}$, mask $\bm{M}$, trained model $\bm{u}_{\bm{\phi}}$, ensemble size $K$, weight schedule $\omega_s=s$, total jump stages $n$
\ENSURE imputed complete data $\bm{x}_0$
\STATE Generate $K$ random context masks $\{\bm{M}_{\text{ctx}}^{(k)}\}_{k=1}^K \subseteq \bm{M}$
\STATE $E_{\text{imp}} \leftarrow \frac{1}{K} \sum_{k=1}^K \bm{u}_{\bm{\phi}}\left(\bm{M}_{\text{ctx}}^{(k)} \odot \bm{x}_{\text{obs}}, \bm{M}_{\text{ctx}}^{(k)}\right)$
\STATE initialize: $\bm{x}_T \sim \mathcal{N}(0, I)$
\STATE $s \leftarrow T$
\STATE $k \leftarrow n - 1$ 
\WHILE{$s > 0$}
    \STATE $t \leftarrow s - 1$
    \STATE sample a $\bm{M}_{\text{ctx}}$
    \STATE $E_{\text{diff}} \leftarrow \bm{u}_{\bm{\phi}}(s, \bm{M}_{\text{ctx}} \odot \bm{x}_s, \bm{M}_{\text{ctx}})$
    \STATE $\bm{\hat{x}}_0 \leftarrow \omega_s E_{\text{diff}} + (1 - \omega_s) E_{\text{imp}}$
    \STATE $\bm{\epsilon}_{\text{unobs}} \leftarrow \frac{\bm{x}_s - \alpha_s \bm{\hat{x}}_0}{\sigma_s}$
    \STATE $\bm{\epsilon}_{\text{obs}} \leftarrow \frac{\bm{x}_s - \alpha_s \bm{x}_{\text{obs}}}{\sigma_s}$
    \STATE $\bm{\epsilon}_{\text{full}} \leftarrow \bm{M} \odot \bm{\epsilon}_{\text{obs}} + (1 - \bm{M}) \odot \bm{\epsilon}_{\text{unobs}}$
    \STATE $\bm{x}_t \leftarrow \text{DiffusionODE}(s, t, \bm{x}_s, \bm{\epsilon}_{\text{full}})$
    \IF{$t == 0$ \AND $k > 0$}
        \STATE $T_{\text{jump}} \leftarrow \lfloor \frac{k}{n} T \rfloor$
        \STATE $\bm{x}_{T_{\text{jump}}} \leftarrow \text{ForwardDiffusion}(0, T_{\text{jump}}, \bm{x}_t)$
        \STATE $s \leftarrow T_{\text{jump}}$
        \STATE $k \leftarrow k - 1$ 
    \ELSE
        \STATE $s \leftarrow t$
    \ENDIF
\ENDWHILE
\STATE $\bm{x}_0 \leftarrow \bm{M} \odot \bm{x}_{\text{obs}} + (1 - \bm{M}) \odot \bm{x}_t$
\RETURN $\bm{x}_0$
\end{algorithmic}
\end{algorithm}

%% file: ref.bib
@article{barth2024ensemble,
  title={Ensemble reconstruction of missing satellite data using a denoising diffusion model: application to chlorophyll a concentration in the Black Sea},
  author={Barth, Alexander and Brajard, Julien and Alvera-Azc{\'a}rate, Aida and Mohamed, Bayoumy and Troupin, Charles and Beckers, Jean-Marie},
  journal={Ocean Science},
  volume={20},
  number={6},
  pages={1567--1584},
  year={2024},
  publisher={Copernicus Publications G{\"o}ttingen, Germany}
}

@article{zhou2025incomplete,
  title={Incomplete Data, Complete Dynamics: A Diffusion Approach},
  author={Zhou, Zihan and Wang, Chenguang and Ye, Hongyi and Guan, Yongtao and Yu, Tianshu},
  journal={arXiv preprint arXiv:2509.20098},
  year={2025}
}

@inproceedings{lugmayr2022repaint,
  title={Repaint: Inpainting using denoising diffusion probabilistic models},
  author={Lugmayr, Andreas and Danelljan, Martin and Romero, Andres and Yu, Fisher and Timofte, Radu and Van Gool, Luc},
  booktitle={Proceedings of the IEEE/CVF conference on computer vision and pattern recognition},
  pages={11461--11471},
  year={2022}
}

@misc{graves2025bayesianflownetworks,
      title={Bayesian Flow Networks}, 
      author={Alex Graves and Rupesh Kumar Srivastava and Timothy Atkinson and Faustino Gomez},
      year={2025},
      eprint={2308.07037},
      archivePrefix={arXiv},
      primaryClass={cs.LG},
      url={https://arxiv.org/abs/2308.07037}, 
}

@article{xue2024unifying,
  title={Unifying bayesian flow networks and diffusion models through stochastic differential equations},
  author={Xue, Kaiwen and Zhou, Yuhao and Nie, Shen and Min, Xu and Zhang, Xiaolu and Zhou, Jun and Li, Chongxuan},
  journal={arXiv preprint arXiv:2404.15766},
  year={2024}
}

@article{qiu2024empower,
  title={Empower structure-based molecule optimization with gradient guided bayesian flow networks},
  author={Qiu, Keyue and Song, Yuxuan and Yu, Jie and Ma, Hongbo and Cao, Ziyao and Zhang, Zhilong and Wu, Yushuai and Zheng, Mingyue and Zhou, Hao and Ma, Wei-Ying},
  journal={arXiv preprint arXiv:2411.13280},
  year={2024}
}

@misc{daras2023ambientdiffusionlearningclean,
      title={Ambient Diffusion: Learning Clean Distributions from Corrupted Data}, 
      author={Giannis Daras and Kulin Shah and Yuval Dagan and Aravind Gollakota and Alexandros G. Dimakis and Adam Klivans},
      year={2023},
      eprint={2305.19256},
      archivePrefix={arXiv},
      primaryClass={cs.LG},
      url={https://arxiv.org/abs/2305.19256}, 
}

@misc{ouyang2023missdifftrainingdiffusionmodels,
      title={MissDiff: Training Diffusion Models on Tabular Data with Missing Values}, 
      author={Yidong Ouyang and Liyan Xie and Chongxuan Li and Guang Cheng},
      year={2023},
      eprint={2307.00467},
      archivePrefix={arXiv},
      primaryClass={cs.LG},
      url={https://arxiv.org/abs/2307.00467}, 
}

@inproceedings{zheng2023improved,
  title={Improved techniques for maximum likelihood estimation for diffusion odes},
  author={Zheng, Kaiwen and Lu, Cheng and Chen, Jianfei and Zhu, Jun},
  booktitle={International Conference on Machine Learning},
  pages={42363--42389},
  year={2023},
  organization={PMLR}
}

@misc{erdinc2024generativegeostatisticalmodelingincomplete,
      title={Generative Geostatistical Modeling from Incomplete Well and Imaged Seismic Observations with Diffusion Models}, 
      author={Huseyin Tuna Erdinc and Rafael Orozco and Felix J. Herrmann},
      year={2024},
      eprint={2406.05136},
      archivePrefix={arXiv},
      primaryClass={physics.geo-ph},
      url={https://arxiv.org/abs/2406.05136}, 
}

@misc{liu2025imputemacfmimputationbasedmaskaware,
      title={Impute-MACFM: Imputation based on Mask-Aware Flow Matching}, 
      author={Dengyi Liu and Honggang Wang and Hua Fang},
      year={2025},
      eprint={2509.23126},
      archivePrefix={arXiv},
      primaryClass={cs.LG},
      url={https://arxiv.org/abs/2509.23126}, 
}

@misc{yu2025missingdataimputationreducing,
      title={Missing Data Imputation by Reducing Mutual Information with Rectified Flows}, 
      author={Jiahao Yu and Qizhen Ying and Leyang Wang and Ziyue Jiang and Song Liu},
      year={2025},
      eprint={2505.11749},
      archivePrefix={arXiv},
      primaryClass={stat.ML},
      url={https://arxiv.org/abs/2505.11749}, 
}

@misc{liang2024latentspacescorebaseddiffusion,
      title={Latent Space Score-based Diffusion Model for Probabilistic Multivariate Time Series Imputation}, 
      author={Guojun Liang and Najmeh Abiri and Atiye Sadat Hashemi and Jens Lundström and Stefan Byttner and Prayag Tiwari},
      year={2024},
      eprint={2409.08917},
      archivePrefix={arXiv},
      primaryClass={cs.LG},
      url={https://arxiv.org/abs/2409.08917}, 
}

@inproceedings{Bora2018AmbientGANGM,
  title={AmbientGAN: Generative models from lossy measurements},
  author={Ashish Bora and Eric Price and Alexandros G. Dimakis},
  booktitle={International Conference on Learning Representations},
  year={2018},
  url={https://api.semanticscholar.org/CorpusID:3481010}
}

@misc{li2021fourierneuraloperatorparametric,
      title={Fourier Neural Operator for Parametric Partial Differential Equations}, 
      author={Zongyi Li and Nikola Kovachki and Kamyar Azizzadenesheli and Burigede Liu and Kaushik Bhattacharya and Andrew Stuart and Anima Anandkumar},
      year={2021},
      eprint={2010.08895},
      archivePrefix={arXiv},
      primaryClass={cs.LG},
      url={https://arxiv.org/abs/2010.08895}, 
}

@misc{huang2024diffusionpdegenerativepdesolvingpartial,
      title={DiffusionPDE: Generative PDE-Solving Under Partial Observation}, 
      author={Jiahe Huang and Guandao Yang and Zichen Wang and Jeong Joon Park},
      year={2024},
      eprint={2406.17763},
      archivePrefix={arXiv},
      primaryClass={cs.LG},
      url={https://arxiv.org/abs/2406.17763}, 
}

@misc{karras2024analyzingimprovingtrainingdynamics,
      title={Analyzing and Improving the Training Dynamics of Diffusion Models}, 
      author={Tero Karras and Miika Aittala and Jaakko Lehtinen and Janne Hellsten and Timo Aila and Samuli Laine},
      year={2024},
      eprint={2312.02696},
      archivePrefix={arXiv},
      primaryClass={cs.CV},
      url={https://arxiv.org/abs/2312.02696}, 
}

@misc{dosovitskiy2021imageworth16x16words,
      title={An Image is Worth 16x16 Words: Transformers for Image Recognition at Scale}, 
      author={Alexey Dosovitskiy and Lucas Beyer and Alexander Kolesnikov and Dirk Weissenborn and Xiaohua Zhai and Thomas Unterthiner and Mostafa Dehghani and Matthias Minderer and Georg Heigold and Sylvain Gelly and Jakob Uszkoreit and Neil Houlsby},
      year={2021},
      eprint={2010.11929},
      archivePrefix={arXiv},
      primaryClass={cs.CV},
      url={https://arxiv.org/abs/2010.11929}, 
}

@misc{hoogeboom2021argmaxflowsmultinomialdiffusion,
      title={Argmax Flows and Multinomial Diffusion: Learning Categorical Distributions}, 
      author={Emiel Hoogeboom and Didrik Nielsen and Priyank Jaini and Patrick Forré and Max Welling},
      year={2021},
      eprint={2102.05379},
      archivePrefix={arXiv},
      primaryClass={stat.ML},
      url={https://arxiv.org/abs/2102.05379}, 
}

@misc{austin2023structureddenoisingdiffusionmodels,
      title={Structured Denoising Diffusion Models in Discrete State-Spaces}, 
      author={Jacob Austin and Daniel D. Johnson and Jonathan Ho and Daniel Tarlow and Rianne van den Berg},
      year={2023},
      eprint={2107.03006},
      archivePrefix={arXiv},
      primaryClass={cs.LG},
      url={https://arxiv.org/abs/2107.03006}, 
}

@misc{gat2024discreteflowmatching,
      title={Discrete Flow Matching}, 
      author={Itai Gat and Tal Remez and Neta Shaul and Felix Kreuk and Ricky T. Q. Chen and Gabriel Synnaeve and Yossi Adi and Yaron Lipman},
      year={2024},
      eprint={2407.15595},
      archivePrefix={arXiv},
      primaryClass={cs.LG},
      url={https://arxiv.org/abs/2407.15595}, 
}

@misc{cachay2023dyffusiondynamicsinformeddiffusionmodel,
      title={DYffusion: A Dynamics-informed Diffusion Model for Spatiotemporal Forecasting}, 
      author={Salva Rühling Cachay and Bo Zhao and Hailey Joren and Rose Yu},
      year={2023},
      eprint={2306.01984},
      archivePrefix={arXiv},
      primaryClass={cs.LG},
      url={https://arxiv.org/abs/2306.01984}, 
}

@misc{lippe2023pderefinerachievingaccuratelong,
      title={PDE-Refiner: Achieving Accurate Long Rollouts with Neural PDE Solvers}, 
      author={Phillip Lippe and Bastiaan S. Veeling and Paris Perdikaris and Richard E. Turner and Johannes Brandstetter},
      year={2023},
      eprint={2308.05732},
      archivePrefix={arXiv},
      primaryClass={cs.LG},
      url={https://arxiv.org/abs/2308.05732}, 
}

@misc{bastek2025physicsinformeddiffusionmodels,
      title={Physics-Informed Diffusion Models}, 
      author={Jan-Hendrik Bastek and WaiChing Sun and Dennis M. Kochmann},
      year={2025},
      eprint={2403.14404},
      archivePrefix={arXiv},
      primaryClass={cs.LG},
      url={https://arxiv.org/abs/2403.14404}, 
}

@misc{shysheya2024conditionaldiffusionmodelspde,
      title={On conditional diffusion models for PDE simulations}, 
      author={Aliaksandra Shysheya and Cristiana Diaconu and Federico Bergamin and Paris Perdikaris and José Miguel Hernández-Lobato and Richard E. Turner and Emile Mathieu},
      year={2024},
      eprint={2410.16415},
      archivePrefix={arXiv},
      primaryClass={cs.LG},
      url={https://arxiv.org/abs/2410.16415}, 
}

@misc{majid2026ambientphysicstrainingneural,
      title={Ambient Physics: Training Neural PDE Solvers with Partial Observations}, 
      author={Harris Abdul Majid and Giannis Daras and Francesco Tudisco and Steven McDonagh},
      year={2026},
      eprint={2602.13873},
      archivePrefix={arXiv},
      primaryClass={cs.AI},
      url={https://arxiv.org/abs/2602.13873}, 
}

@article{bi2023accurate,
  title={Accurate medium-range global weather forecasting with 3D neural networks},
  author={Bi, Kaifeng and Xie, Lingxi and Zhang, Hengheng and Chen, Xin and Gu, Xiaotao and Tian, Qi},
  journal={Nature},
  volume={619},
  number={7970},
  pages={533--538},
  year={2023},
  publisher={Nature Publishing Group UK London}
}

@article{lam2023learning,
  title={Learning skillful medium-range global weather forecasting},
  author={Lam, Remi and Sanchez-Gonzalez, Alvaro and Willson, Matthew and Wirnsberger, Peter and Fortunato, Meire and Alet, Ferran and Ravuri, Suman and Ewalds, Timo and Eaton-Rosen, Zach and Hu, Weihua and others},
  journal={Science},
  volume={382},
  number={6677},
  pages={1416--1421},
  year={2023},
  publisher={American Association for the Advancement of Science}
}

@article{martin2025generative,
  title={Generative data assimilation for surface ocean state estimation from multi-modal satellite observations},
  author={Martin, Scott A and Manucharyan, Georgy E and Klein, Patrice},
  journal={Journal of Advances in Modeling Earth Systems},
  volume={17},
  number={8},
  pages={e2025MS005063},
  year={2025},
  publisher={Wiley Online Library}
}

@article{chen2025transfer,
  title={Transfer machine learning framework for efficient full-field temperature response reconstruction of thermal protection structures with limited measurement data},
  author={Chen, Yuluo and Chen, Qiang and Ma, Han and Chen, Shuailong and Fei, Qingguo},
  journal={International Journal of Heat and Mass Transfer},
  volume={242},
  pages={126785},
  year={2025},
  publisher={Elsevier}
}

@article{hsu2025forecasting,
  title={Forecasting corporate financial performance using deep learning with environmental, social, and governance data},
  author={Hsu, Wan-Lu and Lin, Ying-Lei and Lai, Jung-Pin and Liu, Yu-Hui and Pai, Ping-Feng},
  journal={Electronics},
  volume={14},
  number={3},
  pages={417},
  year={2025},
  publisher={MDPI}
}

@article{ashfaq2025predicting,
  title={Predicting wheat yield using deep learning and multi-source environmental data},
  author={Ashfaq, Muhammad and Khan, Imran and Shah, Dilawar and Ali, Shujaat and Tahir, Muhammad},
  journal={Scientific Reports},
  volume={15},
  number={1},
  pages={26446},
  year={2025},
  publisher={Nature Publishing Group UK London}
}

@techreport{letraon:hal-03405376,
  TITLE = {{THE COPERNICUS MARINE ENVIRONMENTAL MONITORING SERVICE: MAIN SCIENTIFIC ACHIEVEMENTS AND FUTURE PROSPECTS}},
  AUTHOR = {Letraon, Py and Ali, A. and Fanjul, E. Alvarez and Aouf, L. and Axell, L. and Aznar, R. and Ballarotta, M. and Behrens, A. and Benkiran, M. and Bentamy, A. and Bertino, L. and Bowyer, P. and Brando, V. and Breivik, L. A. and Nardelli, B. Buongiorno and Cailleau, S. and Ciliberti, S. A. and Colella, S. and Connell, N. Mc and Coppini, G. and Cossarini, G. and Dabrowski, T. and Alonsomu{\~n}oyerro, M. de Alfonso and O'dea, E. and Desportes, C. and Dinessen, F. and Drevillon, M. and Drillet, Y. and Drudi, M. and Dussurget, R. and Faug{\`e}re, Y. and Forneris, V. and Fratianni, C. and Galloudec, O. Le and Hermosa, I. Garcia and Sotillo, M. Garc{\'i}a and Garnesson, P. and Garric, G. and Golbeck, I. and Gourrion, J. and Gr{\'e}goire, M. L. and Guinehut, S. and Gutknecht, E. and Harris, C. and Hernandez, F. and Huess, V. and Johannessen, J.A. and Kay, S. and Killick, R. and King, R. and Kloe, J. De and Korres, G. and Lagemaa, P. and Lecci, R. and Legeais, J.F. and Lellouche, J. M. and Levier, B. and Lorente, P. and Mangin, A. and Martin, M. and Melet, A. and Murawski, J. and {\"O}zsoy, E. and Palazov, A. and Pardo, S. and Parent, L. and Pascual, A. and Paul, J. and Peneva, E. and Perruche, C. and Peterson, D. and Villeon, L. Petit de La and Pinardi, N. and Pouliquen, S. and Pujol, M. I. and Rainaud, R. and Rampal, Pierre and Reffray, G. and Regnier, C. and Reppucci, A. and Ryan, A. and Salon, S. and Samuelsen, A. and Santoleri, R. and Saulter, Andrew and She, J. and Solidoro, C. and Stanev, E. and Staneva, J. and Stoffelen, A. and Storto, A. and Sykes, P. and Szekely, T. and Taburet, G. and Taylor, B. and Tintore, J. and Toledano, C. and Tonani, M. and Tuomi, L. and Volpe, G. and Wedhe, H. and Williams, T. and Vandendbulcke, L. and Zanten, D. Van and Schuckmann, K. Von and Xie, J. and Zacharioudaki, A. and Zuo, H.},
  URL = {https://hal.univ-grenoble-alpes.fr/hal-03405376},
  TYPE = {Research Report},
  NUMBER = {56},
  INSTITUTION = {{MERCATOR OCEAN}},
  YEAR = {2017},
  MONTH = Sep,
  DOI = {10.25575/56},
  HAL_ID = {hal-03405376},
  HAL_VERSION = {v1},
}

@misc{park2025measurementscorebaseddiffusionmodel,
      title={Measurement Score-Based Diffusion Model}, 
      author={Chicago Y. Park and Shirin Shoushtari and Hongyu An and Ulugbek S. Kamilov},
      year={2025},
      eprint={2505.11853},
      archivePrefix={arXiv},
      primaryClass={eess.IV},
      url={https://arxiv.org/abs/2505.11853}, 
}

@inproceedings{li2018misgan,
  title={Learning from Incomplete Data with Generative Adversarial Networks},
  author={Steven Cheng-Xian Li and Bo Jiang and Benjamin Marlin},
  booktitle={International Conference on Learning Representations},
  year={2019},
  url={https://openreview.net/forum?id=S1lDV3RcKm},
}
